\documentclass{article}
\usepackage[preprint]{neurips_2026}
\usepackage[T1]{fontenc}
\usepackage{amsmath}
\usepackage{amsthm}
\usepackage{libertinus}
\usepackage{libertinust1math}
\usepackage{microtype}
\usepackage{graphicx}
\usepackage{booktabs}
\usepackage{tabularx}
\usepackage{titletoc}
\usepackage[font=small,labelfont=bf]{caption}
\usepackage{thmtools}
\usepackage{mathrsfs}
\usepackage{bm}
\usepackage{comment}
\usepackage{float}
\usepackage{placeins}
\usepackage{algorithm}
\usepackage{algpseudocode}
\usepackage{xcolor}
\definecolor{preprintblue}{RGB}{18,61,120}
\usepackage{hyperref}
\hypersetup{
  pdftitle={A Splitting Method for SDE Terminal-Law Estimation},
  pdfauthor={Rushil Gupta and Sandeep Juneja},
  colorlinks=true,
  linkcolor=preprintblue,
  citecolor=preprintblue,
  urlcolor=preprintblue
}

\title{A Splitting Method for SDE Terminal-Law Estimation}
\author{%
  \begin{tabular}[t]{c@{\hspace{1.2in}}c}
    \textbf{Rushil Gupta}\textsuperscript{1} &
    \textbf{Sandeep Juneja}\textsuperscript{1}
  \end{tabular} \\[0.6em]
  \normalfont\textsuperscript{1}Safexpress Centre for Data, Learning and Decision Sciences \\
  \normalfont Ashoka University
}
\date{}

\newcommand{\E}{\mathbb{E}}
\newcommand{\Pbb}{\mathbb{P}}
\newcommand{\R}{\mathbb{R}}
\newcommand{\G}{\mathbb{G}}
\newcommand{\Var}{\operatorname{Var}}
\newcommand{\Cov}{\operatorname{Cov}}

\newcommand{\KS}{\operatorname{KS}}
\newcommand{\PAVA}{\operatorname{PAVA}}
\newcommand{\tr}{\operatorname{tr}}

\newcommand{\one}{\mathbf{1}}
\newcommand{\cA}{\mathcal{A}}
\newcommand{\cF}{\mathcal{F}}
\newcommand{\sF}{\mathscr{F}}
\newcommand{\ResultCell}[2]{#1\,(#2)}
\newcommand{\BestResultCell}[2]{{\bfseries\boldmath #1\,(#2)}}

\declaretheoremstyle[
    spaceabove=1.0\baselineskip,
    spacebelow=0.5\baselineskip
]{spacedtheorem}
\declaretheorem[style=spacedtheorem,name=Theorem]{theorem}
\declaretheorem[style=spacedtheorem,name=Corollary]{corollary}
\declaretheorem[style=spacedtheorem,name=Lemma]{lemma}
\declaretheorem[style=spacedtheorem,name=Proposition]{proposition}
\declaretheorem[style=spacedtheorem,name=Remark]{remark}
\declaretheorem[style=spacedtheorem,name=Assumption]{assumption}

\begin{document}
\maketitle

\begin{abstract}
In many settings involving stochastic differential equations (SDEs), including diffusion based generative AI, our aim is to accurately generate samples from a distribution. Typically, this is done by generating i.i.d. samples of the SDE. Given a fixed simulation budget, a reasonable way to gain efficiency may be to instead generate a tree of paths through appropriately split partial paths. This suggests improved performance, but one worries about the injected dependence. In this paper, we study this issue comprehensively. We develop a general method to construct the splitting design and analyze its performance for a broad class of error measures, giving insight into the elegant underlying structure in the problem. For the Kolmogorov-Smirnov (KS) distance and maximum mean discrepancy (MMD), we identify the limiting errors of the associated empirical distributions as the simulation budget increases to infinity. We characterize a splitting strategy motivated by a corresponding asymptotic optimization problem. Overall, we observe a 12-26\% improvement in mean KS over i.i.d. samples in many settings. On CIFAR-10 and FFHQ, splitting lowers the MMD by 17-26\% relative to i.i.d. samples.\footnote{Code is available at \url{https://github.com/RushilGupta4/splitting}.}
\end{abstract}

\section{Introduction}
\label{sec:introduction}

Many applications require samples from a terminal  distribution of  a stochastic differential equation (SDE). In diffusion-based generative AI, e.g., in DDPMs \citep{Ho2020DDPM} and in EDMs \citep{Karras2022edm}, a trained model is used to run an SDE, and samples from its terminal distribution are of interest. These samples are often used to estimate quantities of interest to make decisions. For example, GenCast is a diffusion model that generates samples of weather forecasts, which are used to estimate performance measures such as the likelihood of extreme weather and uncertainty in wind power \citep{PriceEtAl2025}. 

SDEs are also standard for modeling the evolution of financial quantities such as asset prices and interest rates. Here, the underlying SDE parameters are typically learned from data and one may be interested in ascertaining the probability distribution at the terminal time. Recently, diffusion models have also been used to generate financial time series \citep{HuangEtAl2024}.

In practice, one may be interested in a variety of functionals of the SDE terminal distribution, such as expectations of functions of the terminal output or tail risk measures. However, in this paper, we focus on the fundamental question of accurately estimating the underlying distribution given a fixed simulation budget. This is often done using i.i.d. sampling. We ask whether a well-optimized branching process that splits paths at intermediate times can estimate the terminal distribution better than i.i.d. sampling with the same computational effort.

Our theoretical results reveal the elegant dependence structure, which is completely captured by the covariance of the SDE across time segments. This leads to an allocation problem where the amount of splitting in the branching process is chosen to minimize the limiting error. We solve a relaxed optimization problem, and approximate the solution with a finite mixture of trees with a provable approximation guarantee. This is implemented using a two-phase procedure that first learns the variance contributions and then spends the remaining budget according to the learned allocation. The limit theory applies to any error that is a continuous function of the empirical process under the uniform norm over a class of test functions on which the process satisfies a functional central limit theorem. In this paper, we develop it for the Kolmogorov-Smirnov (KS) distance and the maximum mean discrepancy (MMD).  Experiments spanning SDEs and diffusion models show consistent improvements over independent full-path simulation.

The paper is structured as follows. Section~\ref{sec:related-work} reviews the closest work. Section~\ref{sec:discretization-splitting-estimator} develops theory for a finite-budget and fixed-threshold setting and identifies the optimal allocation in that setting. Section~\ref{sec:asymptotic-analysis} develops limit theorems by modeling the empirical process. Section~\ref{sec:ks-statistic} studies the Kolmogorov-Smirnov error and the corresponding relaxed allocation problem. Section~\ref{sec:rounding} converts the relaxed allocation into a mixture of exact dyadic trees. Finally, Section~\ref{sec:numerical-experiments} presents the experiments.

\section{Related Work}
\label{sec:related-work}

The idea of splitting to reduce Monte Carlo error is not new. In classical multilevel splitting, paths are copied when they reach intermediate levels leading to a rare event. \citet{GlassermanEtAl1999} study this through branching-process arguments and show how the number of copies affects both variance and computational cost. Their analysis suggests keeping the expected particle population stable across levels. However, this line of work is mainly concerned with estimating a rare-event probability.

Related allocation questions also appear outside rare-event simulation. \citet{Melas1993} studies branching estimators for functionals of Markov chains and allows the branching rate to depend on the current state. More simulation effort is assigned where the future contribution to the estimator is more uncertain, while total cost is controlled. Splitting has also been combined with Markov chain Monte Carlo (MCMC) methods. \citet{BotevKroese2012} apply splitting to static problems, such as estimating normalizing constants. \citet{CerouGuyader2016} study an adaptive version that, at each stage, keeps only the particles with the highest scores, resamples from them, and moves the copies by MCMC. However, both papers target a single rare-event probability or integral for a static random vector, while we estimate the terminal CDF of a process that evolves in time.

Distributional estimation is considered more directly by \citet{AndersonEhlert2022}. They estimate the terminal probability mass function of a stochastic reaction network by simulating paths to one intermediate time and then generating several independent continuations from each partial path. They choose the split time and the number of continuations by balancing distributional error against simulation cost. However, their construction uses a discrete state space and a single split.

For discretized SDEs, \citet{GilesHajiAli2024} use repeated path branching for a digital payoff within multilevel Monte Carlo. Their analysis studies the dependence created by shared histories and gives rules for placing repeated binary split times near maturity. The target is one payoff correction, and the main design variable is the timing of the splits rather than a general allocation of particle counts.

\section{Discretization and Splitting Estimator}
\label{sec:discretization-splitting-estimator}
\begin{figure}[!htbp]
    \centering
    \begingroup
    \setlength{\fboxsep}{4pt}
    \fbox{\includegraphics[width=0.9\textwidth]{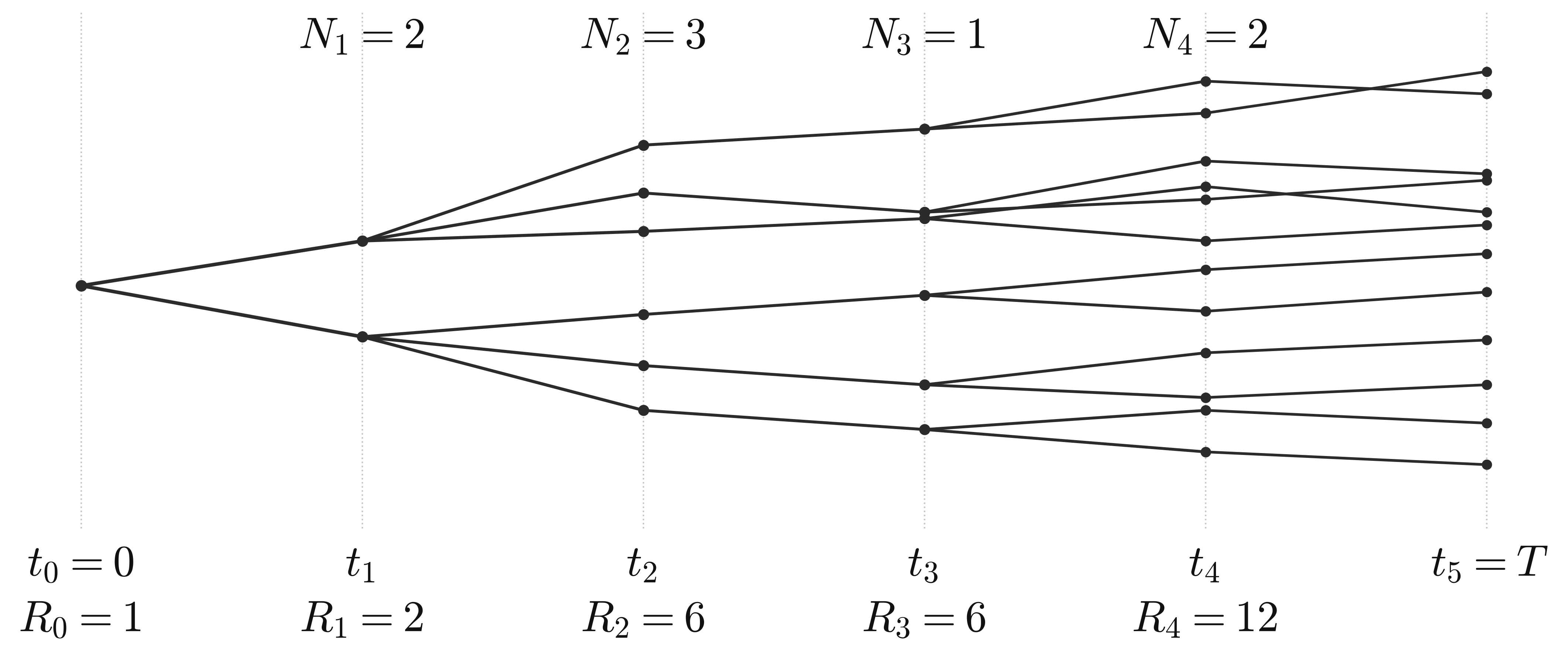}}
    \endgroup
    \caption{Illustration of splitting: at time \(t_{i,B}\), each particle is split into \(N_{i,B}\) copies.}
    \label{fig:splitting-schematic}
\end{figure}

Let \(X=(X_t)_{0\le t\le T}\) be a solution to
\[
    dX_t=b(t,X_t)\,dt+\sigma(t,X_t)\,dW_t,
    \qquad X_0\sim\mu_0,
    \qquad X_0\perp W.
\]
We aim to estimate its terminal distribution function \(F(x)=\Pbb(X_T\le x)\) for all \(x \in\R^d\). For each budget \(B\), choose a time partition and the corresponding partition size
\[
    0=t_{0,B}<t_{1,B}<\cdots<t_{K_B,B}=T, \qquad h_B=\max_{0\le i<K_B}(t_{i+1,B}-t_{i,B}).
\]
We assume throughout that \(h_B\to0\) as \(B\to\infty\). Let \(X_0^B,X_1^B,\ldots,X_{K_B}^B\) be a Markov-chain discretization of the SDE along this partition, with \(X_0^B\sim\mu_0\).  Its terminal distribution function is
\[
    F_B(x)=\Pbb\left(X_{K_B}^B\le x\right).
\]
There are \(K_B\) transitions. Transition \(i\) propagates particles from \(t_{i,B}\) to \(t_{i+1,B}\). At \(t_{0,B}\), there are \(N_{0,B}\) independent root particles. For \(1\le i\le K_B-1\), split \(i\) occurs after particles arrive at \(t_{i,B}\) and before transition \(i\). Each arriving particle is copied \(N_{i,B}\ge1\) times, and the copies are propagated independently through transition \(i\). An illustration is shown in Figure~\ref{fig:splitting-schematic}. Define
\[
    R_{0,B}=1,
    \qquad
    R_{i,B}=\prod_{j=1}^i N_{j,B},
    \quad 1\le i\le K_{B}-1.
\]
At the terminal time, there are \(R_{K_B-1,B}\) descendants per root. If \(Z_{B,j}^{(\ell)}\) denotes terminal descendant \(\ell\) from root \(j\), the splitting estimator is
\[
    \widehat F_B(x)
    =
    \frac1{N_{0,B}R_{K_B-1,B}}
    \sum_{j=1}^{N_{0,B}}
    \sum_{\ell=1}^{R_{K_B-1,B}}
    \one\!\left(Z_{B,j}^{(\ell)}\le x\right),
    \qquad x\in\R^d.
\]
This estimator is unbiased for \(F_B(x)\). For simplicity, we assume that each transition has the same unit cost. However, this can be relaxed, and our results extend to settings where each transition has a different cost; see Appendix~\ref{app:cost-weighted}. In the unit cost setting, a fully used budget satisfies
\[
    B=N_{0,B}\sum_{i=0}^{K_B-1}R_{i,B}.
\]

To analyse the variance of this estimator, observe that particles that share a root remain dependent due to their common history. Once they split, their transitions are conditionally independent. The next result shows how this structure can be understood by the covariance of the empirical CDF. For \(i=0,\ldots,K_B\), define \(u^B_{i,x}\) to be the conditional expectation of \(\one(X_{K_B}^B\le x)\) as
\[
    u^B_{i,x}(z)=\E[\one(X_{K_B}^B\le x)\mid X_i^B=z].
\]
\begin{lemma}[Splitting covariance]\label{lem:splitting-cov}
For every \(x,y\in\R^d\),
\[
    \Cov(\widehat F_B(x),\widehat F_B(y))
    =
    \frac1{N_{0,B}}
    \left[
        \Gamma_{0,B}(x,y)
        +
        \sum_{i=0}^{K_B-1}\frac{v_{i,B}(x,y)}{R_{i,B}}
    \right],
\]
where
\begin{align*}
    \Gamma_{0,B}(x,y)
    &=\Cov\!\left(u^B_{0,x}(X_0^B),u^B_{0,y}(X_0^B)\right),\\
    v_{i,B}(x,y)
    &=\E\!\left[
        \Cov\!\left(
            u^B_{i+1,x}(X_{i+1}^B),
            u^B_{i+1,y}(X_{i+1}^B)
            \mid X_i^B
        \right)
    \right].
\end{align*}
\end{lemma}

The proof is given in Appendix~\ref{app:proof-splitting-cov}. This separates variance of the estimator into the variance of each step of the process. Here, \(v_{i,B}\) is the covariance while crossing transition \(i\), and it is divided by the population propagated through that transition. Increasing the population before one transition therefore reduces variance created there and later. To get insight into the optimal allocation, we first consider the simpler problem of minimizing the variance at one threshold \(x\) for a fixed budget \(B\).  Write
\[
    \widetilde v_k
    =
    v_{k,B}(x,x)+\one_{\{k=0\}}\Gamma_{0,B}(x,x),
    \qquad
    n_k=N_{0,B}R_{k,B}.
\]
Here, \(n_k\) is the number of particles propagated through transition \(k\). We consider a relaxation where the \(n_k\) can be positive reals. The fixed-budget variance minimization problem is then
\[
    \min_{n\in\R_{>0}^{K_B}}
    \sum_{k=0}^{K_B-1}\frac{\widetilde v_k}{n_k}
    \quad
    \text{s.t.}
    \quad
    \sum_{k=0}^{K_B-1}n_k=B,
    \qquad
    n_0\le n_1\le \cdots \le n_{K_B-1}.
\]
The monotonicity constraint is handled by using Pool-Adjacent-Violators Algorithm (PAVA) \citep[Chapter~1]{RobertsonWrightDykstra1988}. See Appendix~\ref{app:pava} for a detailed discussion and construction of PAVA.

\begin{lemma}[Variance Optimization with Monotone Allocation]\label{lem:var-opt}
    Fix \(x\in\R^d\), and assume \(\widetilde v_k>0\) for all \(k\). The unique minimizer is
    \[
        n_k^*
        =
        B
        \frac{\sqrt{\widehat v_k}}
        {\sum_{j=0}^{K_B-1}\sqrt{\widehat v_j}}, \qquad \widehat v = \PAVA(\widetilde v).
    \]
\end{lemma}

The proof is given in Appendix~\ref{app:proof-var-opt}. The key idea here is that \( n_k \) follow a Neyman allocation with respect to the variances, after accounting for the monotonicity constraint. This solves the problem for one threshold, whereas a terminal-law estimator has to account for all thresholds, and a single allocation must control the entire CDF together. We develop theory for this in the next section.

\section{Asymptotic Analysis}
\label{sec:asymptotic-analysis}

This section develops the asymptotic theory of the splitting estimator in three steps: a continuous-time covariance limit, a functional central limit theorem, and a bias-variance tradeoff.

\subsection{Covariance Convergence}
\label{sec:covariance-asymptotic-convergence}

In this subsection, our goal is to replace the covariance formula from Lemma~\ref{lem:splitting-cov} with a continuous-time counterpart. Under suitable assumptions, we identify the continuous-time limit of \(v_i\). We then provide an expression analogous to Lemma~\ref{lem:splitting-cov}. For \(x\in\R^d\), write \( a_x(z)=\one(z\le x)\) and let
\[
    \sF_t=\sigma(X_0)\vee\sigma(W_s:0\le s\le t),
    \qquad 0\le t\le T,
\]
be the natural filtration of the SDE. Define the cumulative splitting function
\[
    r_B(t)=R_{i,B},
    \qquad t\in[t_{i,B},t_{i+1,B}).
\]
\begin{assumption}[Strong terminal convergence]\label{ass:terminal-strong-consistency}
The chain \(X^B\) is driven by \(X_0\) and \(W\), \(X^B\) and \(X\) are Markov
with respect to \((\sF_{t_{i,B}})_i\) and \((\sF_t)_t\), and
\(\|X^B_{K_B}-X_T\|_{L^2}\to0\).
\end{assumption}
Assumption~\ref{ass:terminal-strong-consistency} ensures consistency of the discretization scheme and assumes the Markov property for \(X\). It is satisfied by the Euler-Maruyama scheme under global Lipschitz and linear-growth conditions on the drift and diffusion coefficients; see \citet[Theorem 10.2.2]{KloedenPlaten1992}.
\begin{assumption}[Splitting-function limit]\label{ass:splitting-function}
The splitting functions satisfy \(r_B(0+)=1\), and each \(r_B\) is nondecreasing. Moreover,
\[
    r_B
    \longrightarrow
    r
    \quad\text{in }L^1([0,T]),
    \qquad r(0+)=1.
\]
\end{assumption}
Since the splitting function is a design choice, many irregular and feasible splitting functions can be constructed. Assumption~\ref{ass:splitting-function} is a regularity condition which restricts the class of functions permitted.

\begin{assumption}[Continuous terminal CDF]\label{ass:terminal-weak-conv}
The terminal CDF \(F(x) = \Pbb(X_T\le x)\) is continuous.
\end{assumption}

We begin by identifying the continuous-time limit of the covariance of the splitting estimator. As we saw in the discrete case, \(u^B_{i,x}\) was central to our arguments, and the continuous-time analog of it is
\[
    u_x(t,z)=\E\left[a_x(X_T)\mid X_t=z\right], \qquad u_x(t,X_t)=\E\left[a_x(X_T) \mid \sF_t\right].
\]
Since \(a_x(X_T)\) is bounded, \(u_x(t,X_t)\) is a bounded square-integrable martingale. Applying martingale representation conditionally on \(X_0\), there is a predictable process \(H_x\) such that
\[
    u_x(t,X_t)
    =
    u_x(0,X_0)+\int_0^t H_x(\tau)^\top dW_\tau,
    \qquad 0\le t\le T.
\]
The following lemma shows that the entire covariance structure of the SDE is captured by the predictable process \(H_x\) and identifies the continuous-time analog of the one-step covariances \(v_{i,B}\).

\begin{lemma}
\label{lem:continuous-step-covariance}
For \(x,y\in\R^d\), define, for almost every \(t\in[0,T]\),
\[
    g_{x,y}(t)
    :=
    \E\left[H_x(t)^\top H_y(t)\right].
\]
Fix \(B\) and \(0\le i\le K_B-1\), and set \(s=t_{i,B}\) and \(t=t_{i+1,B}\). Then the true SDE's covariance contribution over the interval \([s,t]\) is equal to the integral of \(g_{x,y}\) over that interval, i.e.,
\[
    v^{\mathrm{SDE}}_{i,B}(x,y)
    :=
    \E\left[
        \Cov\!\left(
            u_x(t,X_t),
            u_y(t,X_t)
            \mid \sF_s
        \right)
    \right]
    = \int_s^t g_{x,y}(\tau)\,d\tau.
\]
Moreover, \(g_{x,y}\in L^1([0,T])\).
\end{lemma}

The proof is given in Appendix~\ref{app:proof-continuous-step-covariance}. Here, \(g_{x,y}(t)\) is interpreted as the covariance density at time \(t\) for the SDE, and represents the rate at which covariance is being changed at time \(t\). As in the discrete case, let the initial variance be
\[
    \Gamma_0(x,y)
    :=
    \Cov\!\left(
        u_x(0,X_0),
        u_y(0,X_0)
    \right).
\]
The martingale representation separates these two orthogonal sources of randomness, giving
\[
    \Cov(a_x(X_T),a_y(X_T))
    =
    \Gamma_0(x,y)
    +
    \int_0^T g_{x,y}(t)\,dt.
\]

\begin{remark}
If, in addition, \(u_x \in C^{1,2}\), and hence sufficiently smooth for It\^o's formula to apply, then
\[
    H_x(t)=\sigma(t,X_t)^\top \nabla u_x(t,X_t),
\]
and therefore we can identify
\[
    g_{x,y}(t)
    =
    \E\left[
        \nabla u_x(t,X_t)^\top
        \sigma(t,X_t)\sigma(t,X_t)^\top
        \nabla u_y(t,X_t)
    \right].
\]
\end{remark}

Having identified the covariance density, we now show that the discrete one-step covariances converge to it. Using the splitting function defined above, define the scaled covariance kernel
\[
    \Gamma_B(x,y)
    :=
    N_{0,B}\Cov(\widehat F_B(x),\widehat F_B(y))
    =
    \Gamma_{0,B}(x,y)
    +
    \sum_{i=0}^{K_B-1}\frac{v_{i,B}(x,y)}{R_{i,B}}.
\]

\begin{proposition}
\label{prop:covariance-convergence}
Suppose Assumptions~\ref{ass:terminal-strong-consistency} and~\ref{ass:terminal-weak-conv} hold. Then, for every
\(x,y\in\R^d\),
\[
    \Gamma_{0,B}(x,y)\to \Gamma_0(x,y),
    \quad
    \text{and}
    \quad
    \sum_{i=0}^{K_B-1}
    \left|
        v_{i,B}(x,y)-v^{\mathrm{SDE}}_{i,B}(x,y)
    \right|
    \to 0.
\]
\end{proposition}

The proof is given in Appendix~\ref{app:proof-covariance-convergence}. This shows that each \(v_{i,B}\) and the initial variance converge to their continuous-time analogs. However, this does not establish a covariance limit for the splitting estimator.

\begin{proposition}[Covariance limit]\label{prop:covlimit}
Suppose Assumptions \ref{ass:terminal-strong-consistency}, \ref{ass:splitting-function}, and \ref{ass:terminal-weak-conv} hold.  Then, for every \(x,y\in\R^d\),
\[
    \Gamma_B(x,y)\to\Gamma(x,y)
    =
    \Gamma_0(x,y)
    +
    \int_0^T \frac{g_{x,y}(t)}{r(t)}\,dt.
\]
Moreover, \(\Gamma\) is positive semidefinite and
\[
    d_\Gamma(x,y)^2
    :=
    \Gamma(x,x)+\Gamma(y,y)-2\Gamma(x,y)
    \le \left\|a_x(X_T)-a_y(X_T)\right\|_{L^2}^2.
\]
\end{proposition}

The proof is given in Appendix~\ref{app:proof-covlimit}. This is the continuous-time analog of Lemma~\ref{lem:splitting-cov} and shows that the covariance of the splitting estimator converges to a natural limit.

\subsection{Splitting Central Limit Theorem}
\label{sec:gaussian-limit-splitting-noise}

This subsection shows that the splitting estimator obeys a central limit theorem. In particular, we have a triangular-array empirical process, where the entire simulation depends on \(B\). Note that at budget \(B\), there are \(N_{0,B}\) independent root particles, and each root generates many dependent samples. Hence, the natural scaling occurs with \(\sqrt{N_{0,B}}\). Let
\[
    \cA=\left\{a_x:x\in\R^d\right\}
\]
be the lower-orthant class. Define
\[
    \rho(a_x,a_y)
    :=
    \left\|a_x(X_T)-a_y(X_T)\right\|_{L^2},
    \qquad a_x,a_y\in\cA.
\]
For notational convenience, let \(m_B=R_{K_B-1,B}\) be the number of terminal descendants generated by one root, and let
\[
    E_B=(\R^d)^{m_B},
    \qquad
    Z_B=\left(Z_B^{(1)},\ldots,Z_B^{(m_B)}\right)\in E_B,
\]
where, \(E_B\) is the sample space of the terminal descendants from one root, and \(Z_B\) is the corresponding random vector. For \(z\in E_B\), define
\[
    f_{B,x}(z)
    =
    \frac1{m_B}
    \sum_{\ell=1}^{m_B}a_x(z^{(\ell)}),
    \qquad
    \cF_B=\left\{f_{B,x}:x\in\R^d\right\}.
\]
For each fixed \(B\), the empirical process is naturally computed on \(\cF_B\) over the sample space \(E_B\). Although \(E_B\) and \(\cF_B\) change with \(B\), we label the coordinate corresponding to \(f_{B,x}\) by \(a_x\) and regard the process as indexed by the fixed class \(\cA\). We are now ready to state the central limit theorem.

\begin{theorem}[Splitting CLT]\label{thm:splitting-clt}
Under Assumptions \ref{ass:terminal-strong-consistency}, \ref{ass:splitting-function}, and \ref{ass:terminal-weak-conv}, for \(a_x\in\cA\), define
\[
    \G_B(a_x)
    =
    \sqrt{N_{0,B}}\left(\widehat F_B(x)-F_B(x)\right).
\]
If \(N_{0,B} \to \infty\), then
\[
    \G_B
    \Rightarrow
    \G
    \quad\text{in }\ell^\infty(\cA),
    \qquad \Cov\!\left(\G(a_x),\G(a_y)\right)=\Gamma(x,y).
\]
Here, \(\G\) is a centered Gaussian process indexed by \(\cA\) with covariance \(\Gamma\) defined in Proposition~\ref{prop:covlimit}.
\end{theorem}

The proof is given in Appendix~\ref{app:proof-splitting-clt}. This shows us that the covariance of the Gaussian process is dependent on the splitting function \(r\), and hence it completely specifies the limiting distribution.

\subsection{Bias-Variance Tradeoff}
\label{sec:bias-variance-tradeoff}

The CLT above describes only the simulation noise around \(F_B\). To obtain the error around \(F\), we need to account for the bias introduced due to discretization. We follow the approach of \citet{DuffieGlynn1995} to balance the bias and variance, where the variance and bias are modeled as functions of the budget \(B\). For a given \(B\), \(N_{0,B}\) depends on the cost of one root tree as follows
\[
    C_B = B/N_{0,B} = \sum_{i=0}^{K_B-1}R_{i,B}.
\]
Consider an asymptotically uniform time partition, in the sense that
\[
    \max_{0\le i<K_B}
    \left|
        \frac{t_{i+1,B}-t_{i,B}}{h_B}-1
    \right|
    \longrightarrow0.
\]
By Assumption~\ref{ass:splitting-function},
\[
    h_BC_B
    =
    \sum_{i=0}^{K_B-1}h_BR_{i,B}
    \longrightarrow
    \int_0^T r(t)\,dt
    =:
    \gamma.
\]
Here, \(\gamma\) is finite by Assumption~\ref{ass:splitting-function} and is interpreted as the limiting cost of one root tree. Hence
\[
    \frac{B}{N_{0,B}}
    =
    \frac{\gamma}{h_B}
    +
    o(h_B^{-1}).
\]
The functional CLT concerns \(\widehat F_B-F_B\). To recover the error relative to \(F\), write
\[
    \widehat F_B-F
    =
    (\widehat F_B-F_B)+(F_B-F).
\]
Since \(N_{0,B}\sim Bh_B/\gamma\), the CLT identifies
\[
    \left(\frac{\gamma}{Bh_B}\right)^{1/2}
\]
as the natural scale of the centered simulation error. For the discretization error \(F_B-F\), we assume only the following polynomial bound to provide a rate for the total error.

\begin{assumption}[Weak discretization error]\label{ass:bias}
There is \(p>0\) such that \(\|F_B-F\|_\infty=O(h_B^p)\).
\end{assumption}
For the Euler scheme, \citet{BallyTalay1996} prove weak error of order one for a broad class of test functions, including indicator functions, under suitable regularity conditions. This motivates the important case \(p=1\).

\begin{corollary}\label{cor:dg-limit}
Suppose the assumptions of Theorem~\ref{thm:splitting-clt} and Assumption~\ref{ass:bias} hold. Assume that the time partition is asymptotically uniform. Put
\[
    \lambda_B=h_BB^{1/(1+2p)}.
\]
If \(\lambda_B\to \lambda\in(0,\infty)\), then
\[
    B^{p/(1+2p)}(\widehat F_B-F_B)
    \Rightarrow
    \left(\frac{\gamma}{\lambda}\right)^{1/2}\G
    \qquad\text{in }\ell^\infty(\cA).
\]
Moreover,
\[
    B^{p/(1+2p)}
    \|\widehat F_B-F\|_\infty
\]
is bounded in probability.
\end{corollary}

The proof is given in Appendix~\ref{app:proof-dg-limit}. Assumption~\ref{ass:bias} controls only the magnitude of \(F_B-F\), and only gives a rate for the total error rather than a weak limit. For the remainder of the paper, we consider the non-degenerate regime \(\lambda_B\to \lambda\in(0,\infty)\), which happens when \( h_B \propto B^{-1/(1+2p)}\).

\begin{remark}
    We retain the asymptotic uniform time-partition requirement in the main text for simplicity. More generally, a nonuniform discretization can be handled by assigning different costs to different transitions. Appendix~\ref{app:cost-weighted} shows that this results in a limiting cost density \(\omega\), and that the corollary above then holds with \(\gamma\) replaced by the cost-weighted \(\gamma_\omega\).
\end{remark}

\section{KS Statistic and Minimax Allocation}
\label{sec:ks-statistic}

The Kolmogorov-Smirnov (KS) distance from the target law is
\[
    \KS(\widehat F_B,F)
    =
    \sup_{x\in\R^d}
    |\widehat F_B(x)-F(x)|.
\]
Since \(F_B-F\) is independent of the splitting function \(r\), the allocation can affect only \(\widehat F_B-F_B\). We therefore use the limiting KS loss of the centered simulation error as the allocation criterion. By the continuous mapping theorem,
\[
    B^{p/(1+2p)}
    \KS(\widehat F_B,F_B)
    \Rightarrow
    \left(\frac{\gamma}{\lambda}\right)^{1/2}
    \sup_{x\in\R^d}|\G(a_x)|.
\]
We are interested in minimizing this quantity. However, this is generally intractable. Fortunately, we can upper and lower bound its expected value by bounds that differ up to a logarithmic factor. These bounds then motivate a tractable allocation problem.

\begin{proposition}[Expected Limiting Centered KS Bounds]\label{prop:ks-bounds}
Let
\[
    \Gamma^*
    =
    \sup_{x\in\R^d}\sqrt{\Gamma(x,x)}.
\]
If \(\Gamma^*>0\), then there exists a constant \(C<\infty\), depending only on \(d\), such that
\[
    \sqrt{\frac{2}{\pi}}
    \left(\frac{\gamma}{\lambda}\right)^{1/2}
    \Gamma^*
    \le
    \left(\frac{\gamma}{\lambda}\right)^{1/2}
    \E\left[
        \sup_{x\in\R^d}|\G(a_x)|
    \right]
    \le
    C
    \left(\frac{\gamma}{\lambda}\right)^{1/2}
    \Gamma^*
    \sqrt{1-\log(\Gamma^*)}.
\]
\end{proposition}

The proof is given in Appendix~\ref{app:proof-ks-bounds}. This proposition shows that for a fixed \( \lambda\), the bias term does not depend on the splitting design, and we can limit our attention to the variance term. Notice that the bounds agree up to a logarithmic factor. Dropping this factor gives the surrogate
\[
    \gamma(r)^{1/2}\Gamma^*.
\]
Since the surrogate is nonnegative, minimizing its square gives the proxy objective
\[
    \inf_{r\in\mathcal R}
    \gamma(r)
    \sup_{x\in\R^d}
    \left[
        \Gamma_0(x,x)+\int_0^T \frac{g_{x,x}(t)}{r(t)}\,dt
    \right],
\]
where \(\mathcal R\) is the set of nondecreasing functions \(r\in L^1([0,T])\) with \(r(0+) = 1\). This problem has an elegant interpretation: we are minimizing the worst-case variance of the Gaussian process over all thresholds \(x\) by choosing the splitting function \(r\). To solve this optimization problem, let
\[
    \mathcal X=[-\infty,\infty]^d
\]
have the product topology. By Lemma~\ref{lem:threshold-continuity}, \(x\mapsto\Gamma_0(x,x)\) is continuous on \(\mathcal X\), and \(x\mapsto g_{x,x}\) is continuous from \(\mathcal X\) into \(L^1([0,T])\). For fixed \(r\in\mathcal R\), we have \(1/r\le1\). Thus
\[
    x\longmapsto
    \Gamma_0(x,x)
    +
    \int_0^T\frac{g_{x,x}(t)}{r(t)}\,dt
\]
is continuous on \(\mathcal X\). Since \(\R^d\) is dense in \(\mathcal X\), compactification does not change the supremum. Inspired by Lemma~\ref{lem:var-opt}, we first describe the continuous-time analog of PAVA. For any integrable function \(f: [0,T]\to\R\), let \(C[f]\) be the greatest convex minorant of \(t\mapsto\int_0^t f(s)\,ds\) on \([0,T]\). Then, the continuous-time PAVA of \(f\) is defined as the right derivative of \(C[f]\), that is \(\PAVA(f) = C[f]'_+\).

Now, note that the threshold that maximizes the variance depends on the allocation, and hence they need to be controlled jointly. This can be thought of as a game between the designer and an adversary, where the adversary chooses weights on the thresholds to maximize the variance. We therefore introduce mixtures over thresholds. For \(\pi\in\mathcal P(\mathcal X)\), define
\[
    \Gamma_0^\pi
    =
    \int_{\mathcal X}\Gamma_0(x,x)\,\pi(dx),
    \qquad
    g^\pi
    =
    \int_{\mathcal X}g_{x,x}\,\pi(dx)
    \quad\text{in }L^1([0,T]).
\]
These integrals are well defined by
Lemma~\ref{lem:threshold-continuity}.
The PAVA construction above can be extended to include an atom \(\Gamma_0^\pi\) at \(t=0\). See Appendix~\ref{app:pava} for the complete construction. Applying this to the pair \((\Gamma_0^\pi,g^\pi)\), we get
\[
    h_\pi
    =
    \PAVA(\Gamma_0^\pi,g^\pi).
\]
We solve this problem by replacing the supremum over \(x\) with a maximization over probability distributions \(\pi\) on \(\mathcal X\), which leaves its value unchanged. For \(r\in\mathcal R\), set \(n=r/\gamma(r)\). The objective is convex in \(n\) and affine in \(\pi\). Under a suitable topology, \(\mathcal P(\mathcal X)\) is compact. We then use Sion's minimax theorem and the PAVA construction above to conclude. Appendix~\ref{app:proof-functional-minimax-opt} gives the details.

\begin{theorem}[Minimax Allocation]
\label{thm:functional-minimax-opt}
Suppose Assumption~\ref{ass:terminal-weak-conv} holds. Then
\[
    \inf_{r\in\mathcal R}
    \gamma(r)
    \sup_{x\in\R^d}
    \left[
        \Gamma_0(x,x)
        +
        \int_0^T \frac{g_{x,x}(t)}{r(t)}\,dt
    \right]
    =
    \max_{\pi\in\mathcal P(\mathcal X)}
    \left(
        \int_0^T\sqrt{h_\pi(t)}\,dt
    \right)^2.
\]
The maximum on the right is attained. For any maximizer \(\pi^*\), the infimum on the left is attained if and only if \(h_{\pi^*}(0+)>0\). In that case,
\[
    r^*(t)
    =
    \sqrt{
        \frac{h_{\pi^*}(t)}
             {h_{\pi^*}(0+)}
    },
    \qquad 0\le t<T,
\]
defines a minimizer.
\end{theorem}

The proof is given in Appendix~\ref{app:proof-functional-minimax-opt}. Notably, the optimal allocation still follows Neyman allocation, but now with respect to the PAVA projection of a mixture of covariance densities.

Since simulation is performed with a finite budget, we also provide a finite-budget counterpart to Theorem~\ref{thm:functional-minimax-opt}, which uses the same argument for the minimax version of Lemma~\ref{lem:var-opt}. Put \(m=K_B\) and
\[
    \widetilde v_i(x)
    =
    v_{i,B}(x,x)+\one_{\{i=0\}}\Gamma_{0,B}(x,x),
    \qquad 0\le i<m.
\]
With \(n_i=N_0R_i\), define the feasible allocation set
\[
    \mathcal N
    =
    \left\{
        n\in\R_{>0}^m:
        n_0\le\cdots\le n_{m-1},
        \ \sum_{i=0}^{m-1}n_i=B
    \right\}.
\]
To account for the coupling across thresholds, consider the closed convex hull of the variances
\[
    \widetilde V(x)
    =
    \big(\widetilde v_0(x),\ldots,\widetilde v_{m-1}(x)\big),
    \qquad
    \mathcal V
    =
    \overline{\operatorname{conv}}\left\{\widetilde V(x):x\in\R^d\right\}
    \subset \R_+^m.
\]

\begin{corollary}[Finite-Budget Minimax Allocation]\label{cor:minimax-var-opt}
    Assume that there exists \(x_0\in\R^d\) such that \(\widetilde v_i(x_0)>0\) for every \(i\). For \(a\in\mathcal V\), let \(\widehat a=\PAVA(a)\). Then the minimum below is attained, and
    \[
        \min_{n\in\mathcal N}
        \sup_{x\in\R^d}
        \sum_{i=0}^{m-1}\frac{\widetilde v_i(x)}{n_i}
        =
        \frac1B
        \max_{a\in\mathcal V}
        \left(\sum_{i=0}^{m-1}\sqrt{\widehat a_i}\right)^2.
    \]
    Moreover, if \(a^*\) maximizes the right-hand side and \(h^*=\PAVA(a^*)\), then \(h_i^*>0\) for every \(i\), and every minimizing allocation satisfies
    \[
        n_i^*
        =
        B
        \frac{\sqrt{h_i^*}}
        {\sum_{j=0}^{m-1}\sqrt{h_j^*}},
        \qquad 0 \le i \le m - 1.
    \]
\end{corollary}

The proof is given in Appendix~\ref{app:proof-minimax-var-opt}.

We emphasize that the uniform cost and time partition assumptions are made here to present the results in a simple form. Under a single regularity condition on the costs, both results extend by measuring time in cost rather than in elapsed time. Appendix~\ref{app:cost-weighted} gives the cost-weighted forms of Theorem~\ref{thm:functional-minimax-opt} and Corollary~\ref{cor:minimax-var-opt}.

\section{Rounding}
\label{sec:rounding}

For a fixed \(B\), let \(n^*\) be an optimal allocation in Corollary~\ref{cor:minimax-var-opt}. This need not define a valid tree because it can have noninteger entries. Hence, we approximate it with a mixture of trees. Set \(m=K_B\) and let \(r_i^* = n_i^*/n_0^*\) be the normalized allocation. Define the relaxed normalized variance
\[
    V^*
    =
    \left(\sum_{i=0}^{m-1}r_i^*\right)
    \sup_{x\in\R^d}
    \sum_{i=0}^{m-1}
    \frac{\widetilde v_i(x)}{r_i^*}.
\]

Given \(r^*\), we first construct a finite set of dyadic trees following the construction of \citet[Section~2]{Roundy1989}. Let \(U \in [0,1)\). For each value of \(U\), associate a tree with profile
\[
    R_i(U)
    =
    2^{\lfloor\log_2r_i^*+U\rfloor},
    \qquad
    0\le i<m.
\]
Since \(r^*\) is nondecreasing, the exponents are also nondecreasing, so every ratio \(R_i(U)/R_{i-1}(U)\) is a nonnegative power of two and every \(R(U)\) defines a valid tree. As \(U\) varies over \([0,1)\), each coordinate can change at most once. Hence, at most \(m\) distinct trees occur, which are our candidates.

We now construct a mixture of these trees. Let \(R^{(1)},\ldots,R^{(S)}\) be the distinct profiles, where \(S\le m\). Suppose we generate \(M_s\) independent trees of type \(s\) to construct the empirical CDF \(\widehat F_{B,s}\), and assign weights \(\lambda_s\) to each type with \(\sum_s\lambda_s=1\). We define the mixture CDF to be
\[
    F_B^\lambda
    =
    \sum_{s=1}^S
    \lambda_s\widehat F_{B,s}.
\]
Since each type is unbiased, the mixture is unbiased. From Lemma~\ref{lem:splitting-cov}, we have for every \(x,y\in\R^d\),
\[
    \Cov\!\left(F_B^\lambda(x),F_B^\lambda(y)\right)
    =
    \sum_{s=1}^S\frac{\lambda_s^2}{M_s}
    \left[
        \Gamma_{0,B}(x,y)
        +
        \sum_{i=0}^{m-1}\frac{v_{i,B}(x,y)}{R_i^{(s)}}
    \right].
\]
Here, \(\lambda_s\) is interpreted as the budget share assigned to type \(s\). If \(C_s\) is the cost of generating one tree of type \(s\), we have \(M_s=\lambda_sB/C_s\). We choose these weights to minimize the worst-case variance and show that the resulting mixture approximates the relaxed optimum up to a small constant factor.

\begin{proposition}[Dyadic Mixture Approximation]
\label{prop:dyadic-mixture-approximation}
If the candidate types are generated from \(r^*\), then
\[
    V^*
    \le
    \min_{\substack{\lambda_s\ge0\\ \sum_s\lambda_s=1}}
    \sup_{x\in\R^d}
    \sum_{s=1}^S
    \lambda_sC_s
    \sum_{i=0}^{m-1}
    \frac{\widetilde v_i(x)}{R_i^{(s)}}
    \le
    \frac{2}{e\log 2}V^*.
\]
\end{proposition}

The proof is given in Appendix~\ref{app:proof-dyadic-mixture-approximation}. We next show that the earlier asymptotic analysis extends to the resulting mixture. Write \(C_{s,B}=\sum_{i=0}^{K_B-1}R_{i,B}^{(s)}\). For each \(B\), let \(\lambda_B=(\lambda_{1,B},\ldots,\lambda_{S_B,B})\) be an optimizer of the problem in Proposition~\ref{prop:dyadic-mixture-approximation}. The optimizer need not be unique, and the number of candidate types \(S_B\) can grow with \(B\). Instead, we analyse the normalized mixture cost, which is defined as
\[
    \overline\gamma_B
    =
    h_B\sum_{s=1}^{S_B}\lambda_{s,B}C_{s,B},
\]
and, for \(t\in[t_{i,B},t_{i+1,B})\), define the effective splitting function
\[
    \overline r_B(t)
    =
    \frac{
        \sum_{s=1}^{S_B}\lambda_{s,B}C_{s,B}
    }{
        \sum_{s=1}^{S_B}
        \lambda_{s,B}C_{s,B}/R_{i,B}^{(s)}
    }.
\]
The function \(\overline r_B\) is nondecreasing and satisfies \(\overline r_B(0+)=1\). To obtain integer root counts, set
\[
    M_{s,B}
    =
    \left\lfloor
        \frac{B\lambda_{s,B}}{C_{s,B}}
    \right\rfloor
\]
For type \(s\), let \(\widehat F_{B,s}\) be the empirical CDF formed from its \(M_{s,B}\) root. The mixture estimator is then
\[
    \widehat F_B^{\mathrm{mix}}
    =
    \frac{
        \sum_{s=1}^{S_B}
        M_{s,B}C_{s,B}\widehat F_{B,s}
    }{
        \sum_{s=1}^{S_B}M_{s,B}C_{s,B}
    }.
\]
The candidate profiles and weights are fixed before simulation, and all roots are simulated independently within and across candidate types.

\begin{proposition}[Mixture Functional Limit]
\label{prop:mixture-functional-limit}
Suppose Assumptions~\ref{ass:terminal-strong-consistency} and~\ref{ass:terminal-weak-conv} hold, the time partitions are asymptotically uniform, and for some \(\overline\gamma\in(0,\infty)\) and \(\overline r\in\mathcal R\),
\[
    Bh_B^2\longrightarrow\infty,
    \qquad
    \overline\gamma_B\longrightarrow\overline\gamma,
    \qquad
    \overline r_B\longrightarrow\overline r
    \quad\text{in }L^1([0,T]).
\]
Then
\[
    \sqrt{Bh_B}
    \left(
        \widehat F_B^{\mathrm{mix}}-F_B
    \right)
    \Rightarrow
    \sqrt{\overline\gamma}\,\G
    \qquad
    \text{in }\ell^\infty(\cA),
\]
where \(\G\) is a centered Gaussian process with covariance
\[
    \Cov\!\left(\G(a_x),\G(a_y)\right)
    =
    \Gamma_0(x,y)
    +
    \int_0^T
    \frac{g_{x,y}(t)}{\overline r(t)}\,dt.
\]
Moreover,
\[
    \sqrt{Bh_B}
    \KS\!\left(\widehat F_B^{\mathrm{mix}},F_B\right)
    \Rightarrow
    \sqrt{\overline\gamma}
    \sup_{x\in\R^d}
    \left|\G(a_x)\right|.
\]
\end{proposition}

The proof is given in Appendix~\ref{app:proof-mixture-functional-limit}. The condition \(Bh_B^2\to\infty\) makes the budget lost to flooring negligible. Under the nondegenerate scaling \(h_B\propto B^{-1/(1+2p)}\), this holds whenever \(p>1/2\), including the usual Euler case \(p=1\).

The individual optimized weights need not converge, and a candidate type need not receive a growing number of roots. The earlier covariance, functional CLT, and KS limits centered at \(F_B\) carry over with \((\gamma,r)\) replaced by \((\overline\gamma,\overline r)\). Under Assumption~\ref{ass:bias}, the bias-variance analysis carries over as well. Hence, the relaxed allocation can be approximated by a mixture of valid dyadic trees while retaining the same form of the asymptotic limit theory.

\section{Numerical Experiments}
\label{sec:numerical-experiments}

The optimal allocation is determined by the variance contributions \(v_i(x,x)\) and \(\Gamma_{0,B}(x,x)\). In practice they are not known. Hence, we propose the following two-phase procedure.

\subsection{Two-Phase Algorithm}
\label{sec:two-phase-algorithm}

Before the simulation begins, we choose a small set of split times
\[
    0=t_0<t_1<\cdots<t_{L-1}<t_L=T,
\]
where $t_1,\ldots,t_{L-1}$ are the times at which splitting may occur. Particles follow the original dynamics inside a segment. Phase~1 of the algorithm uses $B_1\propto B^\alpha$ budget with $\alpha<1$ to simulate complete paths without splitting. For KS, from these Phase~1 paths we fit small MLPs that learn
\[
    \widehat u_x(t,X_t)
    \approx
    u_x(t,X_t)
    =
    \Pbb(X_T\le x\mid X_t).
\]
We convert these predictions into estimates of $\widetilde v_i$ through the second moments
\[
    Q_i(x)=\E[u_x(t_i,X_{t_i})^2].
\]
Then, since \(u_x(t_i,X_{t_i})\) is a martingale, we have the following expression for $\widetilde v_i$
\[
    \widetilde v_0(x)=Q_1(x)-F(x)^2,
    \qquad
    \widetilde v_i(x)=Q_{i+1}(x)-Q_i(x)\quad(i\ge1).
\]
We use these estimates in the allocation problem in Corollary~\ref{cor:minimax-var-opt}, where the supremum over all thresholds is approximated by a maximum over a finite set.

The resulting allocation is used to construct candidate dyadic trees, and the weights of the mixture are determined by solving a linear program, as detailed in Section~\ref{sec:rounding}. Phase~2 uses the remaining budget to run the selected trees. The final estimate combines the Phase~1 terminal samples with the weighted empirical distribution of the Phase~2 trees. See Appendix~\ref{app:allocation-estimation} for the procedure for KS and Appendix~\ref{app:mmd-procedure} for the procedure for MMD.

\subsection{Experimental Setup}
\label{sec:experiment-design}

We study five examples: (i) a two-dimensional Ornstein-Uhlenbeck (OU) process, (ii) two-dimensional overdamped Langevin dynamics, (iii) a stochastic EDM model trained on a two-dimensional Gaussian mixture, (iv) the pretrained \texttt{google/ddpm-cifar10-32} model, and (v) the pretrained \texttt{asparius/ldm-ffhq-256} latent diffusion model for FFHQ. For every example, we consider schedules with 9, 19 and 39 split times. Complete details of the models, their corresponding SDEs, and configurations are given in Appendix~\ref{app:benchmark-details}.

At each budget, we compare learned splitting with independent full-path simulation and the Uniform-\(c\) and Learned-\(c\) baselines, where we split by a factor of \(c\) each time. Refer to Appendix~\ref{app:c-baselines} for both baseline procedures. Due to compute constraints, we only test CIFAR-10 and FFHQ on Learned-\(c\). For the two-dimensional examples, we compute the exact KS distance of the empirical CDF to a fixed reference CDF. Computing the KS distance exactly is not practical for images, so for CIFAR-10 and FFHQ we use the MMD. To choose the allocation, Phase~1 computes the MMD in the native space of the sampler. This is the pixel space for CIFAR-10 and the latent space for FFHQ. To report results, we compute the MMD on image features from two pretrained networks, Inception-v3 \citep{Szegedy2016} and DINOv2 \citep{Oquab2023}. Appendix~\ref{app:mmd-theory} develops the theory for MMD and its allocation problem. See Appendix~\ref{app:evaluation-metrics} for evaluation details.

\subsection{Results}
\label{sec:experiment-results}

Figure~\ref{fig:metric-gain} shows the percentage reduction in the mean error metric relative to i.i.d. simulation, budget by budget. Figure~\ref{fig:allocation-profiles} shows the learned cumulative allocation profiles at the largest tested budget for each model. Refer to Appendix~\ref{app:complete-results} for complete numerical results.

\begin{figure}[h]
    \centering
    \begingroup
    \setlength{\fboxsep}{0pt}
    \fbox{\includegraphics[width=0.995\textwidth]{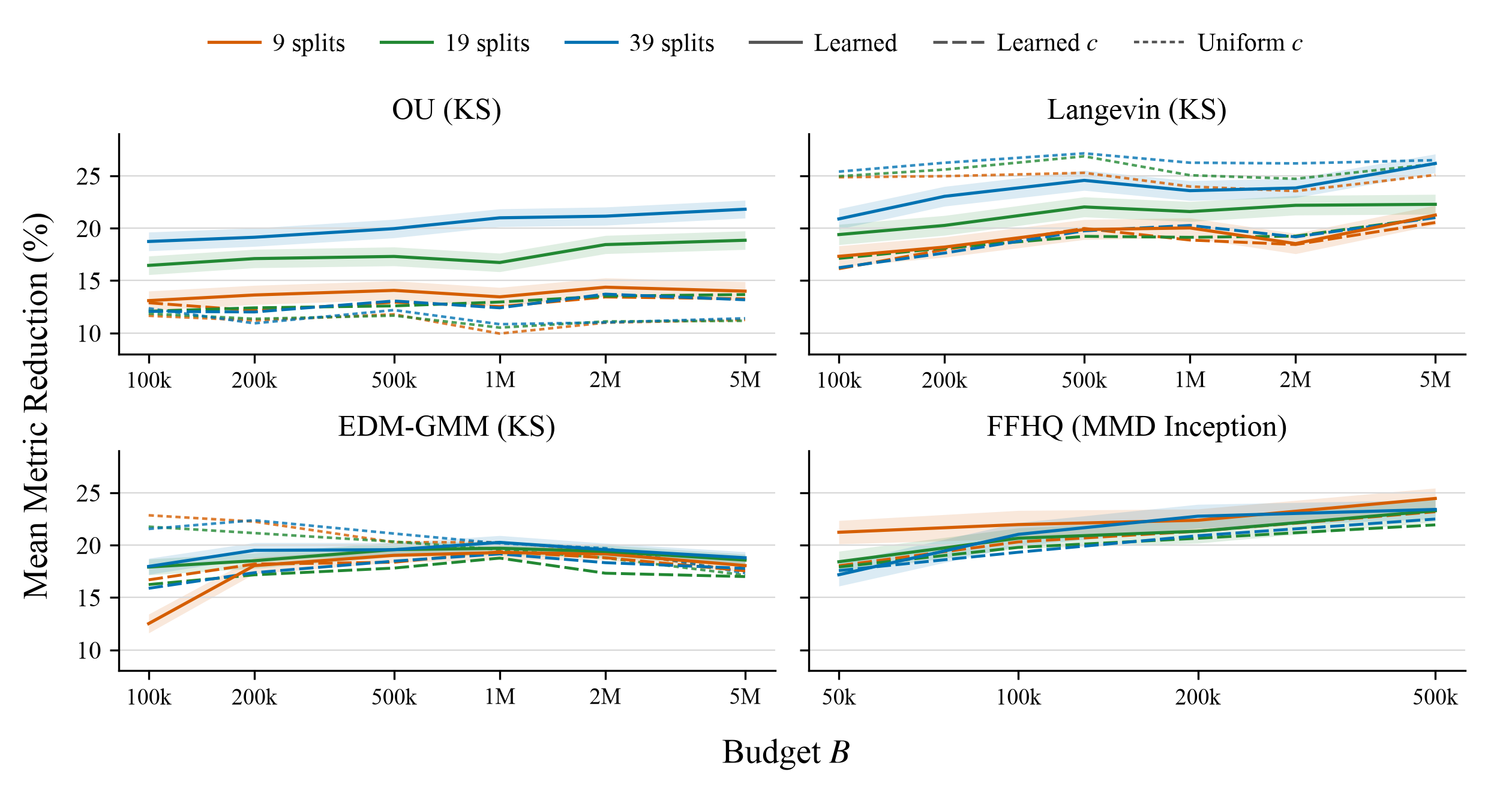}}
    \endgroup
    \caption{Percentage reduction in the relevant metric relative to independent full-path simulation for the learned allocation, Uniform-$c$, and Learned-$c$. For Uniform-$c$, the best-performing \(c\) is shown. The first three panels show mean KS over $2{,}500$ repetitions, and the FFHQ panel shows the mean Inception MMD over 50 repetitions. Shading shows 90\% normal intervals.}
    \label{fig:metric-gain}
\end{figure}

\begin{figure}[h]
    \centering
    \begingroup
    \setlength{\fboxsep}{0pt}
    \fbox{\includegraphics[width=0.995\textwidth]{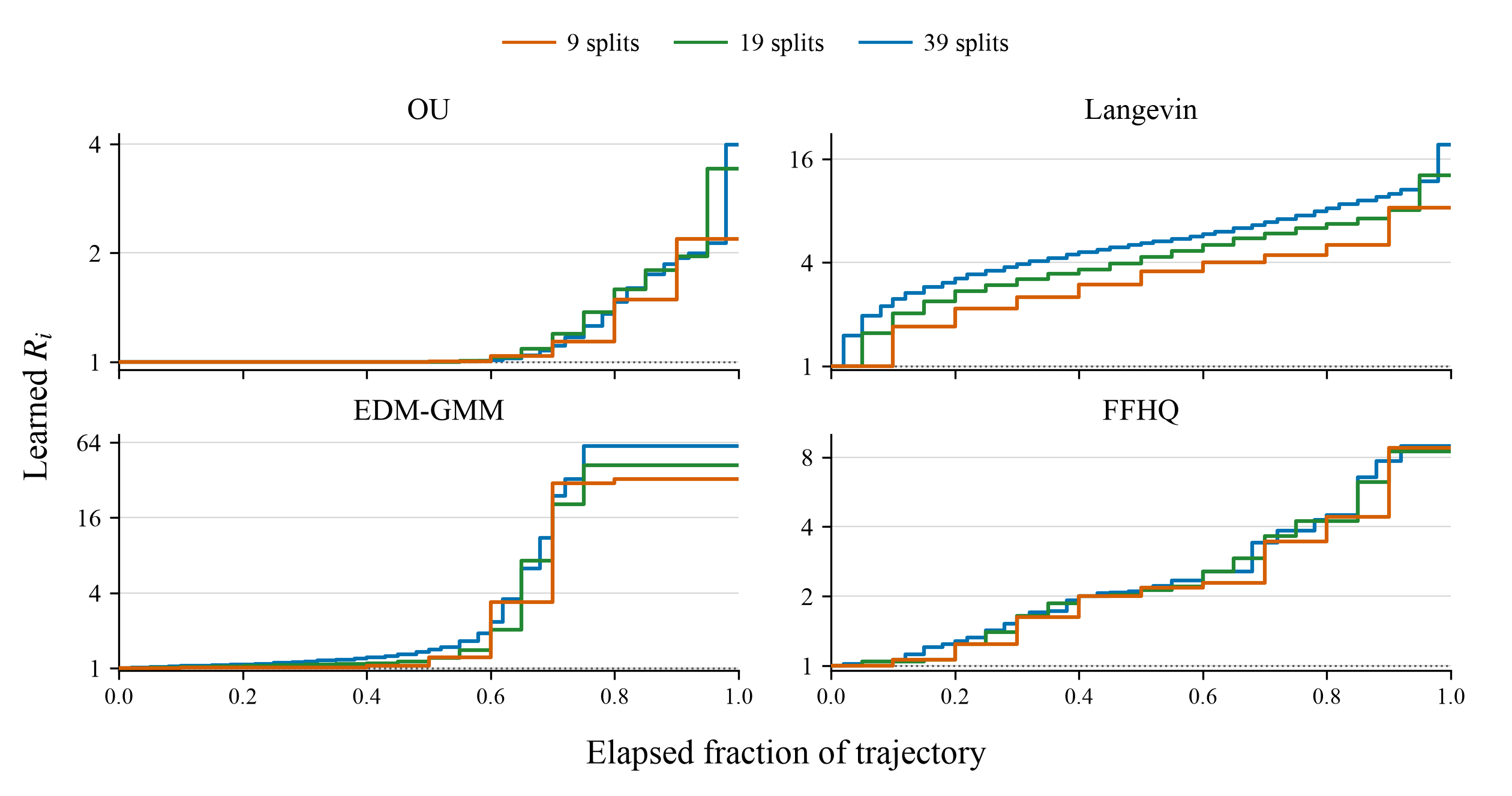}}
    \endgroup
    \caption{Allocation profiles $R_i$ for the learned allocation at each model's largest tested budget.}
    \label{fig:allocation-profiles}
\end{figure}

The learned allocation lowers the mean error in every tested setting. Reductions reach \(26\%\) on the two-dimensional examples. On CIFAR-10, they reach \(26\%\) in Inception features and \(24\%\) in DINOv2 features. On FFHQ, they reach \(24\%\) in both Inception and DINOv2 features. The baseline where \(c\) is chosen before running performs well. However, it requires a choice of \(c\) that is not known in advance, and some choices of \(c\) perform worse than i.i.d. simulation. The baseline where \(c\) is learned using a two-phase procedure also beats i.i.d. simulation in every case. It is close to the learned allocation in many of the models.

A key observation is that the learned allocation profiles show a non-trivial splitting behavior. For example, for the EDM model, the learned allocation only splits near the midpoint. This agrees with the intuition that most denoising, and hence higher variance, occurs in the middle of the trajectory. For FFHQ, the learned allocation places splits along the whole trajectory, with the most splits near the end.

We also measure the wall-clock time of the learned allocation relative to i.i.d. simulation. In settings where the simulation cost is large, for example CIFAR-10 and FFHQ, the learned allocation adds at most \(2.2\%\) to the running time at \(B=500\)k. Hence, the gains in the error metric are achieved with negligible additional cost. See Appendix~\ref{app:wall-clock} for the timing results.

To demonstrate the effectiveness of the two-phase procedure in small dimensions, we compare the learned allocation with an oracle for the OU process, where the variance contributions of the OU process can be computed analytically. Figure~\ref{fig:ou-oracle-allocations} shows that the learned allocation closely matches the oracle allocation. At \(B=5\times10^6\), the reduction from the learned allocation ranges from \(82.6\%\) to \(100.7\%\) of the reduction from the oracle across the three schedules. See Appendix~\ref{app:ou-oracle} for the construction and results.

\section{Conclusion}
\label{sec:conclusion}

In this paper, we propose a splitting estimator to estimate the terminal distribution of stochastic differential equations under a fixed simulation budget. We develop a general method that can be applied to a broad class of error measures. Focusing on the Kolmogorov-Smirnov distance and the maximum mean discrepancy, we do analysis to identify the optimal allocation. We also provide a practical two-phase procedure to implement the estimator, which is shown to be effective in extensive numerical experiments.

Our approach has several limitations. For the Kolmogorov-Smirnov distance, the upper and lower bounds differ by a logarithmic factor, and the design criterion ignores this gap. Moreover, some implementation details, such as the finite query set and estimated variance values, remain outside the theory, which does not account for these errors. The theory also does not cover the shrinkage step of the MMD procedure, or the image experiments, where the allocation is chosen with one MMD and reported with another. Another important extension is to consider conditional diffusion models, where the target distribution is conditioned on some input, and the goal is to either (i) learn a splitting allocation that is optimal for all inputs or (ii) learn a mapping from inputs to splitting allocations that is optimal for each input. We leave these extensions to future work.

\begingroup
\small
\setlength{\bibsep}{2pt plus 0.5pt minus 0.5pt}
\renewcommand{\bibsection}{\section*{References}\leavevmode\par\nobreak\vspace{0.25\baselineskip}}
\bibliographystyle{plainnat}
\bibliography{references}
\endgroup

\clearpage
\appendix
\startcontents[appendices]
\section*{Appendix Contents}
\begingroup
\small
\printcontents[appendices]{}{1}[2]{}
\endgroup

\section{Proofs and Auxiliary Results}
\label{app:proofs}

\subsection{PAVA and convex minorants}
\label{app:pava}

This subsection defines the discrete and continuous PAVA constructions used in Lemma~\ref{lem:var-opt} and Theorem~\ref{thm:functional-minimax-opt}. In both settings, PAVA replaces a nonnegative variance profile by a nondecreasing effective profile. We first describe the discrete construction and then define its continuous-time analogue. For \(z=(z_0,\ldots,z_{m-1})\in[0,\infty)^m\), define the equal-weight nondecreasing PAVA fit by
\[
    \widehat z=\PAVA(z)
    :=
    \arg\min_{\theta_0\le\cdots\le\theta_{m-1}}
    \sum_{k=0}^{m-1}(z_k-\theta_k)^2.
\]
The fit can be computed by starting with the singleton blocks \(\{0\},\ldots,\{m-1\}\). For a consecutive block \(I\), write
\[
    \overline z_I
    =
    \frac1{|I|}\sum_{k\in I}z_k.
\]
Whenever two adjacent blocks \(I\) and \(J\) satisfy \(\overline z_I>\overline z_J\), merge them and assign the combined block its average. Continue until the block averages are nondecreasing. If
\(I_1,\ldots,I_L\) are the final PAVA blocks, then
\[
    \widehat z_k=\overline z_{I_\ell},
    \qquad
    k\in I_\ell,\quad 1\le\ell\le L.
\]
The resulting fit is unique and does not depend on the order in which violations are pooled. This is the standard equal-weight PAVA construction \citep[Chapter~1]{RobertsonWrightDykstra1988}. The same fit is obtained from a convex minorant. Define
\[
    S_0=0,
    \qquad
    S_j=\sum_{i=0}^{j-1}z_i,
    \quad 1\le j\le m,
\]
and let \(C_z\) be the greatest convex minorant of the polygonal line through
\[
    (0,S_0),(1,S_1),\ldots,(m,S_m).
\]
Then
\[
    \widehat z_i=C_z(i+1)-C_z(i),
    \qquad 0\le i<m.
\]
Thus the final PAVA blocks are the intervals on which the convex minorant has constant slope. Moreover, \(C_z(m)=S_m\), and hence
\[
    \sum_{i=0}^{m-1}\widehat z_i
    =
    \sum_{i=0}^{m-1}z_i.
\]

To see how this construction enters the discrete allocation problem, fix \(x\) and \(B\), and put \(m=K_B\). Define
\[
    \widetilde v_k
    =
    v_{k,B}(x,x)
    +
    \one_{\{k=0\}}\Gamma_{0,B}(x,x),
    \qquad
    n_k=N_{0,B}R_{k,B}.
\]
Since \(R_{0,B}=1\), both \(\Gamma_{0,B}(x,x)\) and \(v_{0,B}(x,x)\) are divided by the initial population \(n_0=N_{0,B}\). This is why the initial-state variance is included in the first coordinate of the variance-contribution profile. The variance is
\[
    \sum_{k=0}^{m-1}
    \frac{\widetilde v_k}{n_k},
\]
while the relaxed particle counts must satisfy
\[
    0<n_0\le n_1\le\cdots\le n_{m-1}.
\]
Write
\[
    \widetilde v
    =
    (\widetilde v_0,\ldots,\widetilde v_{m-1}).
\]
Under the positivity condition in Lemma~\ref{lem:var-opt}, let
\[
    \widehat v=\PAVA(\widetilde v).
\]
That lemma shows that the optimal relaxed allocation is
\[
    n_k^*
    =
    B
    \frac{\sqrt{\widehat v_k}}
         {\displaystyle\sum_{j=0}^{m-1}\sqrt{\widehat v_j}},
    \qquad
    0\le k<m.
\]
Thus PAVA pools adjacent variance contributions that conflict with the monotonicity constraint, after which the allocation follows the usual square-root rule. 

We next define the continuous analogue, including an atom at time zero. Let \(a_0\ge0\), and let \(a\in L^1([0,T])\) be nonnegative almost everywhere. Define the cumulative variance profile \(A[a_0,a]\) by
\[
    A[a_0,a](0)=0,
    \qquad
    A[a_0,a](t)
    =
    a_0+\int_0^t a(s)\,ds,
    \quad 0<t\le T.
\]
The jump at zero records the initial variance atom. Let \(C[a_0,a]\) be the greatest continuous convex function \(C\) on
\([0,T]\) satisfying
\[
    C(0)=0,
    \qquad
    C(t)\le A[a_0,a](t),
    \quad 0<t\le T.
\]
We define the continuous PAVA profile by
\[
    h(t)=\PAVA(a_0,a)(t)
    :=
    \frac{d^+}{dt}C[a_0,a](t),
    \qquad
    0\le t<T.
\]
The function \(h\) is nonnegative, nondecreasing, right-continuous, and integrable. The greatest convex minorant meets the cumulative profile at the terminal time, so
\[
    C[a_0,a](T)
    =
    A[a_0,a](T)
    =
    a_0+\int_0^T a(t)\,dt.
\]
Consequently,
\[
    C[a_0,a](t)=\int_0^t h(s)\,ds,
    \qquad
    \int_0^T h(t)\,dt
    =
    a_0+\int_0^T a(t)\,dt.
\]

The gap
\[
    G(t)=A[a_0,a](t)-C[a_0,a](t)
\]
is nonnegative, satisfies \(G(T)=0\), and the convex minorant is affine on every connected component of \(\{t:G(t)>0\}\). Therefore \(h\) is constant on every interval over which the continuous construction pools variance. For any bounded, nonnegative, nonincreasing, right-continuous function \(q\), integration by parts gives
\[
    a_0q(0+)+\int_0^T a(t)q(t)\,dt
    =
    \int_0^T h(t)q(t)\,dt
    -
    \int_{(0,T)}G(t)\,dq(t)
    \ge
    \int_0^T h(t)q(t)\,dt.
\]
This is the continuous counterpart of the PAVA block inequality used in Lemma~\ref{lem:var-opt}. The discrete and continuous constructions agree on a fixed grid. Let
\(\Delta=T/m\) and define
\[
    d_0
    =
    a_0+\int_0^\Delta a(s)\,ds,
    \qquad
    d_i
    =
    \int_{i\Delta}^{(i+1)\Delta}a(s)\,ds,
    \quad 1\le i<m.
\]
Set
\[
    D_0=0,
    \qquad
    D_j=\sum_{i=0}^{j-1}d_i,
    \quad 1\le j\le m.
\]
Then \(D_j=A[a_0,a](j\Delta)\) for \(j\ge1\). The slopes of the greatest convex minorant through
\[
    (0,D_0),(\Delta,D_1),\ldots,(m\Delta,D_m)
\]
are
\[
    \frac{\PAVA(d)_i}{\Delta},
    \qquad
    0\le i<m.
\]
Thus the initial atom enters the first grid cell and may be pooled with later intervals whenever monotonicity requires it. In Theorem~\ref{thm:functional-minimax-opt}, the inputs are
\[
    a_0=\Gamma_0^\pi,
    \qquad
    a=g^\pi.
\]
The notation used there is therefore
\[
    h_\pi=\PAVA(\Gamma_0^\pi,g^\pi).
\]
The discrete construction in this subsection is equal-weight. The cost-weighted version \(\PAVA_c\), in which cumulative cost replaces the index on the horizontal axis and block averages use transition costs as weights, is defined in Appendix~\ref{app:cost-weighted}.

\subsection{Proof of Lemma~\ref{lem:splitting-cov}}
\label{app:proof-splitting-cov}

\begin{proof}
Let \(\sF_{-1}=\{\emptyset,\Omega\}\). For \(0\le i<K_B\), let \(\sF_i\) be the \(\sigma\)-algebra generated by the particle system immediately before transition \(i\). Thus, \(\sF_0\) contains the initial root particles. For \(i\ge1\), \(\sF_i\) includes split \(i\). Let \(\sF_{K_B}\) contain the terminal particle system. For \(x\in\R^d\), define
\[
    U_i(x)=\E[\widehat F_B(x)\mid\sF_i],
    \qquad i=-1,0,\ldots,K_B.
\]
Then \(U_i(x)\) is a martingale with \(U_{K_B}(x)=\widehat F_B(x)\) and \(U_{-1}(x)=F_B(x)\).
\[
    \Delta_i^x=U_{i+1}(x)-U_i(x),
    \qquad i=-1,0,\ldots,K_B-1.
\]
The martingale differences are orthogonal, hence
\[
    \widehat F_B(x)-F_B(x)
    =
    \sum_{i=-1}^{K_B-1}\Delta_i^x,
    \qquad
    \Cov(\widehat F_B(x),\widehat F_B(y))
    =
    \sum_{i=-1}^{K_B-1}\E[\Delta_i^x\Delta_i^y].
\]
Independence of the root initial states gives
\[
    \E[\Delta_{-1}^x\Delta_{-1}^y]
    =
    \frac1{N_{0,B}}\Gamma_{0,B}(x,y)
\]
Moreover, we also have
\[
    \E[\Delta_i^x\Delta_i^y\mid\sF_i]
    =
    \Cov(U_{i+1}(x),U_{i+1}(y)\mid\sF_i), \quad 0 \le i \le K_B-1.
\]
For \(0\le i<K_B\), let \(I_i\) be the set of particle labels immediately before transition \(i\). Then \(|I_i|=N_{0,B}R_{i,B}\). For \(\alpha\in I_i\), let \(X_i^\alpha\) be the state of particle \(\alpha\). Conditional on \(\sF_i\), transition \(i\) independently produces \(X_{i+1}^\alpha\) from each \(X_i^\alpha\). When \(i<K_B-1\), split \(i+1\) repeats each arriving state \(N_{i+1,B}\) times. Repeating every state the same number of times does not change the particle average. Therefore,
\[
    U_{i+1}(x)
    =
    \frac{1}{N_{0,B}R_{i,B}}
    \sum_{\alpha\in I_i}
    u^B_{i+1,x}\left(X_{i+1}^\alpha\right).
\]
Since the \(X_{i+1}^\alpha\) are conditionally independent given \(\sF_i\),
\begin{align*}
    &\Cov(U_{i+1}(x),U_{i+1}(y)\mid\sF_i) \\
    &\quad=
    \frac1{(N_{0,B}R_{i,B})^2}
    \sum_{\alpha\in I_i}
    \Cov\!\left(
        u^B_{i+1,x}(X_{i+1}^\alpha),
        u^B_{i+1,y}(X_{i+1}^\alpha)
        \mid\sF_i
    \right).
\end{align*}
The covariance in each summand depends on \(\sF_i\) only through \(X_i^\alpha\). Taking expectations, and using that each \(X_i^\alpha\) has the same law, we obtain
\[
    \E\left[\Cov(U_{i+1}(x),U_{i+1}(y)\mid\sF_i)\right]
    =
    \frac{v_{i,B}(x,y)}{N_{0,B}R_{i,B}}.
\]
Therefore
\[
    \Cov(\widehat F_B(x),\widehat F_B(y))
    =
    \frac1{N_{0,B}}
    \left[
        \Gamma_{0,B}(x,y)
        +
        \sum_{i=0}^{K_B-1}\frac{v_{i,B}(x,y)}{R_{i,B}}
    \right].
\]
\end{proof}

\subsection{Proof of Lemma~\ref{lem:var-opt}}
\label{app:proof-var-opt}

\begin{proof}
Put \(m=K_B\). Since \(\widetilde v_k>0\), the objective
\[
n\mapsto \sum_{k=0}^{m-1}\frac{\widetilde v_k}{n_k}
\]
is strictly convex and the feasible set is convex and hence, if a minimizer exists, it is unique. Let \(\widehat v=\PAVA(\widetilde v)\), and let \(I_1,\ldots,I_L\) be the final PAVA blocks. For a block \(I=[a,b]\), \(\widehat v\) is constant on \(I\), with value
\[
\widehat v_k=\bar v_I
=
\frac1{|I|}\sum_{j\in I}\widetilde v_j,
\qquad k\in I.
\]
By the standard min-max characterization of isotonic regression \citep[Chapter~1]{RobertsonWrightDykstra1988}, if \(I=[a,b]\) is a final block, then every initial segment \([a,\ell]\), \(\ell<b\), has average at least the average of \(I\). Thus,
\[
\sum_{k=a}^\ell (\widetilde v_k-\widehat v_k)\ge 0
\quad\text{for }a\le \ell<b,
\qquad
\sum_{k=a}^b (\widetilde v_k-\widehat v_k)=0.
\]

Now let \(n\) be any feasible monotone allocation. Since \(n_a\le n_{a+1}\le\cdots\le n_b\), we have
\[
\frac1{n_a}\ge \frac1{n_{a+1}}\ge\cdots\ge \frac1{n_b}.
\]
For a final PAVA block \(I=[a,b]\), observe that
\begin{align*}
\sum_{k=a}^b \frac{\widetilde v_k-\widehat v_k}{n_k}
&=
\sum_{k=a}^b (\widetilde v_k-\widehat v_k)
\left(
\frac1{n_k}-\frac1{n_b}
\right) \\
&=
\sum_{k=a}^{b-1}(\widetilde v_k-\widehat v_k)
\sum_{\ell=k}^{b-1}
\left(
\frac1{n_\ell}-\frac1{n_{\ell+1}}
\right) \\
&=
\sum_{\ell=a}^{b-1}
\left(
\sum_{k=a}^\ell(\widetilde v_k-\widehat v_k)
\right)
\left(
\frac1{n_\ell}-\frac1{n_{\ell+1}}
\right)
\ge 0.
\end{align*}
The last inequality follows because both factors are nonnegative. Summing over all PAVA blocks yields
\[
\sum_{k=0}^{m-1}\frac{\widehat v_k}{n_k}
\leq
\sum_{k=0}^{m-1}\frac{\widetilde v_k}{n_k}.
\]
Thus, for every feasible monotone allocation \(n\), the objective is bounded below by the same problem with \(\widetilde v\) replaced by its PAVA projection \(\widehat v\). We now minimize this lower bound and then verify that the resulting allocation makes the bound sharp. By Cauchy-Schwarz,
\[
    \sum_{k=0}^{m-1}\sqrt{\widehat v_k}
    =
    \sum_{k=0}^{m-1}
    \frac{\sqrt{\widehat v_k}}{\sqrt{n_k}}\sqrt{n_k}
    \leq
    \left(\sum_{k=0}^{m-1}\frac{\widehat v_k}{n_k}\right)^{1/2}
    \left(\sum_{k=0}^{m-1}n_k\right)^{1/2}.
\]
Since every feasible allocation satisfies \(\sum_k n_k=B\), we have the lower bound
\[
\sum_{k=0}^{m-1}\frac{\widetilde v_k}{n_k}
\ge
\frac{
\left(\sum_{k=0}^{m-1}\sqrt{\widehat v_k}\right)^2
}{B}.
\]
Equality holds in the Cauchy-Schwarz step if and only if \(n_k\propto \sqrt{\widehat v_k}\). The constant is determined by the budget constraint. Denote this allocation by \(n^*\). Since \(\widehat v\) is nondecreasing, \(n^*\) is nondecreasing. Equality also holds in the PAVA-block inequality because \(n^*\) is constant on each final PAVA block. Therefore \(n^*\) attains the lower bound and is optimal.
\end{proof}

\subsection{Splitting-function convergence}
\label{app:splitting-functions}

\begin{lemma}\label{lem:splitting-functions}
Under Assumption~\ref{ass:splitting-function}, \(r\) has a nondecreasing representative on \([0,T)\) satisfying \(r(0+)=1\), \(1\le r<\infty\) a.e., and
\[
    \frac1{r_B}
    \longrightarrow
    \frac1r
    \qquad\text{in }L^1([0,T]).
\]
\end{lemma}

\begin{proof}
Since each \(r_B\) is nondecreasing and \(r_B(0+)=1\), we have \(r_B\ge1\). The \(L^1([0,T])\)-convergence gives a subsequence converging to \(r\) almost everywhere, so \(r\ge1\) almost everywhere. Moreover, \(r<\infty\) almost everywhere because \(r\in L^1([0,T])\). The nondecreasing representative satisfying \(r(0+)=1\) is specified in Assumption~\ref{ass:splitting-function}.

Finally,
\[
    \left|\frac1{r_B}-\frac1r\right|
    =
    \frac{|r_B-r|}{r_Br}
    \le
    |r_B-r|
    \qquad\text{a.e.},
\]
so the claimed \(L^1\)-convergence follows.
\end{proof}

\subsection{Proof of Lemma~\ref{lem:continuous-step-covariance}}
\label{app:proof-continuous-step-covariance}

\begin{proof}
Since \(u_x(t,X_t)\) and \(u_y(t,X_t)\) are martingales, we write the conditional covariance as
\begin{align*}
    v^{\mathrm{SDE}}_{i,B}(x,y)
    &=
    \E\left[
        \Cov\!\big(
            u_x(t,X_t),
            u_y(t,X_t)
            \mid \sF_s
        \big)
    \right] \\
    &=
    \E\left[
        \E\left[
            \left(u_x(t, X_t) - u_x(s, X_s)\right)\left(u_y(t, X_t) - u_y(s, X_s)\right)
        \mid \sF_s \right]
    \right] \\
    &=
    \E\left[
        \left(u_x(t, X_t) - u_x(s, X_s)\right)\left(u_y(t, X_t) - u_y(s, X_s)\right)
    \right]\\
    &=
    \E\left[u_x(t,X_t)u_y(t,X_t)\right]
    -\E\left[
        u_y(s,X_s)\E\left[u_x(t,X_t)\mid\mathcal F_s\right]
    \right] \\
    &\qquad
    -\E\left[
        u_x(s,X_s)\E\left[u_y(t,X_t)\mid\mathcal F_s\right]
    \right]
    +\E\left[u_x(s,X_s)u_y(s,X_s)\right] \\
    &=\E\left[
        u_x(t,X_t)u_y(t,X_t)
        -
        u_x(s,X_s)u_y(s,X_s)
    \right].
\end{align*}
Using the martingale representations,
\[
    d u_x(t,X_t)=H_x(t)^\top dW_t,
    \qquad
    d u_y(t,X_t)=H_y(t)^\top dW_t,
\]
we get
\[
    d\left\langle u_x(\cdot,X_\cdot),u_y(\cdot,X_\cdot)\right\rangle_t
    =
    H_x(t)^\top H_y(t)\,dt.
\]
Since \(u_x(t,X_t)u_y(t,X_t) - \langle u_x(\cdot,X_\cdot),u_y(\cdot,X_\cdot)\rangle_t\) is a martingale, we have
\[
    v^{\mathrm{SDE}}_{i,B}(x,y)
    =
    \E\left[\int_s^t H_x(\tau)^\top H_y(\tau)\,d\tau\right]
    =
    \int_s^t g_{x,y}(\tau)\,d\tau.
\]
It remains to establish integrability. By It\^o's isometry,
\begin{align*}
    \E\left[\int_0^T \|H_x(\tau)\|_2^2\,d\tau\right]
    &=
    \E\left[\left(\int_0^T H_x(\tau)^\top dW_\tau\right)^2\right] \\
    &=
    \E\left[\left(u_x(T,X_T)-u_x(0,X_0)\right)^2\right] \\
    &=
    \E\left[\Var(a_x(X_T)\mid X_0)\right] \\
    &\le
    \Var(a_x(X_T))
    \le \frac14.
\end{align*}
Therefore, by Cauchy-Schwarz,
\begin{align*}
    \int_0^T |g_{x,y}(\tau)|\,d\tau
    &\le
    \E\left[
        \int_0^T \|H_x(\tau)\|_2\|H_y(\tau)\|_2\,d\tau
    \right] \\
    &\le
    \E\left[\int_0^T \|H_x(\tau)\|_2^2\,d\tau\right]^{1/2}
    \E\left[\int_0^T \|H_y(\tau)\|_2^2\,d\tau\right]^{1/2} \\
    &\le \frac14.
\end{align*}
\end{proof}

\subsection{Proof of Proposition~\ref{prop:covariance-convergence}}
\label{app:proof-covariance-convergence}

\begin{proof}
Since \(X^B_{K_B}\to X_T\) in \(L^2\), we also have \(X^B_{K_B}\to X_T\) in probability. Continuity of \(F\) implies that, for any \(x\in\R^d\),
\[
    \Pbb\left(X_T\in \partial(-\infty,x]\right)=0.
\]
Hence
\[
    a_x(X^B_{K_B})\xrightarrow{\Pbb} a_x(X_T).
\]
Since these random variables are bounded by \(1\), the convergence also holds in \(L^2\). Recall that
\[
    u^B_{i,x}(z)=\E[a_x(X^B_{K_B})\mid X_i^B=z],
    \qquad
    u_x(t,z)=\E[a_x(X_T)\mid X_t=z].
\]
For the remaining part of the proof, we suppress the arguments and write
\[
    u^B_{i,x}:=u^B_{i,x}(X_i^B),
    \qquad
    u_{t,x}:=u_x(t,X_t).
\]
We first prove \(\Gamma_{0,B}(x,y)\to \Gamma_0(x,y)\). For any \(x\in\R^d\),
\[
    u^B_{0,x}-u_{0,x}
    =
    \E\left[
        a_x(X^B_{K_B})-a_x(X_T)
        \mid \sF_0
    \right].
\]
By the \(L^2\)-contraction property of conditional expectation,
\[
    \left\|u^B_{0,x}-u_{0,x}\right\|_{L^2}
    \le
    \left\|a_x(X^B_{K_B})-a_x(X_T)\right\|_{L^2}
    \to0,
    \qquad \forall x\in\R^d.
\]
Thus, for all \(x, y \in \R^d\), \(u^B_{0,x}\to u_{0,x}\) and \(u^B_{0,y}\to u_{0,y}\) in \(L^2\). Then, by triangle inequality and then Cauchy-Schwarz, we get
\begin{align*}
    &\left|
        \Cov(u^B_{0,x},u^B_{0,y})
        -
        \Cov(u_{0,x},u_{0,y})
    \right| \\
    &\quad=
    \left|
        \Cov(u^B_{0,x}-u_{0,x},u^B_{0,y})
        +
        \Cov(u_{0,x},u^B_{0,y}-u_{0,y})
    \right| \\
    &\quad\le
    \left|\Cov(u^B_{0,x}-u_{0,x},u^B_{0,y})\right|
    +
    \left|\Cov(u_{0,x},u^B_{0,y}-u_{0,y})\right| \\
    &\quad\le
    \|u^B_{0,x}-u_{0,x}\|_{L^2}\|u^B_{0,y}\|_{L^2}
    +
    \|u_{0,x}\|_{L^2}\|u^B_{0,y}-u_{0,y}\|_{L^2} \\
    &\quad\le
    \|u^B_{0,x}-u_{0,x}\|_{L^2}
    +
    \|u^B_{0,y}-u_{0,y}\|_{L^2}
    \to0.
\end{align*}
Hence,
\[
    \Gamma_{0,B}(x,y)
    =
    \Cov(u^B_{0,x},u^B_{0,y})
    \to
    \Cov(u_{0,x},u_{0,y})
    =
    \Gamma_0(x,y).
\]
It remains to prove the one-step covariance convergence. Write
\[
    \Delta_i^{B,z}
    :=
    u^B_{i+1,z}-u^B_{i,z},
    \qquad
    \Delta_i^z
    :=
    u_{t_{i+1,B},z}-u_{t_{i,B},z} .
\]
Then
\[
    v_{i,B}(x,y)
    =
    \E\left[\Delta_i^{B,x}\Delta_i^{B,y}\right],
    \qquad
    v^{\mathrm{SDE}}_{i,B}(x,y)
    =
    \E\left[\Delta_i^x\Delta_i^y\right].
\]
For every \(z\in\R^d\), \(u^B_{i,z}\), \(u_{t_{i,B},z}\), and \(u^B_{i,z}-u_{t_{i,B},z}\) are square-integrable martingales. In particular, their increments are centered. Hence, using \(ab-cd=(a-c)b+c(b-d)\), the triangle inequality, and Cauchy-Schwarz,
\begin{align*}
    &\sum_{i=0}^{K_B-1}
    \left|
        v_{i,B}(x,y)-v^{\mathrm{SDE}}_{i,B}(x,y)
    \right| \\
    &\quad= \sum_{i=0}^{K_B-1}\left|
        \E\left[
            \Delta_i^{B,x}\Delta_i^{B,y}-\Delta_i^x\Delta_i^y
        \right]
    \right| \\
    &\quad\le
    \sum_{i=0}^{K_B-1}
    \left|
        \E\left[
            (\Delta_i^{B,x}-\Delta_i^x)\Delta_i^{B,y}
        \right]
    \right|
    +
    \sum_{i=0}^{K_B-1}
    \left|
        \E\left[
            \Delta_i^x(\Delta_i^{B,y}-\Delta_i^y)
        \right]
    \right| \\
    &\quad\le
    \left(
        \sum_{i=0}^{K_B-1}
        \E\left[\left|\Delta_i^{B,x}-\Delta_i^x\right|^2\right]
    \right)^{1/2}
    \left(
        \sum_{i=0}^{K_B-1}
        \E\left[\left|\Delta_i^{B,y}\right|^2\right]
    \right)^{1/2} \\
    &\qquad+
    \left(
        \sum_{i=0}^{K_B-1}
        \E\left[\left|\Delta_i^x\right|^2\right]
    \right)^{1/2}
    \left(
        \sum_{i=0}^{K_B-1}
        \E\left[\left|\Delta_i^{B,y}-\Delta_i^y\right|^2\right]
    \right)^{1/2}.
\end{align*}
Therefore, by orthogonality of martingale increments, for all \(x\in\R^d\),
\[
    \sum_{i=0}^{K_B-1}
    \E\left[\left|\Delta_i^{B,x}-\Delta_i^x\right|^2\right]
    \le
    \E\left[\left|u^B_{K_B,x}-u_{T,x}\right|^2\right]
    =
    \E\left[\left|a_x(X^B_{K_B})-a_x(X_T)\right|^2\right]
    \to0,
\]
and,
\[
    \sum_{i=0}^{K_B-1}
    \E\left[\left|\Delta_i^{B,x}\right|^2\right]
    \le
    \E\left[\left|u^B_{K_B,x}\right|^2\right]
    \le1, \qquad
    \sum_{i=0}^{K_B-1}
    \E\left[\left|\Delta_i^x\right|^2\right]
    \le
    \E\left[\left|u_{T,x}\right|^2\right]
    \le1.
\]
Substituting them into the preceding Cauchy-Schwarz bound gives
\[
    \sum_{i=0}^{K_B-1}
    \left|
        v_{i,B}(x,y)-v^{\mathrm{SDE}}_{i,B}(x,y)
    \right|
    \to0 .
\]
\end{proof}

\subsection{Proof of Proposition~\ref{prop:covlimit}}
\label{app:proof-covlimit}

\begin{proof}
By Lemma \ref{lem:splitting-cov} and the definition of \(\Gamma_B\),
\[
    \Gamma_B(x,y)
    =
    \Gamma_{0,B}(x,y)
    +
    \sum_{i=0}^{K_B-1}\frac{v_{i,B}(x,y)}{R_{i,B}}.
\]
By the second statement of Proposition \ref{prop:covariance-convergence}, we have
\[
    \Gamma_B(x,y)
    =
    \Gamma_{0,B}(x,y)
    +
    \sum_{i=0}^{K_B-1}
    \frac{v^{\mathrm{SDE}}_{i,B}(x,y)}{R_{i,B}}
    +o(1).
\]
By Lemma~\ref{lem:continuous-step-covariance},
\[
    v^{\mathrm{SDE}}_{i,B}(x,y)
    =
    \int_{t_{i,B}}^{t_{i+1,B}}g_{x,y}(t)\,dt,
\]
Therefore,
\[
    \Gamma_B(x,y)
    =
    \Gamma_{0,B}(x,y)
    +
    \int_0^T \frac{g_{x,y}(t)}{r_B(t)}\,dt+o(1).
\]
By Lemma~\ref{lem:splitting-functions}, \(1\le r<\infty\) almost everywhere and \(\|1/r_B-1/r\|_{L^1([0,T])}\to0\). Fix \(M>0\). Since \(|1/r_B-1/r|\le2\),
\[
    \int_0^T \left|\frac1{r_B(t)}-\frac1{r(t)}\right|\,|g_{x,y}(t)|\,dt
    \le
    M\left\|\frac1{r_B}-\frac1r\right\|_{L^1([0,T])}
    +
    2\int_0^T |g_{x,y}(t)|\one(|g_{x,y}(t)|>M)\,dt.
\]
Letting \(B\to\infty\) and then \(M\to\infty\), and using \(g_{x,y}\in L^1([0,T])\), gives
\[
    \int_0^T \frac{g_{x,y}(t)}{r_B(t)}\,dt
    \to
    \int_0^T \frac{g_{x,y}(t)}{r(t)}\,dt.
\]
The first statement of Proposition \ref{prop:covariance-convergence} then proves \(\Gamma_B(x,y)\to\Gamma(x,y)\). Since each \(\Gamma_B\) is a covariance kernel, its pointwise limit \(\Gamma\) is positive semidefinite. It remains to prove the bound on \(d_\Gamma\). By definition,
\begin{align*}
    d_\Gamma(x,y)^2
    &=
    \Gamma(x,x)+\Gamma(y,y)-2\Gamma(x,y) \\
    &=
    \Var\!\left(u_x(0,X_0)-u_y(0,X_0)\right)
    +
    \int_0^T \frac1{r(t)}
    [g_{x,x}(t)+g_{y,y}(t)-2g_{x,y}(t)]\,dt.
\end{align*}
Since
\[
    g_{x,x}(t)+g_{y,y}(t)-2g_{x,y}(t)
    =
    \E\left[\left\|H_x(t)-H_y(t)\right\|_2^2\right],
\]
and \(1/r\le1\), we get
\begin{align*}
    d_\Gamma(x,y)^2
    &\le
    \Var\!\left(u_x(0,X_0)-u_y(0,X_0)\right)
    +
    \E\left[\int_0^T \left\|H_x(t)-H_y(t)\right\|_2^2\,dt\right] \\
    &=
    \Var(a_x(X_T)-a_y(X_T)) \\
    &\le
    \E\left[(a_x(X_T)-a_y(X_T))^2\right]
    =
    \left\|a_x(X_T)-a_y(X_T)\right\|_{L^2}^2.
\end{align*}
\end{proof}

\subsection{Finite Donsker Entropy}
\label{app:finite-donsker-entropy}

\begin{lemma}[Finite Donsker Entropy]\label{lem:vc}
There are constants \(A<\infty\), \(v<\infty\), \(A'<\infty\), and \(v'<\infty\), depending only on \(d\), such that, for \(0<\varepsilon<1\),
\[
    \sup_\mu N(\varepsilon,\cA,L^2(\mu))\le A\varepsilon^{-v},
    \qquad
    \sup_{B,\nu}N(\varepsilon,\cF_B,L^2(\nu))
    \le A\varepsilon^{-v}.
\]
If
\[
    \mathcal H_{B,\delta}
    =
    \{f_{B,x}-f_{B,y}:x,y\in\R^d,\ \rho(a_x,a_y)<\delta\},
\]
then
\[
    \sup_{B,\delta,\nu}N(\varepsilon,\mathcal H_{B,\delta},L^2(\nu))
    \le A'\varepsilon^{-v'}.
\]
Consequently, \((\cA,\rho)\) is totally bounded and the uniform entropy integrals for \(\cA\) and \(\mathcal H_{B,\delta}\) with respect to \(L^2(\mu)\) and \(L^2(\nu)\), respectively,
are finite for all measures \(\mu\) and \(\nu\), i.e.,
\[
    \sup_\mu\int_0^1\sqrt{\log N(\varepsilon,\cA,L^2(\mu))}\,d\varepsilon
    <\infty,
    \qquad
    \sup_{B,\delta,\nu}\int_0^1\sqrt{\log N(\varepsilon,\mathcal H_{B,\delta},L^2(\nu))}\,d\varepsilon
    <\infty.
\]
\end{lemma}

\begin{proof}
The class \(\cA\) has VC dimension \(d\), so by \citet[Theorem 2.6.4]{WeakConvergence2023} there are constants \(A<\infty\) and \(v<\infty\), depending only on \(d\), such that
\[
    N(\varepsilon,\cA,L^2(\mu))\le A\varepsilon^{-v}
\]
for every probability measure \(\mu\) on \(\R^d\). To transfer the same bound to \(\cF_B\), let \(\nu\) be any probability measure on \((\R^d)^{m_B}\), and let \(\nu_1,\ldots,\nu_{m_B}\) be its coordinate marginals. Set
\[
    \overline\nu=\frac1{m_B}\sum_{\ell=1}^{m_B}\nu_\ell.
\]
Let \(\mathcal C\) be an \(\varepsilon\)-cover of \(\cA\) in \(L^2(\overline\nu)\).  For \(g\in\mathcal C\), define
\[
    \widetilde g(z)=\frac1{m_B}\sum_{\ell=1}^{m_B}g(z^{(\ell)}).
\]
If \(g\) is within \(\varepsilon\) of \(a_x\) in \(L^2(\overline\nu)\), then Jensen's inequality gives
\begin{align*}
    \|f_{B,x}-\widetilde g\|_{L^2(\nu)}^2
    &=
    \int\left[
        \frac1{m_B}\sum_{\ell=1}^{m_B}
        \{a_x(z^{(\ell)})-g(z^{(\ell)})\}
    \right]^2d\nu(z)  \\
    &\le
    \frac1{m_B}\sum_{\ell=1}^{m_B}
    \int \{a_x(z^{(\ell)})-g(z^{(\ell)})\}^2d\nu(z) \\
    &=
    \|a_x-g\|_{L^2(\overline\nu)}^2
    \le \varepsilon^2.
\end{align*}
Thus the same covering number bound transfers from \(\cA\) to \(\cF_B\). Now fix \(\delta>0\).  Let \(\mathcal C_B\) be an \(\varepsilon/2\)-cover of \(\cF_B\) in \(L^2(\nu)\).  For any \(f_{B,x}-f_{B,y}\in\mathcal H_{B,\delta}\), choose \(g_x,g_y\in\mathcal C_B\) such that
\[
    \|f_{B,x}-g_x\|_{L^2(\nu)}\le\varepsilon/2,
    \qquad
    \|f_{B,y}-g_y\|_{L^2(\nu)}\le\varepsilon/2.
\]
Then the triangle inequality gives
\[
    \|(f_{B,x}-f_{B,y})-(g_x-g_y)\|_{L^2(\nu)}
    \le \varepsilon.
\]
Hence
\[
    N(\varepsilon,\mathcal H_{B,\delta},L^2(\nu))
    \le
    N(\varepsilon/2,\cF_B,L^2(\nu))^2
    \le
    A^2 2^{2v}\varepsilon^{-2v}.
\]
Taking \(A'=A^2 2^{2v}\) and \(v'=2v\) proves the claim. To get the finite entropy bound, use the uniform bound on the covering numbers and observe that the integral is finite.
\end{proof}

\subsection{Proof of Theorem~\ref{thm:splitting-clt}}
\label{app:proof-splitting-clt}

\begin{proof}
Let \(P_B=\mathcal L(Z_B)\) and \(n_B=N_{0,B}\). For each \(B\), \(Z_{B,1},\ldots,Z_{B,n_B}\) are i.i.d. with law \(P_B\), while \(E_B\), \(P_B\), and \(\cF_B\) may depend on \(B\). We regard the row-wise empirical process as indexed by \(\cA\), with the coordinate labeled by \(a_x\) given by
\[
    \G_B(a_x)
    =
    \frac1{\sqrt{n_B}}\sum_{j=1}^{n_B}
    \left(f_{B,x}(Z_{B,j})-\E\left[f_{B,x}(Z_B)\right]\right).
\]
Indeed, \(n_B^{-1}\sum_{j=1}^{n_B}f_{B,x}(Z_{B,j})=\widehat F_B(x)\) and \(\E[f_{B,x}(Z_B)]=F_B(x)\).
To show weak convergence in \(\ell^\infty(\cA)\), it suffices to prove finite-dimensional convergence and asymptotic \(\rho\)-equicontinuity on \(\cA\), with a tight Gaussian limit. We first show that all finite-dimensional distributions converge. Fix \(a_{x_1},\ldots,a_{x_J}\in\cA\). The summands are centered, i.i.d. within each row, and uniformly bounded. Their covariance matrix is \((\Gamma_B(x_\ell,x_m))_{\ell,m}\), which converges to \((\Gamma(x_\ell,x_m))_{\ell,m}\) by Proposition~\ref{prop:covlimit}. By the multivariate Lindeberg-Feller theorem, the finite-dimensional distributions converge. 

We now construct the limiting object as a tight Borel element of \(\ell^\infty(\cA)\). Proposition \ref{prop:covlimit} shows that \(\Gamma\) is positive semidefinite, so Kolmogorov's extension theorem gives a centered Gaussian process \(\overline{\G}\) with covariance \(\Gamma\). Define the canonical semimetric
\[
    d_\Gamma(a_x,a_y)^2
    =
    \Var(\overline{\G}(a_x)-\overline{\G}(a_y))
    =
    \Gamma(x,x)+\Gamma(y,y)-2\Gamma(x,y).
\]
Proposition \ref{prop:covlimit} gives \(d_\Gamma\le\rho\). By Lemma \ref{lem:vc}, \((\cA,\rho)\) is totally bounded and has finite entropy integral, and therefore
\[
\int_0^\infty \sqrt{\log N(\varepsilon,\cA,d_\Gamma)}\,d\varepsilon <\infty.
\]
By Dudley's entropy theorem \citep[Corollary 2.2.9]{WeakConvergence2023}, there is a separable version \(\G\) of this Gaussian process with bounded and uniformly \(d_\Gamma\)-continuous paths. Since \(d_\Gamma\le\rho\), these paths are also uniformly \(\rho\)-continuous. Hence, this version is a tight Borel element of \(\ell^\infty(\cA)\). It remains to show asymptotic equicontinuity of \(\G_B\), that is, for every \(\eta>0\),
\[
    \lim_{\delta\downarrow0}\limsup_{B\to\infty}
    \Pbb^*
    \left(
        \sup_{\rho(a_x,a_y)<\delta}
        |\G_B(a_x)-\G_B(a_y)|>\eta
    \right)
    =0.
\]

By Jensen's inequality and the fact that every terminal descendant has marginal law \(X_{K_B}^B\),
\[
    \E[(f_{B,x}(Z_B)-f_{B,y}(Z_B))^2]
    \le
    \E[(a_x(X_{K_B}^B)-a_y(X_{K_B}^B))^2].
\]
Then, observe that
\begin{align*}
    \E[(a_x(X_{K_B}^B)-a_y(X_{K_B}^B))^2]
    &=
    F_B(x)+F_B(y)-2F_B(x\wedge y), \\
    \rho(a_x,a_y)^2&=F(x)+F(y)-2F(x\wedge y).
\end{align*}
Therefore,
\[
    \E[(f_{B,x}(Z_B)-f_{B,y}(Z_B))^2]
    \le
    \rho(a_x,a_y)^2+4\|F_B-F\|_\infty.
\]
For \(\delta>0\), let \(\mathcal H_{B,\delta}\) be the class defined in Lemma~\ref{lem:vc}.
Put
\[
    \alpha_{B,\delta}^2
    =
    \sup_{h\in\mathcal H_{B,\delta}}\E[h(Z_B)^2]
    =
    \sup_{\rho(a_x,a_y)<\delta}
    \E[(f_{B,x}(Z_B)-f_{B,y}(Z_B))^2].
\]
By Assumptions~\ref{ass:terminal-strong-consistency} and~\ref{ass:terminal-weak-conv}, \(F_B(x)\to F(x)\) for every \(x\in\R^d\). The multivariate P\'olya theorem \citep[Lemma 2.11]{AsymptoticStatistics1998} therefore gives \(\|F_B-F\|_\infty\to0\). Hence
\[
    \lim_{\delta\downarrow0}\limsup_{B\to\infty}\alpha_{B,\delta}=0.
\]

Fix \(t\in(0,1)\). Choose \(\delta_t>0\) such that
\[
    \limsup_{B\to\infty}\alpha_{B,\delta_t}\le t/2.
\]
Then, for every \(\delta\le\delta_t\), we have \(\alpha_{B,\delta}<t\) for all sufficiently large \(B\). Since \(|h|\le1\) for \(h\in\mathcal H_{B,\delta}\), the class \(\mathcal H_{B,\delta}\) has envelope \(H\equiv1\), and \citet[Theorem 2.14.2]{WeakConvergence2023} applied to the \(B\)-th row at variance level \(t\) gives a constant \(C\), independent of \(B\) and \(\delta\), with
\[
    \limsup_{B\to\infty}
    \E^*\left[
    \sup_{\rho(a_x,a_y)<\delta}
    |\G_B(a_x)-\G_B(a_y)|
    \right]
    \le
    \limsup_{B\to\infty}
    C I(t)\left(1+\frac{I(t)}{t^2\sqrt{N_{0,B}}}\right)
    =
    C I(t),
\]
where
\[
    I(t)
    =
    \sup_{B,\delta,\nu}
    \int_0^t
    \sqrt{1+\log N(\varepsilon,\mathcal H_{B,\delta},L^2(\nu))}\,d\varepsilon,
\]
the last equality holding because \(I(t)<\infty\) and \(N_{0,B}\to\infty\). By Lemma \ref{lem:vc}, \(I(t)\to0\) as \(t\downarrow0\). Therefore
\[
    \lim_{\delta\downarrow0}\limsup_{B\to\infty}
    \E^*\left[
    \sup_{\rho(a_x,a_y)<\delta}
    |\G_B(a_x)-\G_B(a_y)|
    \right]
    =0.
\]
Then, by Markov's inequality, \(\G_B\) is asymptotically \(\rho\)-equicontinuous on \(\cA\). Hence,
\[
    \G_B\Rightarrow \G
    \qquad\text{in }\ell^\infty(\cA).
\]
\end{proof}

\subsection{Proof of Corollary~\ref{cor:dg-limit}}
\label{app:proof-dg-limit}

\begin{proof}
Since \(C_B=B/N_{0,B}\),
\[
    B^{p/(1+2p)}N_{0,B}^{-1/2}
    =
    \left(\frac{h_BC_B}{\lambda_B}\right)^{1/2}
    \longrightarrow
    \left(\frac{\gamma}{\lambda}\right)^{1/2}.
\]
Theorem~\ref{thm:splitting-clt} and Slutsky's theorem therefore give
\[
    B^{p/(1+2p)}(\widehat F_B-F_B)
    \Rightarrow
    \left(\frac{\gamma}{\lambda}\right)^{1/2}\G
    \qquad\text{in }\ell^\infty(\cA).
\]
Moreover, Assumption~\ref{ass:bias} gives
\begin{align*}
    B^{p/(1+2p)}\|F_B-F\|_\infty
    &=
    O\!\left(B^{p/(1+2p)}h_B^p\right) \\
    &=O(\lambda_B^p)
    =O(1).
\end{align*}
Weak convergence in \(\ell^\infty(\cA)\) implies that
\[
    B^{p/(1+2p)}
    \|\widehat F_B-F_B\|_\infty
\]
is bounded in probability. The preceding bias bound and the triangle inequality applied to
\[
    \widehat F_B-F
    =
    (\widehat F_B-F_B)+(F_B-F)
\]
complete the proof.
\end{proof}

\subsection{Proof of Proposition~\ref{prop:ks-bounds}}
\label{app:proof-ks-bounds}

\begin{proof}
It suffices to bound
\[
    \E\left[
        \sup_{x\in\R^d}|\G(a_x)|
    \right],
\]
since \((\gamma/\lambda)^{1/2}\) is deterministic.
By Proposition \ref{prop:covlimit}, \(0\le \Gamma^*\le1/2\) and \(d_\Gamma\le\rho\). It follows that the \(d_\Gamma\)-diameter of \(\R^d\) is at most \(2\Gamma^*\). By \citet[Corollary 2.2.9]{WeakConvergence2023}, there is a universal constant \(C_0<\infty\) such that
\[
    \E\left[\sup_{y\in\R^d}|\G(a_y)|\right]
    \le
    \E[|\G(a_x)|]+
    C_0
    \int_0^{2\Gamma^*}
    \sqrt{\log N(\varepsilon,\R^d,d_\Gamma)}\,d\varepsilon, \quad \forall x\in\R^d.
\]
The first term in the upper bound can be taken to zero by taking \(x\) to infinity in all coordinates. For the second term, the covering number is bounded from Lemma \ref{lem:vc} involving a constant that only depends on \(d\). Therefore, for constants \(C,C'<\infty\) depending only on \(d\),
\[
    \E\left[\sup_{x\in\R^d}|\G(a_x)|\right]
    \le
    C
    \int_0^{2\Gamma^*}
    \sqrt{\log A+v\log(1/\varepsilon)}\,d\varepsilon
    \le
    C'\Gamma^*
    \sqrt{1-\log(\Gamma^*)}.
\]
The last inequality follows by the change of variables \(\varepsilon=e^{-u}\), with constants absorbed into \(C'\). For the lower bound,
\[
    \E\left[
        \sup_{x\in\R^d}|\G(a_x)|
    \right]
    \ge
    \sup_{x\in\R^d}
    \E[|\G(a_x)|]
    =
    \sqrt{\frac{2}{\pi}}\Gamma^*.
\]
Restoring the factor \((\gamma/\lambda)^{1/2}\) yields the stated inequalities.
\end{proof}

\subsection{Continuity in the Threshold}
\label{app:threshold-continuity}

\begin{lemma}[Continuity in the threshold]
\label{lem:threshold-continuity}
Let \(\mathcal X=[-\infty,\infty]^d\) have the product topology.
For \(x\in\mathcal X\), define
\[
    a_x(z)=\one(z\le x),
    \qquad
    u_x(t,X_t)=\E[a_x(X_T)\mid\sF_t].
\]
using the extended coordinate order. Write
\[
    u_x(t,X_t)
    =
    u_x(0,X_0)+\int_0^tH_x(s)^\top dW_s.
\]
and set
\[
    \Gamma_0(x,x)=\Var(u_x(0,X_0)),
    \qquad
    g_{x,x}(t)=\E[\|H_x(t)\|_2^2].
\]
Under Assumption~\ref{ass:terminal-weak-conv},
\[
    x\longmapsto\Gamma_0(x,x)
\]
is continuous on \(\mathcal X\), and
\[
    x\longmapsto g_{x,x}
\]
is continuous from \(\mathcal X\) into \(L^1([0,T])\).
\end{lemma}

\begin{proof}
Continuity of \(F\) implies that every marginal of \(X_T\) is
atomless. Let \(x_m\to x\) in \(\mathcal X\). The atomless-marginal property gives
\[
    a_{x_m}(X_T)\longrightarrow a_x(X_T)
    \qquad\text{almost surely}.
\]
This also covers coordinates equal to \(+\infty\) or \(-\infty\).
Since the indicators are bounded,
\[
    \delta_m
    :=
    \|a_{x_m}(X_T)-a_x(X_T)\|_{L^2}
    \longrightarrow0.
\]

Conditional expectation is an \(L^2\)-contraction. Hence
\[
    \|u_{x_m}(0,X_0)-u_x(0,X_0)\|_{L^2}
    \le\delta_m.
\]
It follows that
\[
    \Gamma_0(x_m,x_m)
    \longrightarrow
    \Gamma_0(x,x).
\]

Subtracting the martingale representations and applying It\^o's isometry gives
\[
    \E\int_0^T
    \|H_{x_m}(t)-H_x(t)\|_2^2\,dt
    =
    \E\left[
        \Var\left(
            a_{x_m}(X_T)-a_x(X_T)
            \mid X_0
        \right)
    \right]
    \le\delta_m^2.
\]
The same isometry gives, for every \(z\in\mathcal X\),
\[
    \left(
        \E\int_0^T\|H_z(t)\|_2^2\,dt
    \right)^{1/2}
    \le\frac12.
\]
Therefore, by Cauchy-Schwarz,
\begin{align*}
    \|g_{x_m,x_m}-g_{x,x}\|_{L^1}
    &\le
    \left(
        \E\int_0^T
        \|H_{x_m}(t)-H_x(t)\|_2^2\,dt
    \right)^{1/2} \\
    &\quad\times
    \left[
        \left(
            \E\int_0^T\|H_{x_m}(t)\|_2^2\,dt
        \right)^{1/2}
        +
        \left(
            \E\int_0^T\|H_x(t)\|_2^2\,dt
        \right)^{1/2}
    \right] \\
    &\le\delta_m
    \longrightarrow0.
\end{align*}
\end{proof}

\subsection{Proof of Theorem~\ref{thm:functional-minimax-opt}}
\label{app:proof-functional-minimax-opt}

\begin{proof}
Let
\[
    \mathcal N_\infty
    =
    \left\{
        n:[0,T)\to(0,\infty):
        \begin{array}{l}
            n\text{ is nondecreasing and right-continuous},\\
            \displaystyle\int_0^Tn(t)\,dt=1,\quad n(0+)>0
        \end{array}
    \right\}.
\]
For \(r\in\mathcal R\), set \(n=r/\gamma(r)\). Conversely, \(r=n/n(0+)\) and \(\gamma(r)=1/n(0+)\). For \(n\in\mathcal N_\infty\) and \(\pi\in\mathcal P(\mathcal X)\), define
\[
    L(n,\pi)
    =
    \frac{\Gamma_0^\pi}{n(0+)}
    +
    \int_0^T\frac{g^\pi(t)}{n(t)}\,dt.
\]
For fixed \(n\in\mathcal N_\infty\), let
\[
    \ell_n(x)
    =
    \frac{\Gamma_0(x,x)}{n(0+)}
    +
    \int_0^T\frac{g_{x,x}(t)}{n(t)}\,dt.
\]
Since \(n(t)\ge n(0+)>0\), Lemma~\ref{lem:threshold-continuity} shows that \(\ell_n\) is continuous on \(\mathcal X\). Moreover,
\[
    L(n,\pi)
    =
    \int_{\mathcal X}\ell_n(x)\,\pi(dx).
\]
It follows that
\[
    \sup_{\pi\in\mathcal P(\mathcal X)}L(n,\pi)
    =
    \max_{x\in\mathcal X}\ell_n(x)
    =
    \sup_{x\in\R^d}\ell_n(x).
\]
The criterion therefore becomes
\[
    \inf_{n\in\mathcal N_\infty}
    \sup_{\pi\in\mathcal P(\mathcal X)}L(n,\pi).
\]

Fix \(\pi\), write \(h=h_\pi\), and let
\[
    H(\pi)
    =
    \left(
        \int_0^T\sqrt{h(t)}\,dt
    \right)^2.
\]
Let \(q=1/n\), let \(C=C[\Gamma_0^\pi,g^\pi]\), and set
\[
    \Delta(t)
    =
    \Gamma_0^\pi
    +
    \int_0^t g^\pi(s)\,ds
    -
    C(t).
\]
The greatest convex minorant meets its upper bound at \(T\). Since \(\Delta\ge0\) and \(q\) is nonincreasing, integration by parts and Cauchy-Schwarz give
\begin{align*}
    L(n,\pi)
    &=
    \int_0^T\frac{h(t)}{n(t)}\,dt
    -
    \int_{(0,T)}\Delta(t)\,dq(t) \\
    &\ge
    \int_0^T\frac{h(t)}{n(t)}\,dt \\
    &\ge
    H(\pi).
\end{align*}

For \(\varepsilon>0\), define
\[
    n_\varepsilon(t)
    =
    \frac{\sqrt{h(t)+\varepsilon}}
         {\displaystyle\int_0^T\sqrt{h(s)+\varepsilon}\,ds},
    \qquad
    q_\varepsilon=\frac1{n_\varepsilon}.
\]
Since \(h\) is nondecreasing, \(n_\varepsilon\in\mathcal N_\infty\). The function \(C\) is affine on every interval where \(\Delta>0\). Thus \(h\) and \(q_\varepsilon\) are constant on those intervals. Elsewhere \(\Delta=0\). Hence
\[
    \int_{(0,T)}\Delta(t)\,dq_\varepsilon(t)=0.
\]
It follows that
\[
    L(n_\varepsilon,\pi)
    =
    \left(\int_0^T\sqrt{h(t)+\varepsilon}\,dt\right)
    \int_0^T
    \frac{h(t)}{\sqrt{h(t)+\varepsilon}}\,dt
    \longrightarrow
    H(\pi).
\]
Therefore
\[
    \inf_{n\in\mathcal N_\infty}L(n,\pi)=H(\pi).
\]
When \(H(\pi)>0\), equality in Cauchy-Schwarz forces \(n\) to be proportional to \(\sqrt h\). For this choice, \(q\) is constant wherever \(\Delta>0\), so the first inequality is also an equality. It belongs to \(\mathcal N_\infty\) exactly when \(h(0+)>0\). In that case, the unique minimizer for this mixture is
\[
    n_\pi(t)
    =
    \frac{\sqrt{h_\pi(t)}}
         {\displaystyle\int_0^T\sqrt{h_\pi(s)}\,ds}.
\]
Give \(\mathcal P(\mathcal X)\) the weak topology. Regard \(\mathcal N_\infty\) as a convex subset of \(\R\times L^1([0,T])\) through the map
\[
    n\longmapsto(n(0+),n).
\]
For fixed \(n\), the representation above shows that \(L(n,\cdot)\) is continuous and affine. We next check continuity in \(n\). Suppose
\[
    (n_k(0+),n_k)
    \longrightarrow
    (n(0+),n)
    \quad\text{in }\R\times L^1([0,T]).
\]
Choose \(c>0\) so that, for all large \(k\),
\[
    n_k(t)\ge c,
    \qquad
    n(t)\ge c
\]
for almost every \(t\). Then
\[
    \left\|\frac1{n_k}-\frac1n\right\|_{L^1}
    \le
    \frac1{c^2}\|n_k-n\|_{L^1}
    \longrightarrow0.
\]
Since \(g^\pi\ge0\), for \(A>0\),
\begin{align*}
    \int_0^T
    g^\pi(t)
    \left|\frac1{n_k(t)}-\frac1{n(t)}\right|
    dt
    &\le
    A\left\|\frac1{n_k}-\frac1n\right\|_{L^1} \\
    &\quad+
    \frac2c
    \int_0^T
    g^\pi(t)\one(g^\pi(t)>A)\,dt.
\end{align*}
Letting \(k\to\infty\) and then \(A\to\infty\) proves convergence of the integral term in \(L(n_k,\pi)\). The first term also converges because \(n_k(0+)\to n(0+)>0\). Thus
\[
    L(n_k,\pi)\longrightarrow L(n,\pi).
\]
Thus \(L(\cdot,\pi)\) is continuous. It is convex because \(\Gamma_0^\pi\ge0\), \(g^\pi\ge0\), and \(z\mapsto1/z\) is convex. The set \(\mathcal P(\mathcal X)\) is compact and convex. The set \(\mathcal N_\infty\) is convex. Sion's theorem therefore gives
\[
    \inf_{n\in\mathcal N_\infty}
    \sup_{\pi\in\mathcal P(\mathcal X)}L(n,\pi)
    =
    \sup_{\pi\in\mathcal P(\mathcal X)}
    \inf_{n\in\mathcal N_\infty}L(n,\pi)
    =
    \sup_{\pi\in\mathcal P(\mathcal X)}H(\pi).
\]
Because \(H(\pi)=\inf_nL(n,\pi)\), it is upper semicontinuous. Compactness of \(\mathcal P(\mathcal X)\) therefore gives a maximizer. Denote the maximum by \(V\). Choose \(x\) with \(0<F(x)<1\). The greatest convex minorant meets the cumulative variance at \(T\), so the covariance decomposition gives
\[
    \int_0^T h_{\delta_x}(t)\,dt
    =
    \Gamma_0(x,x)
    +
    \int_0^T g_{x,x}(t)\,dt
    =
    F(x)(1-F(x))
    >
    0.
\]
Hence \(V\ge H(\delta_x)>0\). Fix a maximizer \(\pi^*\) and write \(h^*=h_{\pi^*}\). Suppose first that \(h^*(0+)>0\). The fixed-mixture result gives
\[
    n^*(t)=\frac{\sqrt{h^*(t)}}{\sqrt V}.
\]
Choose \((n_k)\subset\mathcal N_\infty\) such that
\[
    \sup_\pi L(n_k,\pi)\longrightarrow V.
\]
Since \(H(\pi^*)=V\),
\[
    V
    \le
    \int_0^T\frac{h^*(t)}{n_k(t)}\,dt
    \le
    L(n_k,\pi^*)
    \le
    \sup_\pi L(n_k,\pi)
    \longrightarrow V.
\]
Since both \(n_k\) and \(n^*\) integrate to one, Cauchy-Schwarz gives
\[
    \|n_k-n^*\|_1^2
    \le
    \int_0^T
    \frac{(n_k(t)-n^*(t))^2}{n_k(t)}\,dt
    =
    \frac1V
    \int_0^T\frac{h^*(t)}{n_k(t)}\,dt
    -1
    \longrightarrow0.
\]
Thus \(n_k\to n^*\) in \(L^1([0,T])\). After passing to an almost-everywhere convergent subsequence, monotonicity and right-continuity give
\[
    \limsup_k n_k(0+)\le n^*(0+).
\]
Fatou's lemma now gives, for every \(\pi\),
\[
    L(n^*,\pi)
    \le
    \liminf_k L(n_k,\pi)
    \le V.
\]
Since \(L(n^*,\pi^*)=V\), we have
\[
    \sup_\pi L(n^*,\pi)=V.
\]
Thus \(n^*\) is optimal. Returning to \(r\) gives the formula in the theorem. If instead \(h^*(0+)=0\), a minimizer \(\overline n\) would satisfy
\[
    V
    =
    \inf_n L(n,\pi^*)
    \le
    L(\overline n,\pi^*)
    \le
    \sup_\pi L(\overline n,\pi)
    =
    V.
\]
It would therefore attain the fixed-mixture infimum, which is impossible when \(h^*(0+)=0\). Hence the infimum is not attained.
\end{proof}

\subsection{Proof of Corollary~\ref{cor:minimax-var-opt}}
\label{app:proof-minimax-var-opt}

\begin{proof}
Set \(L(n,a)=\sum_i a_i/n_i\). The minimax identity follows from the Sion argument in the proof of Theorem~\ref{thm:functional-minimax-opt} and the fixed-profile calculation in the proof of Lemma~\ref{lem:var-opt}, with sums in place of integrals. For zero coordinates, apply the calculation to \(a+\varepsilon\mathbf 1\) and let \(\varepsilon\downarrow0\). Compactness of \(\mathcal V\) and continuity of PAVA give the maximum.

Let \(c=\min_i\widetilde v_i(x_0)>0\). Since \(\sup_xL(n,\widetilde V(x))\ge c\sum_i n_i^{-1}\), every minimizing sequence stays away from zero coordinates. Compactness of \(\overline{\mathcal N}\) then gives the minimum. For any maximizer \(a^*\) and minimizer \(n^*\), the minimax identity gives
\[
    \inf_n L(n,a^*)
    = L(n^*,a^*)
    = \max_a L(n^*,a).
\]
This value is positive by the preceding bound. Thus \(n^*\) attains the fixed-profile infimum. The equality case in the proof of Lemma~\ref{lem:var-opt} forces every coordinate of \(\PAVA(a^*)\) to be positive and gives the stated formula.
\end{proof}

\subsection{Proof of Proposition~\ref{prop:dyadic-mixture-approximation}}
\label{app:proof-dyadic-mixture-approximation}

\begin{proof}
Let
\[
    \mathcal R_m
    =
    \left\{
        r\in\R_{>0}^m:
        1=r_0\le\cdots\le r_{m-1}
    \right\}.
\]
Every \(r\in\mathcal R_m\) corresponds to the feasible allocation \(n_i=Br_i/\sum_jr_j\). Corollary~\ref{cor:minimax-var-opt} and the definition of \(r^*\) therefore give
\[
    V^*
    =
    \inf_{r\in\mathcal R_m}
    \left(\sum_i r_i\right)
    \sup_{x\in\R^d}
    \sum_i\frac{\widetilde v_i(x)}{r_i}.
\]
We first prove the lower bound. Fix any feasible weights \(\lambda\), and define
\[
    z_i^{(s)}
    =
    \frac{R_i^{(s)}}{C_s},
    \qquad
    \overline z_i
    =
    \sum_s\lambda_sz_i^{(s)}.
\]
Since \(C_s=\sum_iR_i^{(s)}\), we have \(\sum_i z_i^{(s)}=1\) and \(\sum_i\overline z_i=1\). Each \(z^{(s)}\) is positive and nondecreasing. Hence, \(\overline z\) is also positive and nondecreasing. Convexity of \(z\mapsto1/z\) gives
\[
    \sum_i\frac{\widetilde v_i(x)}{\overline z_i}
    \le
    \sum_s\lambda_sC_s
    \sum_i\frac{\widetilde v_i(x)}{R_i^{(s)}}.
\]
Set
\[
    r_i
    =
    \frac{\overline z_i}{\overline z_0}.
\]
Then \(r\in\mathcal R_m\). Since \(\sum_i\overline z_i=1\),
\[
    \left(\sum_i r_i\right)
    \left(\sum_i\frac{\widetilde v_i(x)}{r_i}\right)
    =
    \sum_i\frac{\widetilde v_i(x)}{\overline z_i}.
\]
Therefore,
\[
    V^*
    \le
    \left(\sum_i r_i\right)
    \sup_{x\in\R^d}
    \sum_i\frac{\widetilde v_i(x)}{r_i}
    \le
    \sup_{x\in\R^d}
    \sum_s\lambda_sC_s
    \sum_i\frac{\widetilde v_i(x)}{R_i^{(s)}}.
\]
This holds for every feasible \(\lambda\). Taking the minimum over \(\lambda\) proves the lower bound.

We now prove the upper bound. Let \(U\sim\operatorname{Unif}[0,1)\). For each candidate profile \(R^{(s)}\), define
\[
    I_s
    =
    \{u\in[0,1):R(u)=R^{(s)}\},
    \qquad
    w_s
    =
    \Pbb(U\in I_s).
\]
The sets \(I_s\) partition \([0,1)\), and \(w_s\) is their length. Thus \(w_s\ge0\) and \(\sum_s w_s=1\), so \(w\) is a feasible choice of mixture weights. Fix \(i,j\) and write
\[
    \log_2\!\left(\frac{r_i^*}{r_j^*}\right)
    =
    k+f,
    \qquad
    k\in\mathbb Z,
    \quad
    0\le f<1.
\]
The fractional part of \(\log_2r_j^*+U\) is uniform on \([0,1)\). Hence, the rounded exponent difference equals \(k\) with probability \(1-f\) and \(k+1\) with probability \(f\). Therefore,
\[
    \E_U\!\left[
        \frac{R_i(U)}{R_j(U)}
    \right]
    =
    (1-f)2^k+f2^{k+1}
    =
    \frac{r_i^*}{r_j^*}(1+f)2^{-f}.
\]
The last factor satisfies
\[
    \max_{0\le f<1}(1+f)2^{-f}
    =
    \frac{2}{e\log 2}.
\]
For every \(x\in\R^d\),
\[
    \sum_s w_sC_s
    \sum_j\frac{\widetilde v_j(x)}{R_j^{(s)}}
    =
    \E_U\!\left[
        \left(\sum_iR_i(U)\right)
        \left(\sum_j\frac{\widetilde v_j(x)}{R_j(U)}\right)
    \right].
\]
The ratio bound therefore gives
\[
    \sum_s w_sC_s
    \sum_j\frac{\widetilde v_j(x)}{R_j^{(s)}}
    \le
    \frac{2}{e\log 2}
    \left(\sum_i r_i^*\right)
    \left(\sum_j\frac{\widetilde v_j(x)}{r_j^*}\right).
\]
Taking the supremum over \(x\) and using the definition of \(V^*\) gives
\[
    \min_{\substack{\lambda_s\ge0\\ \sum_s\lambda_s=1}}
    \sup_{x\in\R^d}
    \sum_s\lambda_sC_s
    \sum_j\frac{\widetilde v_j(x)}{R_j^{(s)}}
    \le
    \sup_{x\in\R^d}
    \sum_s w_sC_s
    \sum_j\frac{\widetilde v_j(x)}{R_j^{(s)}}
    \le
    \frac{2}{e\log 2}V^*.
\]
This proves the upper bound.
\end{proof}

\subsection{Proof of Proposition~\ref{prop:mixture-functional-limit}}
\label{app:proof-mixture-functional-limit}

\begin{proof}
For any two candidate types \(s\) and \(t\), logarithmic rounding gives
\(R_{i,B}^{(s)}/R_{i,B}^{(t)}\in[1/2,2]\), and hence
\(C_{s,B}\le2C_{t,B}\). Therefore,
\[
    \max_s C_{s,B}
    \le
    \frac{2\overline\gamma_B}{h_B}
    =
    O(h_B^{-1}),
    \qquad
    \sum_sC_{s,B}
    =
    O(h_B^{-2})
    =
    o(B),
\]
where the first bound follows by averaging \(C_{s,B}\le2C_{t,B}\) with respect to \(\lambda_{t,B}\), and the last uses \(S_B\le K_B=O(h_B^{-1})\) and \(Bh_B^2\to\infty\). Let
\[
    D_B
    =
    \sum_sM_{s,B}C_{s,B},
    \qquad
    w_{s,B}
    =
    \frac{M_{s,B}C_{s,B}}{D_B},
\]
be the realized budget and mixture weights after flooring. Then
\[
    0\le B-D_B<\sum_sC_{s,B}=o(B),
    \qquad
    \frac{D_B}{B}\longrightarrow1,
    \qquad
    \lVert w_B-\lambda_B\rVert_1
    \le
    \frac{2(B-D_B)}{D_B}
    \longrightarrow0.
\]
In particular,
\[
    h_B\sum_s
    \lvert w_{s,B}-\lambda_{s,B}\rvert C_{s,B}
    \le
    \left(\max_s h_BC_{s,B}\right)
    \lVert w_B-\lambda_B\rVert_1
    \longrightarrow0.
\]

Let \(Z_{s,B,j}\) be the vector of terminal descendants from root \(j\) of type \(s\), and let \(f_{s,B,x}(Z_{s,B,j})\) be their empirical CDF at \(x\). Then
\[
    \mathbb G_B^{\mathrm{mix}}(a_x)
    :=
    \sqrt{Bh_B}
    \left(
        \widehat F_B^{\mathrm{mix}}(x)-F_B(x)
    \right)
    =
    \sum_{s,j}
    a_{s,B}
    \left[
        f_{s,B,x}(Z_{s,B,j})-F_B(x)
    \right],
    \qquad
    a_{s,B}
    =
    \frac{C_{s,B}\sqrt{Bh_B}}{D_B}.
\]
Moreover,
\[
    \max_sa_{s,B}
    =
    O((Bh_B)^{-1/2})
    \longrightarrow0,
    \qquad
    Q_B
    :=
    \sum_sM_{s,B}a_{s,B}^2
    =
    \frac{B}{D_B}h_B\sum_sw_{s,B}C_{s,B}
    \longrightarrow
    \overline\gamma,
\]
where the last limit follows from the preceding bound and the definition of \(\overline\gamma_B\). Thus, no individual root dominates, while the total squared coefficient converges to \(\overline\gamma\). We next identify the limiting covariance. For one root of type \(s\), write
\[
    \Gamma_B^{(s)}(x,y)
    =
    \Gamma_{0,B}(x,y)
    +
    \sum_i
    \frac{v_{i,B}(x,y)}{R_{i,B}^{(s)}}.
\]
Independence of the roots and the definition of \(\overline r_B\) give
\begin{align*}
    \Cov\!\left(
        \mathbb G_B^{\mathrm{mix}}(a_x),
        \mathbb G_B^{\mathrm{mix}}(a_y)
    \right)
    &=
    \frac{B}{D_B}h_B
    \sum_s w_{s,B}C_{s,B}\Gamma_B^{(s)}(x,y),\\
    h_B\sum_s
    \lambda_{s,B}C_{s,B}\Gamma_B^{(s)}(x,y)
    &=
    \overline\gamma_B
    \left[
        \Gamma_{0,B}(x,y)
        +
        \sum_i
        \frac{v_{i,B}(x,y)}{\overline r_{i,B}}
    \right].
\end{align*}
Since \(\lvert\Gamma_B^{(s)}(x,y)\rvert\le1/4\), the preceding weight bound allows \(w_B\) to be replaced by \(\lambda_B\) at a cost of \(o(1)\). Proposition~\ref{prop:covariance-convergence} and the proof of Proposition~\ref{prop:covlimit}, with \(\overline r_B\) in place of \(r_B\), therefore give
\[
    \Cov\!\left(
        \mathbb G_B^{\mathrm{mix}}(a_x),
        \mathbb G_B^{\mathrm{mix}}(a_y)
    \right)
    \longrightarrow
    \overline\gamma
    \left[
        \Gamma_0(x,y)
        +
        \int_0^T
        \frac{g_{x,y}(t)}{\overline r(t)}\,dt
    \right].
\]
Since the centered root averages are bounded by one, for every \(\eta>0\),
\[
    \sum_{s,j}
    \E\!\left[
        \left\lVert
            a_{s,B}(f_{s,B,\cdot}-F_B)
        \right\rVert_{\cA}^2
        \mathbf 1\!\left\{
            \left\lVert
                a_{s,B}(f_{s,B,\cdot}-F_B)
            \right\rVert_{\cA}>\eta
        \right\}
    \right]
    \le
    Q_B\mathbf 1\{\max_sa_{s,B}>\eta\}
    \longrightarrow0.
\]
Thus, the envelope Lindeberg condition holds, and the covariance limit and the multivariate Lindeberg-Feller theorem give the finite-dimensional limits.

To upgrade this to process convergence, it remains to prove asymptotic \(\rho\)-equicontinuity. For every root type, Jensen's inequality gives
\[
    \E\!\left[
        \left(
            f_{s,B,x}(Z_{s,B,j})
            -
            f_{s,B,y}(Z_{s,B,j})
        \right)^2
    \right]
    \le
    \E\!\left[
        \left(
            a_x(X_{K_B}^B)-a_y(X_{K_B}^B)
        \right)^2
    \right]
    \le
    \rho(a_x,a_y)^2
    +
    4\lVert F_B-F\rVert_\infty.
\]
The proof of Theorem~\ref{thm:splitting-clt} shows that \(\lVert F_B-F\rVert_\infty\to0\). Hence, for every \(\delta_B\downarrow0\),
\[
    \sup_{\rho(a_x,a_y)<\delta_B}
    \sum_{s,j}a_{s,B}^2
    \E\!\left[
        \left(
            f_{s,B,x}(Z_{s,B,j})
            -
            f_{s,B,y}(Z_{s,B,j})
        \right)^2
    \right]
    \le
    Q_B
    \left(
        \delta_B^2
        +
        4\lVert F_B-F\rVert_\infty
    \right)
    \longrightarrow0.
\]
For the entropy condition, define
\[
    d_B(a_x,a_y)^2
    =
    \sum_{s,j}a_{s,B}^2
    \left(
        f_{s,B,x}(Z_{s,B,j})
        -
        f_{s,B,y}(Z_{s,B,j})
    \right)^2.
\]
Let \(\nu_B\) be the probability measure
\[
    \nu_B
    =
    \frac1{Q_B}
    \sum_{s,j}
    \frac{a_{s,B}^2}{R_{K_B-1,B}^{(s)}}
    \sum_{\ell=1}^{R_{K_B-1,B}^{(s)}}
    \delta_{Z_{s,B,j}^{(\ell)}}.
\]
Jensen's inequality gives
\[
    d_B(a_x,a_y)
    \le
    \sqrt{Q_B}\,
    \lVert a_x-a_y\rVert_{L^2(\nu_B)}.
\]
The first entropy bound in Lemma~\ref{lem:vc} therefore gives, uniformly over realizations,
\[
    N(\varepsilon,\cA,d_B)
    \le
    A\left(
        1+\frac{\sqrt{Q_B}}{\varepsilon}
    \right)^v.
\]
Consequently, for every \(\eta_B\downarrow0\),
\[
    \int_0^{\eta_B}
    \sqrt{\log N(\varepsilon,\cA,d_B)}\,d\varepsilon
    \lesssim
    \eta_B\sqrt{\log(1+\eta_B^{-1})}
    \longrightarrow0.
\]
The class \(\cA\) is pointwise measurable and \((\cA,\rho)\) is totally bounded by Lemma~\ref{lem:vc}. The envelope, increment, and entropy bounds above therefore imply asymptotic \(\rho\)-equicontinuity by the independent-array equicontinuity theorem \citep[Theorem 2.11.1]{WeakConvergence2023}.

The limiting covariance kernel is positive semidefinite, and Proposition~\ref{prop:covlimit} bounds its canonical semimetric by \(\sqrt{\overline\gamma}\rho\). Lemma~\ref{lem:vc} and Dudley's entropy theorem therefore give a tight Gaussian process with the covariance in the statement. Together with the finite-dimensional convergence and asymptotic equicontinuity,
\[
    \mathbb G_B^{\mathrm{mix}}
    \Rightarrow
    \sqrt{\overline\gamma}\,\G
    \qquad
    \text{in }\ell^\infty(\cA).
\]
The KS limit follows from the continuous mapping theorem.
\end{proof}

\subsection{Cost-weighted extensions}
\label{app:cost-weighted}

The main text assigns unit cost to every transition and uses an asymptotically uniform time partition. Neither restriction is essential, and neither holds in our experiments: the Heun sampler of Appendix~\ref{app:benchmark-details} spends two network evaluations on every step but the last, and the diffusion samplers use graded time grids. This subsection removes both. The budget is spent on cost rather than on elapsed time, so we measure time in cost. Each result below is then the corresponding unit-cost result in this variable.

Let \(c_{i,B}>0\) be the cost of propagating one particle through transition \(i\). One root tree costs
\[
    C_B
    =
    \sum_{i=0}^{K_B-1}c_{i,B}R_{i,B},
    \qquad
    N_{0,B}=B/C_B.
\]
For \(t\in[t_{i,B},t_{i+1,B})\), define the cost density
\[
    \omega_B(t)
    =
    \frac{h_Bc_{i,B}}
         {t_{i+1,B}-t_{i,B}}.
\]
Here \(\omega_B\) is interpreted as computational cost per unit of physical time, measured on the scale \(h_B\). Summing over transitions gives the exact identity
\[
    h_BC_B
    =
    \int_0^T\omega_B(t)r_B(t)\,dt.
\]
Thus the normalized cost of a root tree is the area under the splitting function when time is weighted by cost. Unit costs on an asymptotically uniform partition give \(\omega_B\to1\) uniformly, and the right side then converges to \(\gamma\), recovering the normalization of Section~\ref{sec:bias-variance-tradeoff}.

Only one hypothesis is needed. It asks that the cost densities have a limit, and that this limit pairs with the splitting functions in the obvious way.

\begin{assumption}\label{ass:cost}
There is a measurable \(\omega:[0,T]\to(0,\infty)\), bounded above and bounded away from zero, such that
\[
    \int_0^T\omega_B(t)r_B(t)\,dt
    \longrightarrow
    \gamma_\omega(r)
    :=
    \int_0^T\omega(t)r(t)\,dt
\]
whenever \(r_B\to r\) as in Assumption~\ref{ass:splitting-function}.
\end{assumption}
Uniform convergence \(\omega_B\to\omega\) is sufficient, since the \(r_B\) are bounded in \(L^1([0,T])\). Here \(\gamma_\omega(r)\) is interpreted as the limiting cost of one root tree. It reduces to \(\gamma(r)\) when \(\omega\equiv1\), and takes its place throughout.

Retain the remaining hypotheses of Corollary~\ref{cor:dg-limit}, and replace unit costs and the asymptotically uniform partition by Assumption~\ref{ass:cost}. If \(\lambda_B\to\lambda\in(0,\infty)\), then
\[
    B^{p/(1+2p)}(\widehat F_B-F_B)
    \Rightarrow
    \left(\frac{\gamma_\omega(r)}{\lambda}\right)^{1/2}\G
    \qquad\text{in }\ell^\infty(\cA),
\]
and \(B^{p/(1+2p)}\|\widehat F_B-F\|_\infty\) is bounded in probability. The costs and the partition enter Corollary~\ref{cor:dg-limit} only through the limit of \(h_BC_B\), which the identity above and Assumption~\ref{ass:cost} now supply. Theorem~\ref{thm:splitting-clt} uses neither, so the rest of that proof is unchanged.

Cost therefore enters the limit through \(\gamma_\omega(r)\) and nowhere else. Proposition~\ref{prop:ks-bounds} applies with \(\gamma(r)\) replaced by \(\gamma_\omega(r)\), and dropping the same logarithmic factor gives the surrogate \(\gamma_\omega(r)^{1/2}\Gamma^*\). Minimizing its square gives the cost-weighted proxy objective
\[
    \inf_{r\in\mathcal R}
    \gamma_\omega(r)
    \sup_{x\in\R^d}
    \left[
        \Gamma_0(x,x)
        +
        \int_0^T\frac{g_{x,x}(t)}{r(t)}\,dt
    \right].
\]

We now measure time in cost. Define the cumulative cost and the total cost
\[
    \Omega(t)=\int_0^t\omega(u)\,du,
    \qquad
    \Omega_T=\Omega(T).
\]
Since \(\omega>0\), the map \(\Omega\) is strictly increasing, so \(s=\Omega(t)\) is a change of variables from \([0,T]\) onto \([0,\Omega_T]\). Writing \(\overline r(s)=r(\Omega^{-1}(s))\) gives
\[
    \gamma_\omega(r)
    =
    \int_0^{\Omega_T}\overline r(s)\,ds,
    \qquad
    \int_0^T\frac{g_{x,x}(t)}{r(t)}\,dt
    =
    \int_0^{\Omega_T}\frac{\overline g_{x,x}(s)}{\overline r(s)}\,ds,
    \qquad
    \overline g_{x,x}(s)
    =
    \frac{g_{x,x}(\Omega^{-1}(s))}
         {\omega(\Omega^{-1}(s))}.
\]
Since \(\Omega\) is increasing, \(\overline r\) is again nondecreasing with \(\overline r(0+)=1\), so \(r\mapsto\overline r\) matches \(\mathcal R\) with the same class on \([0,\Omega_T]\). The cost-weighted objective is thus the unit-cost objective in the variable \(s\). Accordingly, for \(\pi\in\mathcal P(\mathcal X)\), set
\[
    \overline g^\pi(s)
    =
    \frac{g^\pi(\Omega^{-1}(s))}
         {\omega(\Omega^{-1}(s))},
    \qquad
    \overline h_\pi
    =
    \PAVA(\Gamma_0^\pi,\overline g^\pi),
    \qquad 0\le s<\Omega_T.
\]

Under Assumption~\ref{ass:terminal-weak-conv},
\[
    \inf_{r\in\mathcal R}
    \gamma_\omega(r)
    \sup_{x\in\R^d}
    \left[
        \Gamma_0(x,x)
        +
        \int_0^T\frac{g_{x,x}(t)}{r(t)}\,dt
    \right]
    =
    \max_{\pi\in\mathcal P(\mathcal X)}
    \left(
        \int_0^{\Omega_T}\sqrt{\overline h_\pi(s)}\,ds
    \right)^2.
\]
For any maximizer \(\pi^*\), the infimum is attained if and only if \(\overline h_{\pi^*}(0+)>0\), in which case
\[
    r^*(t)
    =
    \sqrt{
        \frac{\overline h_{\pi^*}(\Omega(t))}
             {\overline h_{\pi^*}(0+)}
    },
    \qquad 0\le t<T,
\]
defines a minimizer. Both claims are Theorem~\ref{thm:functional-minimax-opt} applied on \([0,\Omega_T]\) with inputs \((\Gamma_0^\pi,\overline g^\pi)\), read back through \(\Omega\).

Geometrically, \(\overline h_\pi\) is obtained by plotting cumulative variance against cumulative cost, with \(\Gamma_0^\pi\) as an initial atom at cost zero, and taking the slopes of the greatest convex minorant. In the main text every transition costs the same, so cumulative cost is the index and this is the equal-weight construction of Appendix~\ref{app:pava}. Here the horizontal axis is stretched by \(\omega\). Equivalently, in physical time PAVA acts on \(g^\pi(t)/\omega(t)\) with weights \(\omega(t)\,dt\).

The finite-budget construction is the same, with finitely many points, and defines the cost-weighted fit referred to in Appendix~\ref{app:pava}. Fix \(B\), put \(m=K_B\), and write \(c_i=c_{i,B}\). For \(z\in[0,\infty)^m\), let \(\PAVA_c(z)_i\) be the slope of the greatest convex minorant of the polygonal line through
\[
    \left(
        \sum_{j<k}c_j,
        \ \sum_{j<k}z_j
    \right),
    \qquad 0\le k\le m.
\]
As in Appendix~\ref{app:pava}, these slopes are constant on the final blocks, so that
\[
    \PAVA_c(z)_i
    =
    \frac{\sum_{j\in I}z_j}
         {\sum_{j\in I}c_j},
    \qquad i\in I.
\]
Cumulative cost has replaced the index on the horizontal axis, and the block average has become a cost-weighted average. For \(c\equiv1\) this is the equal-weight fit \(\PAVA\).

Replace the feasible set \(\mathcal N\) of Corollary~\ref{cor:minimax-var-opt} by
\[
    \mathcal N_c
    =
    \left\{
        n\in\R_{>0}^m:
        n_0\le\cdots\le n_{m-1},
        \ \sum_{i=0}^{m-1}c_in_i=B
    \right\},
\]
and retain \(\widetilde v_i\) and \(\mathcal V\) as defined there.

Under the assumptions of Corollary~\ref{cor:minimax-var-opt}, the minimum below is attained, and
\[
    \min_{n\in\mathcal N_c}
    \sup_{x\in\R^d}
    \sum_{i=0}^{m-1}\frac{\widetilde v_i(x)}{n_i}
    =
    \frac1B
    \max_{a\in\mathcal V}
    \left(
        \sum_{i=0}^{m-1}c_i\sqrt{\PAVA_c(a)_i}
    \right)^2.
\]
If \(a^*\) is a maximizer and \(h^*=\PAVA_c(a^*)\), then \(h_i^*>0\) for every \(i\), and every minimizing allocation satisfies
\[
    n_i^*
    =
    B
    \frac{\sqrt{h_i^*}}
         {\sum_{j=0}^{m-1}c_j\sqrt{h_j^*}},
    \qquad 0\le i<m.
\]
The argument of Corollary~\ref{cor:minimax-var-opt} goes through with two substitutions. The PAVA block inequality becomes \(\sum_ia_iq_i\ge\sum_ic_ih_iq_i\) for nonincreasing \(q\), which is again summation by parts against a minorant gap that is nonnegative and vanishes at both ends. Cauchy-Schwarz is then applied as
\[
    \left(\sum_{i=0}^{m-1}\frac{c_ih_i}{n_i}\right)
    \left(\sum_{i=0}^{m-1}c_in_i\right)
    \ge
    \left(\sum_{i=0}^{m-1}c_i\sqrt{h_i}\right)^2,
\]
with equality when \(n\propto\sqrt h\), which is nondecreasing and hence feasible. The Sion argument and the positivity of \(h^*\) are unchanged.

This is the allocation problem solved in the experiments. Appendix~\ref{app:allocation} gives the Frank-Wolfe scheme used to compute a maximizing \(a^*\), and Appendix~\ref{app:ou-oracle} uses it for the oracle benchmark.

\subsection{Extension to MMD}
\label{app:mmd-theory}

Let \(k\) be a bounded kernel on \(\R^d\) with reproducing kernel Hilbert space \(\mathcal H\). The maximum mean discrepancy between two distributions \(P\) and \(P'\) is
\[
    \operatorname{MMD}(P,P')
    =
    \sup_{\|f\|_{\mathcal H}\le1}\left|Pf-P'f\right|.
\]
For a Gaussian kernel, it is zero only when \(P=P'\) \citep{GrettonEtAl2012}. Computing it between two samples requires the kernel at every pair of points. Random Fourier features \citep{RahimiRecht2007} avoid this. For the Gaussian kernel \(k(z,z')=\exp(-\|z-z'\|_2^2/(2\eta^2))\), Bochner's theorem gives \(k(z,z')=\E[\cos(\omega^\top(z-z'))]\) with \(\omega\sim N(0,\eta^{-2}I)\). Draw frequencies \(\omega_1,\ldots,\omega_M\) and define \(\phi:\R^d\to\R^{d_\phi}\), with \(d_\phi=2M\), by
\[
    \phi(z)
    =
    \frac1{\sqrt M}
    \left(
        \cos\omega_1^\top z,\ldots,\cos\omega_M^\top z,\ 
        \sin\omega_1^\top z,\ldots,\sin\omega_M^\top z
    \right).
\]
Then \(\phi(z)^\top\phi(z')=M^{-1}\sum_{j}\cos(\omega_j^\top(z-z'))\approx k(z,z')\), and \(\|\phi(z)\|_2=1\) for every \(z\). Writing \(P\phi=(P\phi_1,\ldots,P\phi_{d_\phi})\), the MMD becomes
\[
    \operatorname{MMD}(P,P')
    \approx
    \left\|P\phi-P'\phi\right\|_2,
\]
which costs time linear in the sample sizes. In our experiments, \(M=1{,}024\) and the kernel is an equal mixture of Gaussian kernels at eight bandwidths. We apply \(\phi\) in the native space of the sampler. For CIFAR-10 this is the image, after eight-bit quantization and standardization. For FFHQ it is the standardized latent (Appendix~\ref{app:evaluation-metrics}).

The error therefore depends on the splitting estimator only through the \(d_\phi\) means \(\widehat P_B\phi_k\), and it is a continuous function of them. The steps used for KS carry over. We prove a central limit theorem for the vector \(\widehat P_B\phi\), bound the expected norm of its limit, and solve the resulting allocation problem.

\paragraph{Notation.}
Let \(\widehat P_B\phi\) be the average of \(\phi\) over the terminal particles of the splitting estimator, and let \(P_B\phi=\E[\phi(X^B_{K_B})]\). The vector analogue of \(u^B_{i,x}\) is
\[
    \psi_i(z)
    =
    \E\!\left[\phi(X^B_{K_B})\mid X^B_i=z\right]\in\R^{d_\phi},
    \qquad 0\le i\le K_B.
\]
Define the \(d_\phi\times d_\phi\) matrices
\[
    \Gamma^\phi_{0,B}=\Cov\!\left(\psi_0(X^B_0)\right),
    \qquad
    v^\phi_{i,B}=\E\!\left[\Cov\!\left(\psi_{i+1}(X^B_{i+1})\mid X^B_i\right)\right].
\]
Let \(\Gamma_\phi\) be the matrix whose \((k,l)\) entry is the limit \(\Gamma\) of Proposition~\ref{prop:covlimit}, with \(a_x\) and \(a_y\) replaced by \(\phi_k\) and \(\phi_l\).

For CIFAR-10, \(\phi\) is not continuous because of the quantization step. It is discontinuous only where a pixel value lies on a rounding threshold. Let \(D_\phi\) be this set of discontinuity points. We need \(\Pbb(X_T\in D_\phi)=0\), which holds when every pixel of \(X_T\) has a continuous distribution. For FFHQ there is no quantization, so \(\phi\) is continuous and the condition holds trivially.

\begin{lemma}\label{lem:bounded-test-functions}
For every \(B\),
\[
    \Cov\!\left(\widehat P_B\phi\right)
    =
    \frac1{N_{0,B}}
    \left[
        \Gamma^\phi_{0,B}
        +
        \sum_{i=0}^{K_B-1}\frac{v^\phi_{i,B}}{R_{i,B}}
    \right].
\]
If, in addition, Assumptions~\ref{ass:terminal-strong-consistency} and~\ref{ass:splitting-function} hold and \(\Pbb(X_T\in D_\phi)=0\), then
\[
    N_{0,B}\Cov\!\left(\widehat P_B\phi\right)\longrightarrow\Gamma_\phi.
\]
\end{lemma}

\begin{proof}
For the covariance formula, apply Lemma~\ref{lem:splitting-cov} to each pair of coordinates \(\phi_k,\phi_l\). Its proof uses only that \(a_x\) is bounded, and each \(\phi_k\) is bounded.

For the limit, the proofs of Propositions~\ref{prop:covariance-convergence} and~\ref{prop:covlimit} use two facts about \(a_x\). The first is \(|a_x|\le1\), and \(|\phi_k|\le1\) as well. The second is \(a_x(X^B_{K_B})\to a_x(X_T)\) in \(L^2\). The same holds for \(\phi_k\). Assumption~\ref{ass:terminal-strong-consistency} gives \(X^B_{K_B}\to X_T\) in probability. Since \(\Pbb(X_T\in D_\phi)=0\) and \(\phi_k\) is bounded,
\[
    \phi_k(X^B_{K_B})
    \xrightarrow{\ \Pbb\ }
    \phi_k(X_T)
    \qquad\text{and hence}\qquad
    \phi_k(X^B_{K_B})
    \xrightarrow{\ L^2\ }
    \phi_k(X_T).
\]
The first limit is the continuous mapping theorem, and the second follows from boundedness. The rest of both proofs is unchanged.
\end{proof}

The vector \(\widehat P_B\phi\) is the average of \(N_{0,B}\) i.i.d.\ root-tree averages, which are bounded and have mean \(P_B\phi\). Hence, if \(N_{0,B}\to\infty\), the Lindeberg-Feller theorem gives
\[
    \sqrt{N_{0,B}}\left(\widehat P_B\phi-P_B\phi\right)
    \Rightarrow
    \mathbb Z\sim N(0,\Gamma_\phi).
\]
Since \(\|h\|_2\le\sqrt{d_\phi}\,\|h\|_\infty\) for \(h\in\R^{d_\phi}\), the map \(h\mapsto\|h\|_2\) is continuous, and the continuous mapping theorem gives
\[
    \sqrt{N_{0,B}}\left\|\widehat P_B\phi-P_B\phi\right\|_2
    \Rightarrow
    \|\mathbb Z\|_2.
\]

\begin{lemma}\label{lem:gaussian-norm}
Let \(\mathbb Z\sim N(0,\Gamma)\) in \(\R^{d_\phi}\). Then
\[
    \sqrt{2/\pi}\,\sqrt{\tr\Gamma}
    \le
    \E\|\mathbb Z\|_2
    \le
    \sqrt{\tr\Gamma}.
\]
\end{lemma}

\begin{proof}
For the upper bound, Jensen's inequality gives
\[
    \E\|\mathbb Z\|_2
    \le
    \left(\E\|\mathbb Z\|_2^2\right)^{1/2}
    =
    \sqrt{\tr\Gamma}.
\]

For the lower bound, let \(\lambda_1,\ldots,\lambda_{d_\phi}\) be the eigenvalues of \(\Gamma\) with orthonormal eigenvectors \(e_1,\ldots,e_{d_\phi}\). Write \(\mathbb Z=\sum_k\sqrt{\lambda_k}\,\xi_ke_k\), where \(\xi_1,\ldots,\xi_{d_\phi}\) are i.i.d.\ standard normal. By the Cauchy-Schwarz inequality,
\[
    \sum_k\lambda_k|\xi_k|
    =
    \sum_k\sqrt{\lambda_k}\cdot\sqrt{\lambda_k}\,|\xi_k|
    \le
    \sqrt{\tr\Gamma}\,\|\mathbb Z\|_2.
\]
Taking expectations and using \(\E|\xi_k|=\sqrt{2/\pi}\) gives \(\sqrt{2/\pi}\,\tr\Gamma\le\sqrt{\tr\Gamma}\,\E\|\mathbb Z\|_2\).
\end{proof}

\paragraph{Allocation.}
Lemma~\ref{lem:gaussian-norm} is the analogue of Proposition~\ref{prop:ks-bounds}. Its two sides differ only by the constant \(\sqrt{\pi/2}\), so we choose the allocation to minimize the trace of the covariance. Let \(n_i=N_{0,B}R_{i,B}\) be the number of particles propagated through transition \(i\). Taking the trace of the covariance formula in Lemma~\ref{lem:bounded-test-functions} gives
\[
    \tr\Cov\!\left(\widehat P_B\phi\right)
    =
    \sum_{i=0}^{K_B-1}\frac{w_i}{n_i},
    \qquad
    w_i
    =
    \tr v^\phi_{i,B}+\one_{\{i=0\}}\tr\Gamma^\phi_{0,B}.
\]
This has the same form as the variance at a single threshold in Section~\ref{sec:discretization-splitting-estimator}, with \(w_i\) in place of \(\widetilde v_i\). The minimizer over monotone allocations is therefore given by Lemma~\ref{lem:var-opt}. With segment costs \(c_i\), Appendix~\ref{app:cost-weighted} applies with \(\mathcal V=\{w\}\), and
\[
    n_i^*
    =
    B\frac{\sqrt{h_i}}{\sum_{j}c_j\sqrt{h_j}},
    \qquad
    h=\PAVA_c(w).
\]
The dyadic mixture of Proposition~\ref{prop:dyadic-mixture-approximation} then uses a single tree type, and the factor \(2/(e\log2)\) still holds. The weights \(w_i\) can be written in terms of second moments, which is how Phase~1 estimates them in Appendix~\ref{app:mmd-procedure}. Let \(Q_i=\E\|\psi_i(X^B_i)\|_2^2\). Since \(\psi_i(X^B_i)\) is a martingale, \(\tr v^\phi_{i,B}=Q_{i+1}-Q_i\) and \(\tr\Gamma^\phi_{0,B}=Q_0-\|P_B\phi\|_2^2\). Hence
\[
    w_0=Q_1-\|P_B\phi\|_2^2,
    \qquad
    w_i=Q_{i+1}-Q_i,
    \quad 1\le i<K_B,
\]
and \(Q_{K_B}=\E\|\phi(X^B_{K_B})\|_2^2=1\).

\section{Two-Phase Algorithms}
\label{app:two-phase}

This appendix gives the two-phase procedures used in the experiments. Both use Phase~1 to estimate the variance added by each path segment and Phase~2 to run a mixture of dyadic trees. They differ in what Phase~1 estimates and in how the allocation problem is solved.

\subsection{KS}
\label{app:allocation-estimation}

This procedure is used for the OU, overdamped Langevin, and EDM models. Phase~1 uses complete paths to estimate the variance added by each path segment. We then solve a cost-weighted allocation problem. Phase~2 converts the relaxed allocation into a finite mixture of exact dyadic trees. Appendix~\ref{app:mmd-procedure} gives the corresponding procedure for the MMD.

\subsubsection{Phase~1: Simulation, Query Generation, and Estimation}
\label{app:phase-one}

The allowed split times
\[
    0=t_0<t_1<\cdots<t_{L-1}<t_L=T
\]
divide each path into \(L\) segments. Let \(c_i>0\) be the cost of propagating one particle through segment \(i\). One complete path costs
\[
    C_{\mathrm{full}}
    =
    \sum_{i=0}^{L-1}c_i.
\]
Phase~1 receives budget \(B_1\). It simulates
\[
    J
    =
    \left\lfloor
    \frac{B_1}{C_{\mathrm{full}}}
    \right\rfloor
\]
independent complete paths without splitting. Each path is recorded at \(t_0,\ldots,t_L\). The remaining Phase~2 budget is
\[
    B_2
    =
    B-JC_{\mathrm{full}}.
\]
Define the Phase~1 empirical CDF and its realized budget share by
\[
    \widehat F_1(x)
    =
    \frac1J\sum_{j=1}^J
    \one\!\left(X_T^{(j)}\le x\right),
    \qquad
    \beta_{\mathrm P}
    =
    \frac{JC_{\mathrm{full}}}{B}.
\]
We retain these terminal samples for the final estimate, where \(\beta_{\mathrm P}\) is the fraction of the total budget actually spent on the complete Phase~1 paths.
Given these paths, we want to estimate the quantities \(v_i(x)\) for all \(x \in \R^d\). We replace this with a finite-query approximation using \(Q\) queries. Further, to avoid sparse queries, we restrict the number of active coordinates in each query to be at most \(d_{\max}\). We use \(d_{\max}=2\). Each repetition draws its own set of queries.

For query \(q\), draw
\[
    k_q
    \sim
    \operatorname{Unif}\{1,\ldots,d_{\max}\}.
\]
Select \(k_q\) coordinates uniformly without replacement. Set
\[
    \mu_q
    =
    0.05
    +
    \left(q-\tfrac12\right)\frac{0.90}{Q},
    \qquad
    1\le q\le Q.
\]
These levels are evenly spaced between \(0.05\) and \(0.95\). For each query, draw
\[
    (s_{q,1},\ldots,s_{q,k_q})
    \sim
    \operatorname{Dirichlet}(1,\ldots,1).
\]
For each selected coordinate \(a\), set \(x_{q,a}\) to its Phase~1 empirical quantile at level \(\mu_q^{s_{q,a}}\). Set every unselected coordinate to \(+\infty\). Those coordinates then impose no restriction on the event.

Given these queries, we now estimate the variance contributions. For each query \(q\), let \(x_q\) be the corresponding threshold vector. Define
\[
    Y_q^{(j)}
    =
    \one\!\left(X_T^{(j)}\le x_q\right),
    \qquad
    \widehat F_q
    =
    \frac1J\sum_{j=1}^JY_q^{(j)}.
\]
Let \(F_B\) denote the terminal CDF on the discretization used at budget \(B\). For each query, define
\[
    u_q(t_i,z)
    =
    \Pbb(X_T\le x_q\mid X_{t_i}=z)
\]
and
\[
    Q_i(x_q)
    =
    \E\!\left[u_q(t_i,X_{t_i})^2\right].
\]
We estimate \(u_q(t_i,\cdot)\) by regressing \(Y_q^{(j)}\) on the path state at \(t_i\) and the query \(x_q\). We assign the pilot paths cyclically to \(K\) folds. For each fold, we train one MLP on the other folds and predict the held-out paths. The networks have hidden widths \(128\) and \(64\), SiLU activations, and a sigmoid output.

One MLP is shared across all recorded times and queries within a fold. Its input consists of the active-coordinate margins, time, the fraction of active coordinates, and the logit of the query mass. Each input is standardized before fitting. The label is \(Y_q^{(j)}\).

We train for five epochs with squared-error loss. Let \(p_{iq}^{(j)}\) be the held-out prediction for path \(j\), time \(t_i\), and query \(q\). For a fixed prediction rule \(p\),
\[
    \E\!\left[
        2p(X_{t_i})Y_q-p(X_{t_i})^2
    \right]
    =
    Q_i(x_q)
    -
    \E\!\left[
        \bigl(p(X_{t_i})-u_q(t_i,X_{t_i})\bigr)^2
    \right].
\]
This identity motivates
\[
    \widehat Q_i^{\mathrm{raw}}(x_q)
    =
    \frac1J
    \sum_{j=1}^J
    \left[
        2p_{iq}^{(j)}Y_q^{(j)}
        -
        \bigl(p_{iq}^{(j)}\bigr)^2
    \right],
    \qquad
    0\le i<L.
\]
We set
\[
    \widehat Q_L^{\mathrm{raw}}(x_q)
    =
    \widehat F_q.
\]
We note that our estimator is not unbiased. Each \(p_{iq}^{(j)}\) comes from an MLP trained without path \(j\), but the query thresholds and \(\widehat F_q\) still use the full Phase~1 sample. Now, see that the true second moments satisfy
\[
    F_B(x_q)^2
    \le
    Q_0(x_q)
    \le
    \cdots
    \le
    Q_L(x_q)
    =
    F_B(x_q).
\]
Sampling error can violate this order. We first clip the raw estimates to \([\widehat F_q^2,\widehat F_q]\). We then apply nondecreasing PAVA to the full sequence while holding the terminal value fixed at \(\widehat F_q\). Write \(\widehat Q_i(x_q)\) for the projected values. The root population controls both the initial-state variance and the variance added in the first segment. The estimated contributions used by the optimizer are
\[
    \widehat v_0(x_q)
    =
    \widehat Q_1(x_q)-\widehat F_q^2
\]
and
\[
    \widehat v_i(x_q)
    =
    \widehat Q_{i+1}(x_q)-\widehat Q_i(x_q),
    \qquad
    1\le i<L.
\]

\subsubsection{Optimization}
\label{app:allocation}

Let \(n_i>0\) be the relaxed total number of particles propagated through segment \(i\). For query \(q\), define
\[
    \ell_q(n)
    =
    \sum_{i=0}^{L-1}
    \frac{\widehat v_i(x_q)}{n_i}.
\]
We solve
\[
    \min_n
    \max_{1\le q\le Q}
    \ell_q(n)
\]
subject to
\[
    \sum_{i=0}^{L-1}c_in_i=B_2,
    \qquad
    0<n_0\le\cdots\le n_{L-1}.
\]
We solve the query-mixture dual with Frank-Wolfe. Let \(\pi_q\ge0\) be the current query weights, with
\[
    \sum_{q=1}^Q\pi_q=1.
\]
Form
\[
    p_i
    =
    \sum_{q=1}^Q
    \pi_q\widehat v_i(x_q).
\]
Apply the cost-weighted fit of Appendix~\ref{app:cost-weighted} to obtain \(h=\PAVA_c(p)\). The resulting relaxed allocation is
\[
    n_i
    =
    B_2
    \frac{\sqrt{h_i}}
         {\sum_{r=0}^{L-1}c_r\sqrt{h_r}},
    \qquad
    0\le i<L.
\]

At each Frank-Wolfe step, we recompute the allocation and choose
\[
    q^\star
    \in
    \arg\max_{1\le q\le Q}
    \ell_q(n).
\]
A line search chooses \(\eta\in[0,1]\). We then update
\[
    \pi
    \leftarrow
    (1-\eta)\pi+\eta e_{q^\star}.
\]
Here, \(e_{q^\star}\) is the unit vector for query \(q^\star\).
The final Frank-Wolfe iterate gives the relaxed particle counts \(n_0,\ldots,n_{L-1}\).

\subsubsection{Final Algorithm}
\label{app:final-algorithm}

Algorithm~\ref{alg:final-algorithm} summarizes the complete procedure. Phase~1 learns the allocation and supplies a full-path component. Phase~2 uses fresh paths to construct the splitting component.

\begin{algorithm}[H]
\caption{Final two-phase algorithm}
\label{alg:final-algorithm}
\begin{algorithmic}[1]
\Require Budget \(B\), pilot budget \(B_1\), split times, segment costs \(c_0,\ldots,c_{L-1}\), and queries \(x_1,\ldots,x_Q\)
\State Use budget \(B_1\) to estimate the variance contributions and retain \(\widehat F_1\) as in Appendix~\ref{app:phase-one}
\State Use the remaining budget \(B_2\) to solve the relaxed allocation problem in Appendix~\ref{app:allocation}
\State Normalize the solution by setting \(r_i=n_i/n_0\)
\State Construct the candidate dyadic profiles \(R^{(1)},\ldots,R^{(S)}\) from \(r\) as in Section~\ref{sec:rounding}
\State Solve the mixture problem in Proposition~\ref{prop:dyadic-mixture-approximation} to obtain \(\lambda\)
\State For each type \(s\), set \(C_s=\sum_i c_iR_i^{(s)}\), set \(M_s=\lfloor\lambda_sB_2/C_s\rfloor\), and simulate \(M_s\) fresh trees
\State Set \(B_2'=\sum_{s=1}^S M_sC_s\) and, for \(x\in\R^d\),
\Statex \hspace{\algorithmicindent}\(\displaystyle
\widehat F_2(x)=\sum_{s=1}^S\sum_{j=1}^{M_s}\sum_{\ell=1}^{R_{L-1}^{(s)}}
\frac{C_s}{B_2'R_{L-1}^{(s)}}\one\!\left(Z_{s,j}^{(\ell)}\le x\right)\)
\State Set \(\beta_{\mathrm P}=JC_{\mathrm{full}}/B\)
\State \Return \(\widehat F_B^{\mathrm{reuse}}=\beta_{\mathrm P}\widehat F_1+(1-\beta_{\mathrm P})\widehat F_2\)
\end{algorithmic}
\end{algorithm}

Each Phase~1 terminal sample receives weight \(\beta_{\mathrm P}/J\), and each Phase~2 leaf of type \(s\) receives weight
\[
    \frac{(1-\beta_{\mathrm P})C_s}
         {B_2'R_{L-1}^{(s)}}.
\]
Conditional on Phase~1,
\[
    \E\!\left[\widehat F_B^{\mathrm{reuse}}(x)\mid\text{Phase~1}\right]
    =
    \beta_{\mathrm P}\widehat F_1(x)
    +(1-\beta_{\mathrm P})F_B(x),
\]
so the reused estimator is generally not conditionally unbiased. It is exactly unconditionally unbiased because \(\E[\widehat F_1(x)]=F_B(x)\). Moreover, \(B_1=o(B)\) implies \(\beta_{\mathrm P}\to0\), so the conditional discrepancy introduced by reuse vanishes asymptotically.

\subsection{MMD}
\label{app:mmd-procedure}

This procedure is used for CIFAR-10 and FFHQ. It follows Appendix~\ref{app:allocation-estimation} with three changes. Phase~1 estimates the single profile \(w\) of Appendix~\ref{app:mmd-theory} from pairs of paths that share a history, without queries or regression. The allocation then has a closed form. The dyadic mixture uses a single tree type. For FFHQ the chain runs in the latent space and the features are computed on latents.

\textbf{Design map.} Phase~1 uses the feature map \(\phi'\) of Appendix~\ref{app:mmd-theory}, which works in the native space of the sampler. The reported MMD uses other feature maps, built on Inception and DINOv2 features (Appendix~\ref{app:evaluation-metrics}). The allocation is therefore chosen without access to the features used to judge it. Below, \(Q_i\) and \(w_i\) refer to \(\phi'\).

\textbf{Pilot.} Each pilot path consists of a complete path and one branch. The branch starts from the state of the complete path at an anchor time \(t_a\) and runs independently to \(T\), so the path costs \(C_{\mathrm{full}}+\sum_{i\ge a}c_i\). Given the state at \(t_a\), the two terminal features are independent with mean \(\psi_a\) evaluated at that state, so their inner product has mean \(Q_a\). The \(A\) anchor levels are the nearest integers to \(A\) evenly spaced points in \([0,L-1]\), and paths are assigned to them cyclically. We fit as many paths as possible within \(B_1=\lfloor10B^{0.66}\rfloor\), and take the largest \(A\) for which every anchor receives at least five paths. Let \(J\) be the number of pilot paths, \(B_1'\le B_1\) their total cost, and \(B_2=B-B_1'\). In our experiments, the pilot uses between \(40\) and \(87\) paths and between \(11\%\) and \(25\%\) of the budget.

\textbf{Estimation.} Let \(Y_j\) and \(Y_j'\) be the features \(\phi'\) of the terminal states of the complete path and the branch of pilot path \(j\), and let \(\mathcal J_a\) be the set of paths anchored at \(a\). We estimate
\[
    \widehat Q_a
    =
    \frac1{|\mathcal J_a|}
    \sum_{j\in\mathcal J_a}
    Y_j^\top Y_j',
    \qquad
    \widehat S
    =
    \max\left\{
        0,\
        \frac{\left\|\sum_{j=1}^JY_j\right\|_2^2-\sum_{j=1}^J\|Y_j\|_2^2}{J(J-1)}
    \right\},
\]
where \(\widehat S\) estimates \(\|P_B\phi'\|_2^2\). Values between anchors are obtained by linear interpolation, and we set \(\widehat Q_L=1\). The interpolation error is not covered by the theory. As in Appendix~\ref{app:phase-one}, we clip to \([\widehat S,1]\) and apply nondecreasing PAVA with the terminal value fixed. The estimate \(\widehat w\) is then given by the formulas for \(w\) in Appendix~\ref{app:mmd-theory}, with \(\widehat Q\) and \(\widehat S\) in place of \(Q\) and \(\|P_B\phi'\|_2^2\).

\textbf{Shrinkage.} For stability, we shrink \(\widehat w\) toward the cost-proportional profile with the same total. With \(W=\sum_i\widehat w_i\) and \(C=\sum_ic_i\), we replace \(\widehat w_i\) by \((1-s)\,\widehat w_i+sWc_i/C\), with \(s=0.05\). Both the learned allocation and Learned-\(c\) use the shrunk estimate. We empirically tested the shringake factor \(s\in\{0, 0.01, 0.025, 0.05, 0.1\}\) and found that \(s=0.05\) gave the best results.

\begin{algorithm}[H]
\caption{Two-phase algorithm for MMD}
\label{alg:mmd-algorithm}
\begin{algorithmic}[1]
\Require Budget \(B\), pilot budget \(B_1\), split times, segment costs \(c_0,\ldots,c_{L-1}\), and design map \(\phi'\)
\State Choose the anchors and simulate the \(J\) pilot paths, at cost \(B_1'\)
\State Estimate \(\widehat w\) from the sibling inner products, shrink it, and apply the root floor
\State Set \(B_2=B-B_1'\) and \(n_i=B_2\sqrt{h_i}/\sum_rc_r\sqrt{h_r}\) with \(h=\PAVA_c(\widehat w)\)
\State Normalize by setting \(r_i=n_i/n_0\), and construct the dyadic candidates \(R^{(1)},\ldots,R^{(S)}\) from \(r\)
\State Select the candidate \(s\) minimizing \(C_s\sum_i\widehat w_i/R_i^{(s)}\), and simulate \(M=\lfloor B_2/C_s\rfloor\) fresh trees of that type
\State Set \(\beta_{\mathrm P}=B_1'/B\)
\State \Return \(\widehat P_G\), which gives weight \(\beta_{\mathrm P}/(2J)\) to each pilot terminal state and weight \((1-\beta_{\mathrm P})/(MR_{L-1}^{(s)})\) to each Phase~2 leaf
\end{algorithmic}
\end{algorithm}

\section{Metrics and Evaluations}
\label{app:evaluation-metrics}

Appendix~\ref{app:benchmark-details} defines the models and simulation settings. To evaluate each model, we compute a metric between the estimated CDF and a reference CDF. For each model, the reference CDF is generated once and fixed, and the same sample is used for every method, budget, split schedule, and repetition. For the OU-process and overdamped Langevin models, the reference CDF is the empirical distribution of a large number of samples obtained by simulating a finely discretized SDE. For the EDM model, the reference CDF is the empirical distribution of a large number of samples drawn directly from the Gaussian mixture. For CIFAR-10 and FFHQ, the reference CDF is the empirical distribution of a large number of samples generated by the full DDPM sampler, which for FFHQ runs in latent space and is decoded to images. Table~\ref{tab:reference-samples} gives the source and size of each sample.

\begin{table}[H]
\caption{Reference samples used to evaluate the generated samples.}
\label{tab:reference-samples}
\begin{center}
\begin{tabularx}{\textwidth}{@{}lXr@{}}
\toprule
\textbf{Model} & \textbf{Reference source} & \textbf{Samples} \\
\midrule
OU process & Euler-Maruyama with $20{,}000$ steps & $2{,}500{,}000$ \\
Overdamped Langevin & Euler-Maruyama with $20{,}000$ steps & $2{,}500{,}000$ \\
EDM & Direct draws from the Gaussian mixture & $5{,}000{,}000$ \\
CIFAR-10 DDPM & Full $1{,}000$-step DDPM sampler & $20{,}000$ \\
FFHQ LDM & Full $1{,}000$-step latent DDPM sampler, decoded to $256\times256$ images & $20{,}000$ \\
\bottomrule
\end{tabularx}
\end{center}
\end{table}

For any function \(f\) on the terminal state space, define the generated and reference empirical operators by
\[
    \widehat P_G f
    =
    \frac{\beta_{\mathrm P}}{J}
    \sum_{j=1}^J f\!\left(X_T^{(j)}\right)
    +(1-\beta_{\mathrm P})
    \sum_{s=1}^S
    \sum_{j=1}^{M_s}
    \sum_{\ell=1}^{R_{L-1}^{(s)}}
    \frac{C_s}{B_2'R_{L-1}^{(s)}}
    f\!\left(Z_{s,j}^{(\ell)}\right),
    \qquad
    \widehat P_R f
    =
    \frac1{|R|}\sum_{z\in R}f(z).
\]
Thus every Phase~1 output receives weight \(\beta_{\mathrm P}/J\), every Phase~2 leaf of type \(s\) receives weight \((1-\beta_{\mathrm P})C_s/(B_2'R_{L-1}^{(s)})\), and the fixed reference sample remains uniformly weighted. This is the estimator of Appendix~\ref{app:allocation-estimation}. For the MMD procedure, the Phase~1 term averages over the \(2J\) pilot terminal states instead, as in Appendix~\ref{app:mmd-procedure}. The OU, overdamped Langevin, and EDM examples are two-dimensional. For \(a_x(z)=\one(z\le x)\), compare the generated output \(G\) with the fixed reference sample \(R\) using the two-sample lower-orthant KS distance
\[
    D_{\mathrm{KS}}(G,R)
    =
    \sup_{x\in\R^2}
    \left|
        \widehat P_Ga_x-\widehat P_Ra_x
    \right|.
\]
Both empirical CDFs change only at observed coordinate values. We evaluate that grid and compute the weighted supremum exactly.

The same grid calculation is not practical for CIFAR-10, where each image has \(3{,}072\) coordinates, or for FFHQ, where each decoded image has \(196{,}608\). Hence, we use the maximum mean discrepancy \citep{GrettonEtAl2012}. We compute it with the random Fourier feature approximation of \citet{ZhaoMeng2015}, using a fixed map drawn once and reused for every comparison. Phase~1 of the MMD procedure uses a separate map in the native space of the sampler (Appendix~\ref{app:mmd-procedure}).

We compute the reported MMD on image features rather than on pixels. For FFHQ, latents are first decoded to images. Each image is quantized to eight bits and passed through two pretrained networks. The first is Inception-v3 \citep{Szegedy2016}, applied after resizing to \(299\times299\). We use its \(2{,}048\)-dimensional pool3 features. The second is DINOv2 ViT-L/14 \citep{Oquab2023}, applied after resizing to \(224\times224\). We use its \(1{,}024\)-dimensional class token. The main text reports the Inception MMD.

The feature map is built the same way in every space. Each coordinate is standardized using the mean and standard deviation of the fixed reference sample. Let \(\eta\) be the median Euclidean distance over sampled pairs of standardized reference points. We use the eight scales
\[
    \eta\,2^{-4},\ldots,\eta\,2^3.
\]
We draw \(1{,}024\) orthogonal random frequencies and divide them equally across these scales. Each frequency contributes one sine coordinate and one cosine coordinate. The coordinates are scaled so that squared feature distances are averaged over the frequencies.

Let \(\phi\) denote the resulting feature map. For generated output \(G\) and reference sample \(R\), define
\[
    \widehat{\operatorname{MMD}}_\phi(G,R)
    =
    \left\|
        \widehat P_G\phi-\widehat P_R\phi
    \right\|_2.
\]
Each learned-splitting repetition reruns Phase~1, optimization, dyadic-mixture construction, and Phase~2, and reuses its Phase~1 terminal samples with weight \(\beta_{\mathrm P}\). All methods have the same total budget \(B\). We use \(2{,}500\) repetitions for each two-dimensional setting and \(50\) repetitions for each CIFAR-10 and FFHQ setting. We report the mean empirical-reference error.

For a splitting method with mean error \(\overline D_{\mathrm{split}}\) and an independent-path baseline with mean error \(\overline D_{\mathrm{ind}}\), the reported percentage reduction is
\[
    100
    \left(
        1-
        \frac{\overline D_{\mathrm{split}}}
             {\overline D_{\mathrm{ind}}}
    \right).
\]

\section{Benchmark Models and Simulation Settings}
\label{app:benchmark-details}

\subsection{Models}

\begin{table}[H]
\caption{Benchmark models. The EDM denoiser is trained and sampled in standardized coordinates, so its prior is stated in that scale.}
\label{tab:benchmark-models}
\centering
\footnotesize
\renewcommand{\arraystretch}{1.4}
\begin{tabularx}{\textwidth}{@{}>{\raggedright\arraybackslash}p{0.14\textwidth}>{\raggedright\arraybackslash}X >{\raggedright\arraybackslash}p{0.23\textwidth}@{}}
\toprule
\textbf{Model} & \textbf{Dynamics} & \textbf{Settings} \\
\midrule
\textbf{OU Process} &
$\displaystyle
\begin{aligned}
dX_{t,j}
&=\big[1.35(-0.2-X_{t,j})+0.25(\bar X_t-X_{t,j})\big]dt\\
&\quad+0.65\,dW_{t,j},\\
\bar X_t&=(X_{t,1}+X_{t,2})/2
\end{aligned}$
& $\displaystyle
\begin{aligned}
j&=1,2,\\
X_{0,j}&\stackrel{\mathrm{iid}}{\sim}\mathcal N(0,1),\\
T&=1
\end{aligned}$ \\
\cmidrule(lr){1-3}
\textbf{Overdamped Langevin} &
$\displaystyle
\begin{aligned}
U(x_1,x_2)
&=(x_1^2-1)^2+(x_2^2-1)^2+\tfrac12(x_2-x_1)^2,\\
dX_t&=-\nabla U(X_t)\,dt+\sqrt2\,dW_t
\end{aligned}$
& $\displaystyle
\begin{aligned}
j&=1,2,\\
X_{0,j}&\stackrel{\mathrm{iid}}{\sim}\mathcal N(-1,0.05),\\
T&=2
\end{aligned}$ \\
\cmidrule(lr){1-3}
\textbf{EDM} &
$\displaystyle
\begin{aligned}
\frac{dX_\sigma}{d\sigma}
&=\frac{X_\sigma-D_\theta(X_\sigma,\sigma)}{\sigma}
=-\sigma s_\theta(X_\sigma,\sigma),\\[2pt]
s_\theta(x,\sigma)&=\frac{D_\theta(x,\sigma)-x}{\sigma^2}
\end{aligned}$
& $\displaystyle
\begin{aligned}
X_{\sigma_{\max}}&\sim\mathcal N(0,\sigma_{\max}^2I_2),\\
\sigma_{\max}&=80,\\
\sigma_{\min}&=0.002
\end{aligned}$

\textnormal{Neural denoiser:} $D_\theta$ \\
\cmidrule(lr){1-3}
\textbf{CIFAR-10 DDPM} &
\texttt{google/ddpm-cifar10-32}
& Bundled DDPM scheduler, with step counts in Table~\ref{tab:numerical-schedules} \\
\cmidrule(lr){1-3}
\textbf{FFHQ LDM} &
\texttt{asparius/ldm-ffhq-256}, a latent DDPM paired with a KL autoencoder
& DDPM scheduler rebuilt from the shipped betas, with step counts in Table~\ref{tab:numerical-schedules}. Sampling runs on $4\times32\times32$ latents, which are decoded to $256\times256$ images \\
\bottomrule
\end{tabularx}
\end{table}

The EDM row of Table~\ref{tab:benchmark-models} shows the probability-flow ODE used by the sampler. Following \citet{Karras2022edm}, we trained the neural denoiser $D_\theta$ on samples from the equally weighted Gaussian mixture
\[
\frac12\mathcal N\!\left(
\begin{bmatrix}-1\\1\end{bmatrix},
\begin{bmatrix}0.20&0.05\\0.05&0.30\end{bmatrix}
\right)
+
\frac12\mathcal N\!\left(
\begin{bmatrix}1\\-1\end{bmatrix},
\begin{bmatrix}0.30&-0.08\\-0.08&0.15\end{bmatrix}
\right).
\]
We use the stochastic EDM sampler with time-step exponent $3$, in the parameterization of \citet{Karras2022edm}. At each step, stochastic churn is applied before a second-order Heun update of the probability-flow ODE. The churn parameters are $(S_{\mathrm{churn}},S_{\min},S_{\max},S_{\mathrm{noise}})=(40,0,80,1)$. Direct samples from the mixture are used only for the reference distribution.

\subsection{Simulation settings}

The nine-split schedule splits every \(10\%\) from \(10\%\) through \(90\%\). The nineteen-split schedule splits every \(5\%\) from \(5\%\) through \(95\%\). The thirty-nine-split schedule splits every \(2.5\%\) from \(2.5\%\) through \(97.5\%\). For the KS procedure of Appendix~\ref{app:allocation-estimation}, we use \(K=5\) folds and \(Q=1{,}024\) queries. We take \(B_1=\lfloor 5B^{0.66}\rfloor\) for OU and overdamped Langevin and \(B_1=\lfloor 10B^{0.66}\rfloor\) for EDM. CIFAR-10 and FFHQ use the MMD procedure of Appendix~\ref{app:mmd-procedure} with \(B_1=\lfloor 10B^{0.66}\rfloor\). All methods use the same time grid across budgets, with the number of steps per budget given in Table~\ref{tab:numerical-schedules}.

OU and overdamped Langevin are solved with Euler-Maruyama, EDM uses the second-order Heun solver of \citet{Karras2022edm}, and DDPM uses the DDPM sampler. Table~\ref{tab:numerical-schedules} gives the step counts and EDM network-evaluation counts.

\begin{table}[H]
\caption{Numerical schedules by budget. Euler steps apply to OU and overdamped Langevin. DDPM steps apply to both CIFAR-10 and FFHQ, which are run only through $B=500$k. EDM NFE is the number of network evaluations per sample.}
\label{tab:numerical-schedules}
\begin{center}
\begin{tabular}{rrrrr}
\toprule
$B$ & Euler steps & EDM steps & EDM NFE & DDPM steps \\
\midrule
50k  & -   & -  & -   & 200 \\
100k & 130 & 40 & 79  & 252 \\
200k & 160 & 46 & 91  & 317 \\
500k & 220 & 55 & 109 & 430 \\
1M   & 280 & 63 & 125 & -   \\
2M   & 350 & 73 & 145 & -   \\
5M   & 475 & 87 & 173 & -   \\
\bottomrule
\end{tabular}
\end{center}
\end{table}

The step and network-evaluation counts in Table~\ref{tab:numerical-schedules} determine the segment costs $c_i$. Cost is measured in numerical steps for the OU, overdamped Langevin, and DDPM samplers, and in network evaluations for EDM. The final EDM step omits the corrector and costs one evaluation, so $K$ solver steps use $2K-1$ network evaluations.

\subsection{Uniform and Learned \texorpdfstring{\(c\)}{c} Baselines}
\label{app:c-baselines}

Both baselines restrict the normalized relaxed profile to
\[
    r_i(c)=c^i,
    \qquad
    0\le i<L.
\]
To distinguish the split factor \(c\) from the segment costs \(c_i\) above, write those costs as \(\kappa_i\) in this subsection. Uniform-\(c\) uses no pilot phase and fixes \(c\) before simulation. We construct the exact dyadic candidate trees by the logarithmic rounding in Section~\ref{sec:rounding}. Since this baseline has no variance estimates, we assign each candidate the length of the set of shifts \(u\in[0,1)\) that produces it, floor the resulting root counts, and use the realized budget shares in the estimator. For OU, overdamped Langevin, and EDM, we use \(c\in\{1.1,1.15,1.25,1.5\}\) with nine splits, \(c\in\{1.1,1.15,1.2\}\) with nineteen splits, and \(c\in\{1.05,1.1\}\) with thirty-nine splits.

Learned-\(c\) uses the same Phase~1 estimates \(\widehat v_i(x_q)\), pilot budget, and remaining budget \(B_2\) as the learned allocation, but restricts the optimization to the one-parameter family above. It selects
\[
    \widehat c
    \in
    \operatorname*{argmin}_{c\ge1}
    \left(\sum_{i=0}^{L-1}\kappa_i c^i\right)
    \max_{1\le q\le Q}
    \left(\sum_{i=0}^{L-1}
    \widehat v_i(x_q)c^{-i}\right).
\]
After constructing the dyadic candidates \(R^{(1)},\ldots,R^{(S)}\) from \(r(\widehat c)\), define
\[
    C_s=\sum_{i=0}^{L-1}\kappa_iR_i^{(s)},
    \qquad
    E_{sq}
    =
    C_s\sum_{i=0}^{L-1}
    \frac{\widehat v_i(x_q)}{R_i^{(s)}}.
\]
The Phase~2 budget shares solve the linear program
\[
\begin{aligned}
    \min_{\lambda,t}\quad & t \\
    \text{subject to}\quad
    & \sum_{s=1}^S\lambda_sE_{sq}\le t,
      && 1\le q\le Q,\\
    & \sum_{s=1}^S\lambda_s=1,
      \qquad \lambda_s\ge0.
\end{aligned}
\]
We initially assign \(M_s=\lfloor\lambda_sB_2/C_s\rfloor\) roots. Candidate types receiving no roots are removed and the program is re-solved. As in Algorithm~\ref{alg:final-algorithm}, the Phase~1 terminal samples are retained with their realized budget share and the Phase~2 weights are scaled by the remaining share. For CIFAR-10 and FFHQ, Learned-\(c\) uses the trace profile \(\widehat w\) in place of the query values, as described in Appendix~\ref{app:mmd-procedure}.

\section{OU Finite-Query Minimax Benchmark}
\label{app:ou-oracle}

For the linear-Gaussian OU process, the Euler chain gives the variance contributions without Phase~1 estimation. We use them to solve a finite-query version of the minimax problem in Corollary~\ref{cor:minimax-var-opt}. We call the result the oracle allocation.

\subsection{Oracle allocation}

For a lower-orthant event $A_x=\{X_K\le x\}$, define
\[
    M_k(x)=\Pbb(A_x\mid X_k),
    \qquad
    Q_x(k)=\E[M_k(x)^2].
\]
The quantity $Q_x(k)$ is the probability that two terminal descendants both lie in $A_x$ when they share the state at step $k$ and use independent noise afterward. Because the process is linear and Gaussian, we evaluate $F(x)$ and $Q_x(k)$ from the Euler-chain covariances and numerical Gaussian CDFs. If $k_1,\ldots,k_{L-1}$ are the split steps, their successive differences give the segmentwise variance vector
\[
\begin{aligned}
    v_0(x)&=Q_x(k_1)-F(x)^2,\\
    v_i(x)&=Q_x(k_{i+1})-Q_x(k_i),
    \quad 1\le i<L-1,\\
    v_{L-1}(x)&=F(x)-Q_x(k_{L-1}).
\end{aligned}
\]
We evaluate these vectors on a fixed set of \(10{,}200\) lower-orthant queries. A Frank-Wolfe solver optimizes over their convex hull using the cost-weighted finite minimax allocation of Appendix~\ref{app:cost-weighted}. Thus the allocation comes from the finite minimax problem rather than one representative query. Normalize the relaxed oracle counts by their root count and generate the corresponding exact tree candidates. Since the exact variance table is available, solve the cost-weighted tree-mixture linear program with that table. Floor the root counts and use realized budget-share weights.

\subsection{Simulated KS Evaluation}

At $B=5\times10^6$, the Euler chain has $K=475$ steps. We compute the oracle allocation once for each split schedule. Figure~\ref{fig:ou-oracle-allocations} compares this allocation with the learned mean allocation over $2{,}500$ complete experiment repetitions. Here $R_i$ is the number of particles per initial root after split $i$. A value of one means that no branching is needed.

\begin{figure}[H]
    \centering
    \begingroup
    \setlength{\fboxsep}{0pt}
    \fbox{\includegraphics[width=\dimexpr\textwidth-2\fboxrule\relax]{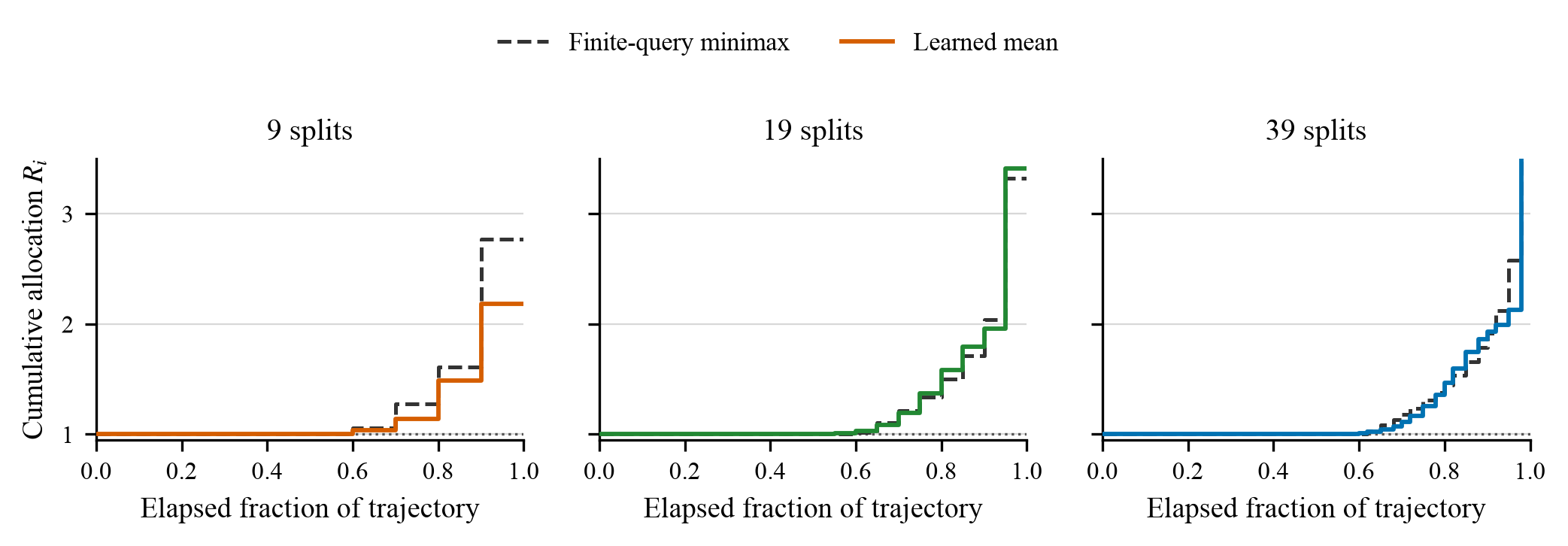}}
    \endgroup
    \caption{Finite-query minimax and learned mean cumulative OU allocations for schedules with 9, 19, and 39 splits of the $475$-step Euler process at $B=5\times10^6$. Learned values are means over $2{,}500$ runs.}
    \label{fig:ou-oracle-allocations}
\end{figure}

We fix the oracle factors and run the same budgeted splitting simulator used for the learned allocation. For each of $10{,}000$ independent repetitions, we compute the exact two-sample lower-orthant KS distance to the fixed reference sample described in Appendix~\ref{app:evaluation-metrics}. The distance for one repetition is computable, but its expectation under the branching scheme is not available in closed form. Table~\ref{tab:ou-oracle-reductions} therefore reports the simulated reduction in mean KS relative to the shared independent-path baseline.

\begin{table}[H]
\caption{Reduction in measured mean OU KS at $B=5\times10^6$ for the finite-query minimax and learned allocations relative to the shared independent-path baseline. Half-widths of two-sided 90\% normal confidence intervals are in parentheses. Oracle means use $10{,}000$ repetitions. Learned and independent-path means use $2{,}500$. The final column is the learned reduction as a percentage of the minimax-allocation reduction.}
\label{tab:ou-oracle-reductions}
\begin{center}
\begin{tabular}{lrrr}
\toprule
Split times & \shortstack{Oracle\\reduction} & \shortstack{Learned\\reduction} & \shortstack{Oracle\\captured} \\
\midrule
9 & \ResultCell{13.9}{0.7} & \ResultCell{14.0}{0.9} & $100.7\%$ \\
19 & \ResultCell{22.8}{0.6} & \ResultCell{18.8}{0.9} & $82.6\%$ \\
39 & \ResultCell{22.4}{0.6} & \ResultCell{21.8}{0.9} & $97.2\%$ \\
\bottomrule
\end{tabular}
\end{center}
\end{table}

\section{Complete Numerical Results}
\label{app:complete-results}

Figure~\ref{fig:absolute-metrics} reports the absolute error behind the percentage reductions in the main text. It also shows how the mean error changes with the computational budget.

\begin{figure}[H]
    \centering
    \begingroup
    \setlength{\fboxsep}{0pt}
    \fbox{\includegraphics[width=\dimexpr\textwidth-2\fboxrule\relax]{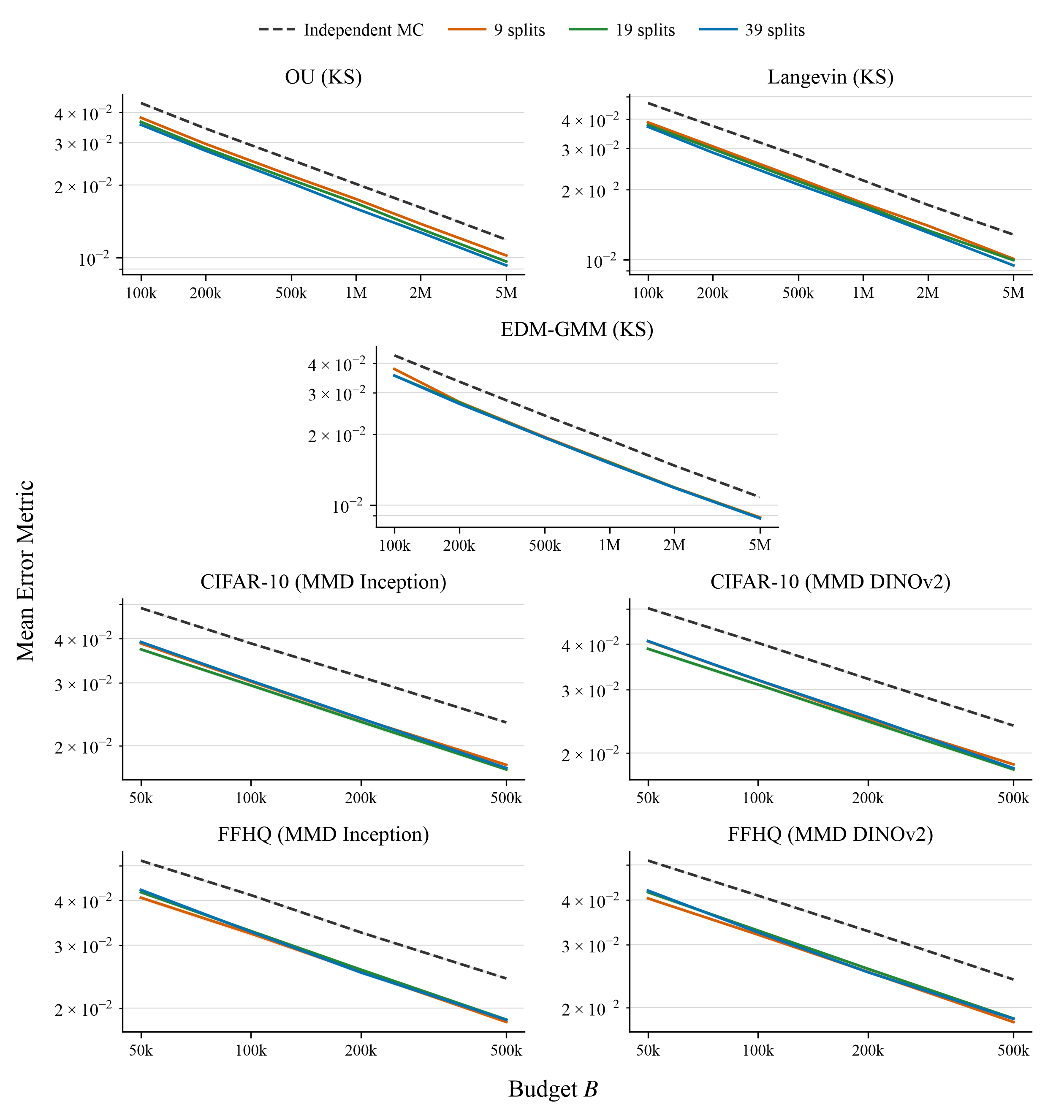}}
    \endgroup
    \caption{Absolute mean error versus budget for independent full-path simulation and the learned allocation. The first three panels show KS over $2{,}500$ repetitions, and the CIFAR-10 and FFHQ panels show MMD with Inception and DINOv2 features over $50$ repetitions. Both axes are logarithmic.}
    \label{fig:absolute-metrics}
\end{figure}

Tables~\ref{tab:complete-ou}-\ref{tab:complete-ffhq} report the percentage reduction in mean error relative to independent full-path simulation. Each cell gives the estimated reduction, with the half-width of its two-sided 90\% normal confidence interval in parentheses, computed by the delta method for the ratio of two independent sample means. In each budget column, the method with the lowest mean error is bold. For CIFAR-10 and FFHQ, this is done separately for each MMD.

\begin{table}[H]
\caption{OU process: percentage reduction in mean KS at each budget and split schedule, for the learned allocation, Uniform-$c$, and Learned-$c$. Half-widths of two-sided 90\% normal confidence intervals are in parentheses.}\label{tab:complete-ou}
\centering
\small
\renewcommand{\arraystretch}{1.05}
\begin{tabular}{@{}lrrrrrr@{}}
\toprule
Allocation & 100k & 200k & 500k & 1M & 2M & 5M \\
\midrule
\multicolumn{7}{@{}l}{\textbf{9 split points}} \\
Learned & \ResultCell{13.1}{0.9} & \ResultCell{13.6}{0.9} & \ResultCell{14.1}{0.9} & \ResultCell{13.4}{0.9} & \ResultCell{14.4}{0.9} & \ResultCell{14.0}{0.9} \\
Learned $c$ & \ResultCell{12.9}{0.9} & \ResultCell{12.1}{0.9} & \ResultCell{12.9}{0.9} & \ResultCell{12.5}{0.9} & \ResultCell{13.4}{0.9} & \ResultCell{13.3}{0.9} \\
Uniform ($c=1.1$) & \ResultCell{9.5}{1.0} & \ResultCell{8.6}{1.0} & \ResultCell{8.7}{1.0} & \ResultCell{8.6}{0.9} & \ResultCell{8.4}{1.0} & \ResultCell{8.8}{0.9} \\
Uniform ($c=1.15$) & \ResultCell{11.5}{1.0} & \ResultCell{12.0}{1.0} & \ResultCell{10.2}{1.0} & \ResultCell{9.8}{1.0} & \ResultCell{10.8}{0.9} & \ResultCell{10.5}{1.0} \\
Uniform ($c=1.25$) & \ResultCell{11.6}{1.0} & \ResultCell{11.2}{1.0} & \ResultCell{11.8}{1.0} & \ResultCell{9.9}{1.0} & \ResultCell{11.0}{1.0} & \ResultCell{11.2}{1.0} \\
Uniform ($c=1.5$) & \ResultCell{-12.1}{1.5} & \ResultCell{-12.6}{1.5} & \ResultCell{-11.4}{1.5} & \ResultCell{-13.8}{1.5} & \ResultCell{-11.3}{1.5} & \ResultCell{-11.0}{1.5} \\
\midrule
\multicolumn{7}{@{}l}{\textbf{19 split points}} \\
Learned & \ResultCell{16.4}{0.9} & \ResultCell{17.1}{0.9} & \ResultCell{17.3}{0.9} & \ResultCell{16.7}{0.9} & \ResultCell{18.4}{0.9} & \ResultCell{18.8}{0.9} \\
Learned $c$ & \ResultCell{12.1}{0.9} & \ResultCell{12.4}{0.9} & \ResultCell{12.6}{0.9} & \ResultCell{13.0}{0.9} & \ResultCell{13.5}{0.9} & \ResultCell{13.7}{0.9} \\
Uniform ($c=1.1$) & \ResultCell{11.8}{1.0} & \ResultCell{11.3}{1.0} & \ResultCell{11.7}{1.0} & \ResultCell{10.5}{1.0} & \ResultCell{11.1}{1.0} & \ResultCell{11.2}{1.0} \\
Uniform ($c=1.15$) & \ResultCell{7.3}{1.1} & \ResultCell{8.3}{1.1} & \ResultCell{6.9}{1.2} & \ResultCell{6.6}{1.1} & \ResultCell{7.6}{1.1} & \ResultCell{6.6}{1.1} \\
Uniform ($c=1.2$) & \ResultCell{-4.0}{1.4} & \ResultCell{-5.9}{1.4} & \ResultCell{-5.4}{1.4} & \ResultCell{-4.9}{1.4} & \ResultCell{-5.0}{1.4} & \ResultCell{-5.2}{1.4} \\
\midrule
\multicolumn{7}{@{}l}{\textbf{39 split points}} \\
Learned & \BestResultCell{18.7}{0.9} & \BestResultCell{19.1}{0.9} & \BestResultCell{20.0}{0.9} & \BestResultCell{21.0}{0.8} & \BestResultCell{21.1}{0.9} & \BestResultCell{21.8}{0.9} \\
Learned $c$ & \ResultCell{12.1}{0.9} & \ResultCell{12.0}{0.9} & \ResultCell{13.1}{0.9} & \ResultCell{12.4}{0.9} & \ResultCell{13.7}{0.9} & \ResultCell{13.2}{0.9} \\
Uniform ($c=1.05$) & \ResultCell{12.3}{1.0} & \ResultCell{10.9}{1.0} & \ResultCell{12.2}{1.0} & \ResultCell{10.8}{1.0} & \ResultCell{11.0}{1.0} & \ResultCell{11.4}{1.0} \\
Uniform ($c=1.1$) & \ResultCell{-9.9}{1.5} & \ResultCell{-8.9}{1.5} & \ResultCell{-8.8}{1.5} & \ResultCell{-9.7}{1.5} & \ResultCell{-9.5}{1.5} & \ResultCell{-9.5}{1.5} \\
\bottomrule
\end{tabular}

\end{table}

\begin{table}[H]
\caption{Overdamped Langevin: percentage reduction in mean KS at each budget and split schedule, for the learned allocation, Uniform-$c$, and Learned-$c$. Half-widths of two-sided 90\% normal confidence intervals are in parentheses.}\label{tab:complete-langevin}
\centering
\small
\renewcommand{\arraystretch}{1.05}
\begin{tabular}{@{}lrrrrrr@{}}
\toprule
Allocation & 100k & 200k & 500k & 1M & 2M & 5M \\
\midrule
\multicolumn{7}{@{}l}{\textbf{9 split points}} \\
Learned & \ResultCell{17.3}{1.0} & \ResultCell{18.2}{1.0} & \ResultCell{19.9}{1.0} & \ResultCell{20.0}{0.9} & \ResultCell{18.5}{1.0} & \ResultCell{21.2}{0.9} \\
Learned $c$ & \ResultCell{16.1}{1.0} & \ResultCell{18.0}{1.0} & \ResultCell{20.0}{1.0} & \ResultCell{18.9}{1.0} & \ResultCell{18.4}{1.0} & \ResultCell{20.5}{0.9} \\
Uniform ($c=1.1$) & \ResultCell{10.6}{1.0} & \ResultCell{10.9}{1.0} & \ResultCell{11.5}{1.0} & \ResultCell{10.9}{1.0} & \ResultCell{11.0}{1.0} & \ResultCell{12.1}{1.0} \\
Uniform ($c=1.15$) & \ResultCell{13.9}{1.0} & \ResultCell{14.2}{1.0} & \ResultCell{16.2}{1.0} & \ResultCell{15.5}{1.0} & \ResultCell{14.7}{1.0} & \ResultCell{15.0}{1.0} \\
Uniform ($c=1.25$) & \ResultCell{19.6}{1.0} & \ResultCell{19.5}{1.0} & \ResultCell{20.7}{1.0} & \ResultCell{20.8}{0.9} & \ResultCell{20.7}{0.9} & \ResultCell{20.7}{1.0} \\
Uniform ($c=1.5$) & \ResultCell{24.9}{1.0} & \ResultCell{25.0}{1.0} & \ResultCell{25.3}{1.0} & \ResultCell{24.0}{1.0} & \ResultCell{23.5}{1.0} & \ResultCell{25.1}{1.0} \\
\midrule
\multicolumn{7}{@{}l}{\textbf{19 split points}} \\
Learned & \ResultCell{19.4}{1.0} & \ResultCell{20.3}{0.9} & \ResultCell{22.0}{1.0} & \ResultCell{21.6}{1.0} & \ResultCell{22.2}{0.9} & \ResultCell{22.3}{0.9} \\
Learned $c$ & \ResultCell{17.1}{1.0} & \ResultCell{18.1}{1.0} & \ResultCell{19.2}{1.0} & \ResultCell{19.1}{1.0} & \ResultCell{19.2}{1.0} & \ResultCell{21.0}{0.9} \\
Uniform ($c=1.1$) & \ResultCell{18.2}{1.0} & \ResultCell{18.7}{1.0} & \ResultCell{19.7}{1.0} & \ResultCell{18.5}{1.0} & \ResultCell{19.0}{1.0} & \ResultCell{20.1}{0.9} \\
Uniform ($c=1.15$) & \ResultCell{24.0}{0.9} & \ResultCell{23.0}{0.9} & \ResultCell{24.5}{0.9} & \ResultCell{23.7}{0.9} & \ResultCell{23.5}{0.9} & \ResultCell{24.0}{0.9} \\
Uniform ($c=1.2$) & \ResultCell{25.0}{0.9} & \ResultCell{25.6}{0.9} & \ResultCell{26.9}{0.9} & \ResultCell{25.0}{0.9} & \ResultCell{24.7}{1.0} & \ResultCell{26.1}{0.9} \\
\midrule
\multicolumn{7}{@{}l}{\textbf{39 split points}} \\
Learned & \ResultCell{20.9}{1.0} & \ResultCell{23.0}{0.9} & \ResultCell{24.6}{1.0} & \ResultCell{23.6}{0.9} & \ResultCell{23.8}{0.9} & \ResultCell{26.2}{0.9} \\
Learned $c$ & \ResultCell{16.2}{1.0} & \ResultCell{17.6}{1.0} & \ResultCell{19.7}{1.0} & \ResultCell{20.3}{1.0} & \ResultCell{19.2}{1.0} & \ResultCell{21.0}{0.9} \\
Uniform ($c=1.05$) & \ResultCell{19.3}{1.0} & \ResultCell{19.7}{0.9} & \ResultCell{20.5}{1.0} & \ResultCell{20.3}{1.0} & \ResultCell{19.6}{0.9} & \ResultCell{21.1}{0.9} \\
Uniform ($c=1.1$) & \BestResultCell{25.4}{1.0} & \BestResultCell{26.2}{0.9} & \BestResultCell{27.1}{1.0} & \BestResultCell{26.2}{1.0} & \BestResultCell{26.2}{0.9} & \BestResultCell{26.5}{0.9} \\
\bottomrule
\end{tabular}

\end{table}

\begin{table}[H]
\caption{EDM Gaussian mixture: percentage reduction in mean KS at each budget and split schedule, for the learned allocation, Uniform-$c$, and Learned-$c$. Half-widths of two-sided 90\% normal confidence intervals are in parentheses.}\label{tab:complete-edm}
\centering
\small
\renewcommand{\arraystretch}{1.05}
\begin{tabular}{@{}lrrrrrr@{}}
\toprule
Allocation & 100k & 200k & 500k & 1M & 2M & 5M \\
\midrule
\multicolumn{7}{@{}l}{\textbf{9 split points}} \\
Learned & \ResultCell{12.5}{0.9} & \ResultCell{18.0}{0.8} & \ResultCell{19.0}{0.7} & \ResultCell{19.3}{0.7} & \ResultCell{19.2}{0.6} & \ResultCell{18.0}{0.6} \\
Learned $c$ & \ResultCell{16.7}{0.8} & \ResultCell{18.2}{0.8} & \ResultCell{18.4}{0.7} & \ResultCell{19.4}{0.7} & \ResultCell{18.8}{0.6} & \ResultCell{17.5}{0.6} \\
Uniform ($c=1.1$) & \ResultCell{8.3}{0.9} & \ResultCell{8.5}{0.9} & \ResultCell{8.7}{0.8} & \ResultCell{8.0}{0.8} & \ResultCell{7.2}{0.7} & \ResultCell{6.6}{0.7} \\
Uniform ($c=1.15$) & \ResultCell{12.0}{0.8} & \ResultCell{12.1}{0.8} & \ResultCell{10.8}{0.8} & \ResultCell{11.2}{0.8} & \ResultCell{10.4}{0.7} & \ResultCell{9.4}{0.7} \\
Uniform ($c=1.25$) & \ResultCell{16.9}{0.8} & \ResultCell{16.9}{0.8} & \ResultCell{15.6}{0.7} & \ResultCell{15.7}{0.7} & \ResultCell{14.5}{0.7} & \ResultCell{13.3}{0.6} \\
Uniform ($c=1.5$) & \BestResultCell{22.8}{0.7} & \ResultCell{22.2}{0.7} & \ResultCell{20.3}{0.7} & \BestResultCell{20.3}{0.7} & \ResultCell{19.4}{0.6} & \ResultCell{17.5}{0.6} \\
\midrule
\multicolumn{7}{@{}l}{\textbf{19 split points}} \\
Learned & \ResultCell{17.9}{0.8} & \ResultCell{18.5}{0.7} & \ResultCell{19.6}{0.7} & \ResultCell{19.7}{0.7} & \ResultCell{19.4}{0.6} & \ResultCell{18.6}{0.6} \\
Learned $c$ & \ResultCell{16.2}{0.8} & \ResultCell{17.2}{0.8} & \ResultCell{17.8}{0.7} & \ResultCell{18.8}{0.7} & \ResultCell{17.3}{0.6} & \ResultCell{17.0}{0.6} \\
Uniform ($c=1.1$) & \ResultCell{14.6}{0.8} & \ResultCell{15.5}{0.8} & \ResultCell{14.4}{0.7} & \ResultCell{13.8}{0.7} & \ResultCell{13.3}{0.7} & \ResultCell{11.8}{0.7} \\
Uniform ($c=1.15$) & \ResultCell{19.0}{0.8} & \ResultCell{18.4}{0.8} & \ResultCell{18.1}{0.7} & \ResultCell{17.6}{0.7} & \ResultCell{16.7}{0.7} & \ResultCell{15.0}{0.6} \\
Uniform ($c=1.2$) & \ResultCell{21.8}{0.7} & \ResultCell{21.1}{0.7} & \ResultCell{20.3}{0.7} & \ResultCell{19.6}{0.7} & \ResultCell{18.8}{0.6} & \ResultCell{17.1}{0.6} \\
\midrule
\multicolumn{7}{@{}l}{\textbf{39 split points}} \\
Learned & \ResultCell{18.0}{0.8} & \ResultCell{19.5}{0.7} & \ResultCell{19.6}{0.7} & \ResultCell{20.3}{0.7} & \ResultCell{19.6}{0.6} & \BestResultCell{18.8}{0.6} \\
Learned $c$ & \ResultCell{15.9}{0.8} & \ResultCell{17.4}{0.8} & \ResultCell{18.5}{0.7} & \ResultCell{19.2}{0.7} & \ResultCell{18.3}{0.6} & \ResultCell{17.8}{0.6} \\
Uniform ($c=1.05$) & \ResultCell{15.1}{0.8} & \ResultCell{15.4}{0.8} & \ResultCell{14.8}{0.7} & \ResultCell{14.1}{0.7} & \ResultCell{13.6}{0.7} & \ResultCell{12.4}{0.6} \\
Uniform ($c=1.1$) & \ResultCell{21.5}{0.7} & \BestResultCell{22.4}{0.7} & \BestResultCell{21.1}{0.7} & \ResultCell{20.2}{0.7} & \BestResultCell{19.7}{0.6} & \ResultCell{17.3}{0.6} \\
\bottomrule
\end{tabular}

\end{table}

\begin{table}[H]
\caption{CIFAR-10 DDPM: percentage reduction in mean MMD with Inception and DINOv2 features at each budget and split schedule, for the learned allocation and Learned-$c$. Half-widths of two-sided 90\% normal confidence intervals are in parentheses.}\label{tab:complete-ddpm}
\centering
\renewcommand{\arraystretch}{1.05}
\begin{tabular}{@{}lrrrr@{}}
\toprule
Allocation & 50k & 100k & 200k & 500k \\
\midrule
\multicolumn{5}{@{}l}{\textbf{MMD Inception}} \\
\multicolumn{5}{@{}l}{\textit{9 split points}} \\
Learned & \ResultCell{20.5}{0.8} & \ResultCell{21.9}{1.0} & \ResultCell{23.8}{0.8} & \ResultCell{24.0}{0.8} \\
Learned $c$ & \ResultCell{20.8}{0.8} & \ResultCell{21.6}{0.9} & \ResultCell{23.5}{0.8} & \ResultCell{23.3}{1.0} \\
\multicolumn{5}{@{}l}{\textit{19 split points}} \\
Learned & \BestResultCell{23.5}{0.9} & \BestResultCell{23.8}{1.1} & \BestResultCell{25.2}{1.2} & \BestResultCell{26.3}{0.9} \\
Learned $c$ & \ResultCell{20.4}{1.0} & \ResultCell{20.7}{1.1} & \ResultCell{22.2}{1.1} & \ResultCell{23.5}{1.0} \\
\multicolumn{5}{@{}l}{\textit{39 split points}} \\
Learned & \ResultCell{19.7}{0.9} & \ResultCell{21.5}{0.9} & \ResultCell{23.6}{0.9} & \ResultCell{25.6}{1.0} \\
Learned $c$ & \ResultCell{19.0}{1.0} & \ResultCell{20.4}{1.0} & \ResultCell{20.7}{1.0} & \ResultCell{22.8}{1.1} \\
\midrule
\multicolumn{5}{@{}l}{\textbf{MMD DINOv2}} \\
\multicolumn{5}{@{}l}{\textit{9 split points}} \\
Learned & \ResultCell{19.0}{0.9} & \ResultCell{21.1}{0.9} & \ResultCell{22.4}{0.9} & \ResultCell{21.9}{0.8} \\
Learned $c$ & \ResultCell{19.6}{0.9} & \ResultCell{21.4}{0.9} & \ResultCell{21.5}{0.9} & \ResultCell{21.7}{1.0} \\
\multicolumn{5}{@{}l}{\textit{19 split points}} \\
Learned & \BestResultCell{22.7}{1.1} & \BestResultCell{23.3}{1.0} & \BestResultCell{23.6}{1.2} & \BestResultCell{24.3}{1.0} \\
Learned $c$ & \ResultCell{19.1}{1.1} & \ResultCell{20.9}{1.0} & \ResultCell{20.6}{1.0} & \ResultCell{21.2}{1.0} \\
\multicolumn{5}{@{}l}{\textit{39 split points}} \\
Learned & \ResultCell{18.8}{0.9} & \ResultCell{21.0}{0.9} & \ResultCell{21.6}{0.8} & \ResultCell{23.7}{1.0} \\
Learned $c$ & \ResultCell{17.7}{0.9} & \ResultCell{19.7}{1.0} & \ResultCell{19.2}{1.0} & \ResultCell{21.0}{1.0} \\
\bottomrule
\end{tabular}

\end{table}

\begin{table}[H]
\caption{FFHQ LDM: percentage reduction in mean MMD with Inception and DINOv2 features at each budget and split schedule, for the learned allocation and Learned-$c$. Half-widths of two-sided 90\% normal confidence intervals are in parentheses.}\label{tab:complete-ffhq}
\centering
\renewcommand{\arraystretch}{1.05}
\begin{tabular}{@{}lrrrr@{}}
\toprule
Allocation & 50k & 100k & 200k & 500k \\
\midrule
\multicolumn{5}{@{}l}{\textbf{MMD Inception}} \\
\multicolumn{5}{@{}l}{\textit{9 split points}} \\
Learned & \BestResultCell{21.2}{1.1} & \BestResultCell{22.0}{1.3} & \ResultCell{22.4}{1.1} & \BestResultCell{24.4}{1.0} \\
Learned $c$ & \ResultCell{18.0}{1.2} & \ResultCell{20.3}{1.0} & \ResultCell{21.3}{1.0} & \ResultCell{23.2}{1.0} \\
\multicolumn{5}{@{}l}{\textit{19 split points}} \\
Learned & \ResultCell{18.4}{1.0} & \ResultCell{20.7}{1.0} & \ResultCell{21.3}{1.1} & \ResultCell{23.3}{1.1} \\
Learned $c$ & \ResultCell{17.9}{0.9} & \ResultCell{19.8}{1.3} & \ResultCell{20.6}{1.1} & \ResultCell{21.9}{1.0} \\
\multicolumn{5}{@{}l}{\textit{39 split points}} \\
Learned & \ResultCell{17.2}{1.1} & \ResultCell{21.0}{1.1} & \BestResultCell{22.8}{1.1} & \ResultCell{23.4}{0.9} \\
Learned $c$ & \ResultCell{17.6}{0.9} & \ResultCell{19.3}{1.2} & \ResultCell{20.9}{1.1} & \ResultCell{22.5}{0.9} \\
\midrule
\multicolumn{5}{@{}l}{\textbf{MMD DINOv2}} \\
\multicolumn{5}{@{}l}{\textit{9 split points}} \\
Learned & \BestResultCell{21.6}{0.9} & \BestResultCell{22.2}{1.1} & \ResultCell{23.1}{1.0} & \BestResultCell{23.9}{0.9} \\
Learned $c$ & \ResultCell{19.2}{1.0} & \ResultCell{20.9}{0.9} & \ResultCell{22.4}{1.0} & \ResultCell{23.0}{0.9} \\
\multicolumn{5}{@{}l}{\textit{19 split points}} \\
Learned & \ResultCell{18.4}{1.1} & \ResultCell{20.1}{1.0} & \ResultCell{21.5}{1.0} & \ResultCell{22.2}{0.9} \\
Learned $c$ & \ResultCell{18.2}{0.9} & \ResultCell{19.9}{1.2} & \ResultCell{21.1}{1.0} & \ResultCell{20.8}{0.9} \\
\multicolumn{5}{@{}l}{\textit{39 split points}} \\
Learned & \ResultCell{17.6}{1.0} & \ResultCell{21.1}{1.1} & \BestResultCell{23.2}{0.9} & \ResultCell{22.4}{0.9} \\
Learned $c$ & \ResultCell{18.9}{0.9} & \ResultCell{20.1}{1.1} & \ResultCell{21.3}{1.0} & \ResultCell{21.7}{1.0} \\
\bottomrule
\end{tabular}

\end{table}

\subsection{Wall-clock cost}
\label{app:wall-clock}

The budget \(B\) counts sampler steps, not seconds. To check that it also tracks running time, we timed \(10\) runs of each method on CIFAR-10 and FFHQ at \(B=500\)k. Figure~\ref{fig:wall-clock} shows the results. Learned splitting takes \(0.7\%\) to \(2.2\%\) longer than independent simulation. The main observation is that the estimation and optimisation steps are negligible compared to the sampling time.

\begin{figure}[H]
    \centering
    \begingroup
    \setlength{\fboxsep}{0pt}
    \fbox{\includegraphics[width=\dimexpr\textwidth-2\fboxrule\relax]{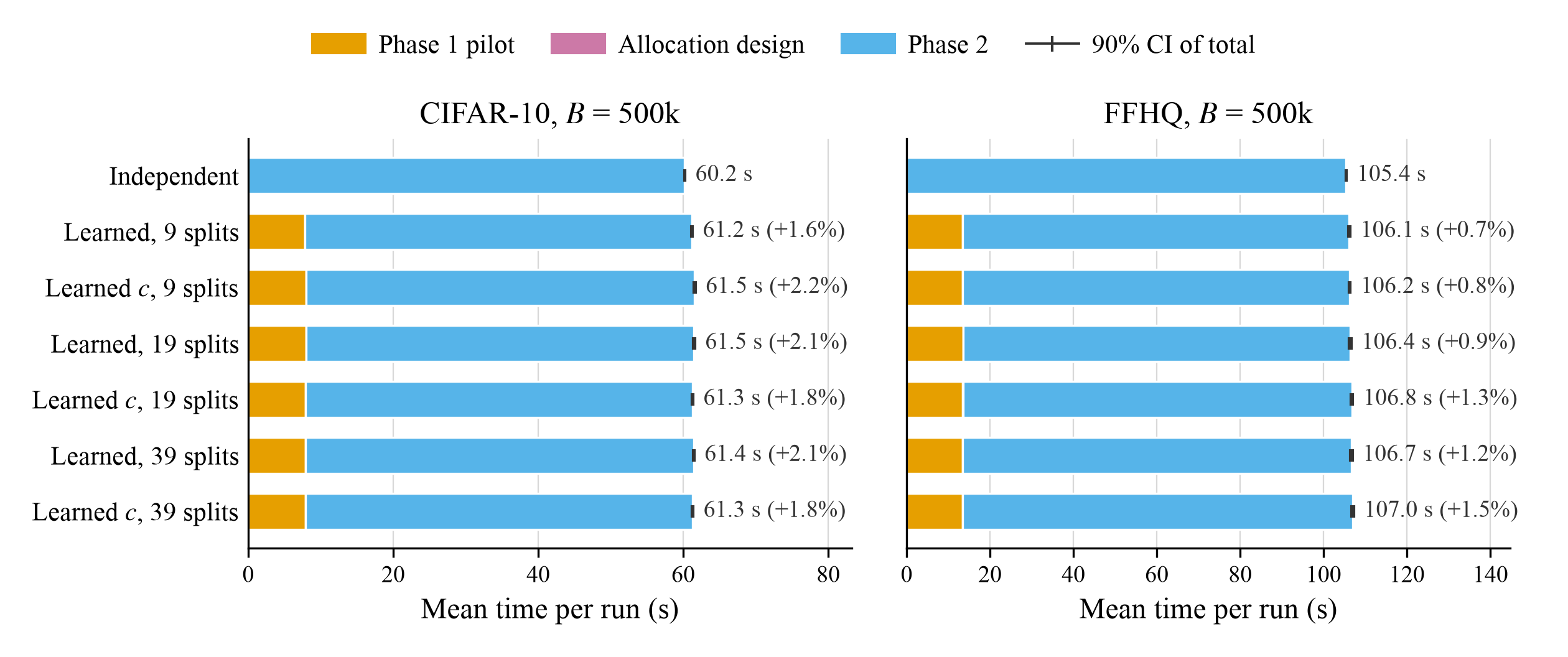}}
    \endgroup
    \caption{Mean running time over 10 runs at \(B=500\)k, split into Phase~1, allocation design and Phase~2.}
    \label{fig:wall-clock}
\end{figure}

\FloatBarrier

\stopcontents[appendices]
\end{document}